\documentclass{article}

\usepackage{microtype}
\usepackage{graphicx}
\usepackage{subfigure}
\usepackage{booktabs} 
\usepackage{multirow}  
\usepackage{bbm}  
\usepackage{enumitem}  

\usepackage{hyperref}

\newcommand{\revised}[1]{{{\textcolor{black}{#1}}}}
\newcommand\ipfmat{X}
\newcommand\ipfrow{p}
\newcommand\ipfcol{q}
\newcommand\ipfparamrow{d^0}
\newcommand\ipfparamcol{d^1}
\newcommand\ipfestmat{M^{\mathrm{IPF}}}
\newcommand\klmat{Y}

\newcommand\Xagg{\bar{X}}
\newcommand\Xtime{X^{(t)}}
\newcommand\estXtime{\hat{X}^{(t)}}
\newcommand\ptime{p^{(t)}}
\newcommand\qtime{q^{(t)}}

\newcommand\vtime{v^{(t)}}

\newcommand\paramrow{u}
\newcommand\paramcol{v}

\usepackage[accepted]{icml2024}

\usepackage{amsmath}
\usepackage{amssymb}
\usepackage{mathtools}
\usepackage{amsthm}

\usepackage[capitalize,noabbrev]{cleveref}

\theoremstyle{plain}
\newtheorem{theorem}{Theorem}[section]

\newtheorem{lemma}[theorem]{Lemma}
\newtheorem{corollary}[theorem]{Corollary}
\theoremstyle{definition}
\newtheorem{definition}[theorem]{Definition}

\theoremstyle{remark}
\newtheorem{remark}[theorem]{Remark}

\usepackage[textsize=tiny]{todonotes}

\icmltitlerunning{Inferring Dynamic Networks from Marginals with Iterative Proportional Fitting}

\begin{document}

\twocolumn[
\icmltitle{Inferring Dynamic Networks from Marginals with Iterative Proportional Fitting}



\icmlsetsymbol{equal}{*}

\begin{icmlauthorlist}
\icmlauthor{Serina Chang}{equal,cs}
\icmlauthor{Frederic Koehler}{equal,stat}
\icmlauthor{Zhaonan Qu}{equal,econ,mse}
\icmlauthor{Jure Leskovec}{cs}
\icmlauthor{Johan Ugander}{mse}
\end{icmlauthorlist}

\icmlaffiliation{cs}{Department of Computer Science, Stanford University}
\icmlaffiliation{stat}{Department of Statistics and Data Science Institute, University of Chicago}
\icmlaffiliation{econ}{Department of Economics, Stanford University}
\icmlaffiliation{mse}{Department of Management Science \& Engineering, Stanford University}

\icmlcorrespondingauthor{Serina Chang}{serinac@stanford.edu}
\icmlcorrespondingauthor{Frederic Koehler}{fkoehler@uchicago.edu}
\icmlcorrespondingauthor{Zhaonan Qu}{zhaonanq@stanford.edu}

\icmlkeywords{Network inference, iterative proportional fitting, Sinkhorn's algorithm}

\vskip 0.3in
]



\printAffiliationsAndNotice{\icmlEqualContribution} 

\begin{abstract}
A common network inference problem, arising from real-world data constraints, is how to infer a dynamic network from its time-aggregated adjacency matrix and time-varying marginals (i.e., row and column sums). 
Prior approaches to this problem have repurposed the classic iterative proportional fitting (IPF) procedure, also known as Sinkhorn's algorithm, with promising empirical results.
However, the statistical foundation for using IPF has not been well understood: under what settings does IPF provide principled estimation of a dynamic network from its marginals, and how well does it estimate the network?
In this work, we establish such a setting, by identifying a generative network model whose maximum likelihood estimates are recovered by IPF.
Our model both reveals implicit assumptions on the use of IPF in such settings and enables new analyses, such as structure-dependent error bounds 
on IPF's parameter estimates.
When IPF fails to converge on sparse network data, we introduce a principled algorithm that guarantees IPF converges under minimal changes to the network structure.
Finally, we conduct experiments with synthetic and real-world data, which demonstrate the practical value of our theoretical and algorithmic contributions.
\end{abstract}

\section{Introduction}
Dynamic networks of human movements are integral to important societal problems, such as epidemic response and transportation planning, but they are rarely fully observed.
Instead, it is often easier to observe the time-varying \textit{marginals} of the network, i.e., the row sums and column sums of its adjacency matrix.
In transportation networks, 
it is easier to observe passengers embarking or disembarking, e.g., at bus stops \citep{navick1994distance}, 
but harder to uncover routes between stops.
In migration networks, it is easier to observe the number of individuals leaving from and arriving at a region but harder to estimate the network of their movements \citep{plane1982information,pham2022migrant}.
In these settings, we observe the time-varying marginals regularly, but only have occasional access to a time-aggregated network (e.g., from surveys). 
Thus, a natural \textit{dynamic network inference problem} arises from such data: given a network with adjacency matrices $\Xtime$ per time step $t$, can we reasonably infer $\Xtime$ from its time-varying marginals and a time-aggregated version, $\sum_t \Xtime$, of the network?

One approach to solving this problem
is to repurpose the classic iterative proportional fitting (IPF) procedure \citep{deming1940ipf}, also widely known as Sinkhorn's algorithm \citep{sinkhorn1974diagonal}.
Given a matrix $\ipfmat$, target row marginals $\ipfrow$, and target column marginals $\ipfcol$, IPF tries to find a \emph{biproportional} scaling of $\ipfmat$ to match the target marginals. 
IPF poses an attractive solution for the network inference problem: it matches the marginal constraints of the problem, is computationally lightweight and space efficient, and is supported by decades of literature devoted to its analysis \citep{bregman1967proof,fienberg1970ipf,csiszar1975iprojection,pukelsheim2009ipf,marino2020optimal,leger2021gradient,carlier2022linear,qu2023sinkhorn}.
Furthermore, prior works have shown empirical success with using IPF to infer networks \citep{liang2006traffic,mccord2010od,chang2021nature}, enabling important applications such as modeling epidemic spread on mobility networks \citep{li2023sdm,chaudhuri2022overdispersion,alimohammadi2023working} and supporting policymakers \citep{chang2021kdd}.
However, despite the appeal and empirical success of using IPF for network inference, what is missing is a firm statistical grounding for \textit{when} and \textit{why} IPF is justified in this setting.
While it is well-known that IPF solves a Kullback-Leibler (KL) divergence minimization problem \citep{ireland1968contingency}, a formal connection to statistical theory is limited, especially in network settings.
Moreover, IPF occasionally fails to converge on sparse network data. While tests exist for whether IPF will converge, there is no clear answer to how to interpret failed convergence, or how to repair it.
Thus, to rigorously employ IPF to infer networks, we ask: under what model is IPF a principled estimator of a dynamic network from its marginals?
Under this model, how well does IPF estimate these networks?
What are principled strategies to ensure that IPF converges?

In this work, we first identify a generative network model, which we term the \emph{biproportional Poisson model}, whose maximum likelihood estimates (MLEs) are recovered by IPF, thus establishing a statistical framework under which IPF is a principled estimation procedure of a dynamic network from its marginals (Theorem \ref{thm:model}). 
Our model clarifies implicit network assumptions when using IPF and enables the analysis of IPF estimates using statistical theory of maximum likelihood estimation.
Next, we provide expectation and tail bounds on estimation errors of the MLEs, and show that finite MLEs exist with high probability under the biproportional Poison model (Theorems \ref{thm:mse}-\ref{thm:finite-mle}).
To address the issue of IPF non-convergence, we introduce a principled and polynomial time algorithm, \texttt{ConvIPF}, that guarantees IPF convergence while making minimal changes to the network structure (Section \ref{sec:convergence}).
Finally, we conduct extensive experiments with synthetic data and two real-world datasets: mobility data from SafeGraph and bikeshare data from New York City's CitiBike (Section \ref{sec:empirics}).\footnote{Our code is available at \url{https://github.com/snap-stanford/ipf-network-inference}. Citibike data \citep{citibike} and SafeGraph mobility data \citep{dewey} are available online.} 
Our experiments demonstrate IPF's ability to infer ground-truth hourly networks, outperforming several baselines, and tie our theoretical and algorithmic contributions to practical insights.

Our results provide much-needed theoretical foundation to justify recent high-impact applications of IPF to infer dynamic networks and to rigorously motivate future uses of IPF for this and related problems.
Given the vast literature on IPF (Sinkhorn's algorithm) and matrix balancing,   
connecting this network inference problem to IPF also opens up future avenues of research, creating a bridge from modern data-driven problems to decades of statistical theory.

\section{Related Work}
\paragraph{Iterative proportional fitting.}
IPF \citep{deming1940ipf}, also known as Sinkhorn's algorithm, biproportional fitting, raking, and the RAS algorithm, has a long history across disciplines.
IPF offers many empirical advantages: it is straightforward to implement (see Algorithm~\ref{alg:ipf} in Appendix~\ref{sec:app-ipf}), transparent, reproducible, space efficient, and computationally lightweight \citep{liang2006traffic,lomax2015geographer,lovelace2015microsim}.
The algorithmic properties of IPF have also been extensively studied \citep{sinkhorn1964relationship,fienberg1970ipf,franklin1989scaling,pukelsheim2009ipf,qu2023sinkhorn}. 
Notably, IPF solves a KL divergence minimization problem \citep{bregman1967proof,ireland1968contingency,csiszar1975iprojection,ruschendorf1995convergence}, and is a coordinate descent type algorithm for its \emph{dual} problem \citep{luo1992convergence}. Recently, \citet{qu2023sinkhorn} further established connections between IPF and ML estimation of choice models. We leverage these observations to identify a Poisson network model whose MLEs are recovered by IPF (Section \ref{sec:model}).
In Appendix \ref{thm:Poisson-mle}, we discuss other settings where IPF is known to recover MLEs of distinct models, such as contingency tables \citep{bishop1969contingency,bishop1974discrete,little1991models,little1993poststrat}. The connections of IPF to Poisson-type models also have precedence in economics, such as trade \citep{silva2006log} and matching \citep{galichon2022cupid,galichon2022estimating}.
IPF (or Sinkhorn’s algorithm) is also closely connected to (discrete) entropy regularized optimal transport and Schr{\"o}dinger bridges, which can be reformulated as matrix balancing problems with a kernel/reference matrix and fixed marginals, and solved using the IPF procedure \citep{cuturi2013sinkhorn,marino2020optimal,de2021diffusion}.

Despite extensive literature on IPF's convergence behavior, few works have discussed principled solutions when IPF does \textit{not} converge, which may occur when there are many zeros in the initial matrix or marginals \citep{bishop1974discrete,wong1992geographer}.
The typical solution is to replace \emph{all} zeros with small positive values \citep{lomax2015geographer,lovelace2015microsim}. However, this approach may result in unrealistic positive entries or, in our setting, vastly alter the structure of the network.
To address this gap in literature, we propose a principled algorithm that guarantees IPF convergence while \emph{minimally} changing the network (\cref{sec:convergence}).

\vspace{-0.2cm}
\paragraph{Network inference.} The problem of inferring dynamic networks from marginals has been central to the impactful guidance of COVID-19 policy from mobility data \citep{chang2021nature,chang2021kdd,chaudhuri2022overdispersion,li2023sdm,alimohammadi2023working}, and also appears across many domains, including transportation \citep{carey1981method}, communication \citep{kruithof1937telefoonverkeersrekening}, and migration \citep{plane1982information,pham2022migrant}.
Prior works have also explored related but distinct network inference problems. 
When node-level signals are observed but the network is not known at all, temporal or spatial relationships are often used to infer which nodes are likely to be connected \citep{gomezrodriguez2012diffusion,hallac2017lasso,akagi2018ijcai,chen2021icml,rossi2022games}.
A closer setting to ours is where the structure of the graph and the node-level marginals are known \citep{kumar2015wsdm,maystre2017choicerank}.
Our setting is distinct since we study dynamic networks and have access to the time-aggregated network, not just the binary structure. Nevertheless, we note similarities, as both settings solve matrix balancing problems, with the ChoiceRank algorithm of \citet{maystre2017choicerank} closely connected to IPF as well \citep{qu2023sinkhorn}.

Our network inference problem is also related to collective graphical models (CGMs), which aim to fit a model of individual behavior from aggregate data \citep{sheldon2011cgm}.
Like CGMs, we also seek to estimate finer-grained information based on coarser data, but our setting differs in a number of ways: CGMs are typically applied in settings where only the time-varying marginals are known, such as to estimate population flows from counts per region over time \citep{iwata2017flow,akagi2018ijcai,iwata2019neural}; CGMs try to learn a model of \textit{individual} behavior, which is not our goal; and their mathematical formulations and estimation procedure are distinct from IPF.
We compare our work to CGMs in greater detail in Appendix \ref{thm:Poisson-mle}. 


\section{Biproportional Poisson Model: A Statistical Framework of IPF for Inferring Networks}
\label{sec:model}
In this section, we provide background on IPF and define our dynamic network inference problem, then propose a statistical framework---via our generative network model---that establishes IPF as a principled solution to this problem.
\vspace{-0.5cm}
\paragraph{Matrix balancing and IPF.} 
IPF is an iterative algorithm that 
seeks to solve the following \textit{matrix balancing problem}: 
\begin{quote}
    Given positive vectors $\ipfrow \in\mathbb{R}_{++}^{m}, \ipfcol \in\mathbb{R}_{++}^{n}$ with
    $\sum_i p_{i}=\sum_j q_{j}=c$ and an initial non-negative matrix $\ipfmat \in\mathbb{R}_{+}^{m\times n}$,
    find positive diagonal matrices $D^{0}$, $D^{1}$ satisfying the
    marginal conditions $D^{0} \ipfmat D^{1}\cdot\revised{\mathbf{1}_n}=\ipfrow$ and $D^{1} \ipfmat^{T} D^{0}\cdot\revised{\mathbf{1}_m}=\ipfcol$. 
\end{quote}
IPF learns the scaling factors $\ipfparamrow$ and $\ipfparamcol$, which are diagonals of $D^0,D^1$, by alternating between scaling the rows to match $\ipfrow$, then scaling the columns to match $\ipfcol$:
\begin{align}
    \ipfparamrow_i(k+1) &= \frac{\ipfrow_i}{\sum_j \ipfmat_{ij} \ipfparamcol_j(k)}, \quad \ipfparamcol_j(k+1) = \ipfparamcol_j(k), \label{eqn:ipf-update} \\ 
    \ipfparamcol_j(k+2) &= \frac{\ipfcol_j}{\sum_i X_{ij} \ipfparamrow_i(k+1)}, \quad \ipfparamrow_i(k+2) = \ipfparamrow_i(k+1). \nonumber 
\end{align}
We denote by $\ipfestmat(k)$ the scaled matrix after the $k$-th iteration: $\ipfestmat(k):= D^0(k) \ipfmat D^1(k)$.
The convergence behavior depends on the problem structure: if the matrix balancing problem has a finite solution, $(D^0(k),D^1(k))$ will converge to it; $(D^0(k),D^1(k))$ can diverge but $\ipfestmat(k)$ converges; or $\ipfestmat(k)$ does not converge and instead oscillates between accumulation points \citep{pukelsheim2009ipf}.
It is well-known that IPF solves the following KL divergence minimization problem, and IPF converges as long as the KL problem is feasible and bounded \citep{bregman1967proof,ireland1968contingency,leger2021gradient}:
\begin{align}
\label{eq:KL-minimization}
\min_{\hat \klmat}D_{\text{KL}}(\hat \klmat\|\ipfmat)\Leftrightarrow \min_{\hat \klmat}\sum_{ij}\hat \klmat_{ij}\log\frac{\hat \klmat_{ij}}{\ipfmat_{ij}},
\end{align}
subject to $\hat \klmat_{ij}\geq0$, $\hat 
\klmat\mathbf{1}_{n}=\ipfrow$, and 
$\hat \klmat^{T}\mathbf{1}_{m}=\ipfcol$.\footnote{Technically, the KL divergence $D_{\text{KL}}(P\|Q)$ is well-defined only on probability distributions $P,Q$. Without loss of generality, we may rescale $\ipfmat$ so that $\sum_{ij}\ipfmat_{ij}=\sum_ip_i=\sum_jq_j=c$. Then normalizing $\hat \klmat, \ipfmat$ by the common constant $c$ yields valid discrete probability distributions. The resulting KL minimization problem is equivalent to \eqref{eq:KL-minimization}.} 
\eqref{eq:KL-minimization} is feasible and bounded if and only if there exists $\hat \klmat$ with the desired marginals $\ipfrow$ and $\ipfcol$ and $\hat \klmat_{ij}=0$ whenever $\ipfmat_{ij}=0$.
Furthermore, the dual problem of \eqref{eq:KL-minimization} minimizes the potential function
\begin{align}
\label{eqn:dual}
g(u,v):=\sum_{ij}\ipfmat_{ij}e^{u_{i}-v_{j}}-\sum_{i}\ipfrow_{i}u_i+\sum_{j}\ipfcol_{j}v_{j},
\end{align}
which is jointly convex in the dual variables $\paramrow \in \mathbb{R}^m, \paramcol \in\mathbb{R}^n$.
IPF is a coordinate descent type algorithm for \eqref{eqn:dual} and $(\ipfparamrow,\ipfparamcol)$ is a solution to the matrix balancing problem if and only if $u = \log \ipfparamrow,v=-\log \ipfparamcol$ is a minimizer of \eqref{eqn:dual} \citep{luo1992convergence}.
For completeness, we provide details of the duality result in Appendix \ref{sec:dual}. The KL minimization problem \eqref{eq:KL-minimization} is also closely related to entropy regularized optimal transport, which we discuss in more detail in \cref{subsec:entropy-ot}.

\paragraph{Dynamic network inference problem.}
In our setting, we have a dynamic network with discrete time steps, where $\Xtime \in \mathbb{R}^{m \times n}_{+}$ represents the weighted adjacency matrix at time $t$.
We do not have access to $\Xtime$ due to privacy or sampling constraints, but we observe its row sums $\ptime:=\Xtime\mathbf{1}_n$ and column sums $\qtime:=(\Xtime)^T\mathbf{1}_m$, as well as a time-aggregated network, $\Xagg := \sum_{t=1}^T \Xtime$, for some large $T$.
The goal of the dynamic network inference problem is to provide a reasonable estimate of $\Xtime$, given $\ptime$, $\qtime$, and $\Xagg$. The correspondence between this problem and the matrix balancing problem is natural: we can treat $\Xagg$ as the initial matrix, $\ptime$ as the target row marginals, and $\qtime$ as the target column marginals. IPF's solution to this matrix balancing problem then serves as an estimate of the hourly network $\Xtime$. 
However, 
how should we interpret this estimate and why is IPF a justified approach here?

\paragraph{Our proposed network model.}
We provide a statistical justification of IPF as a solution to the network inference problem by identifying a generative network model under which IPF in fact recovers the MLEs of the network parameters.
Our model, which we term the \textit{biproportional Poisson model}, is defined as follows with parameters $u,v$:
\begin{align}
    \Xtime_{ij} &\sim \begin{cases}
        \mathrm{Poisson}(e^{\paramrow_i} \Xagg_{ij} e^{-\paramcol_j})\textrm{, if }\Xagg_{ij}>0,\\
        0\textrm{, otherwise.}
    \end{cases} \label{eqn:model} \\
    \ptime_i &= \sum_j \Xtime_{ij}, \quad
    \qtime_j = \sum_i \Xtime_{ij}. \nonumber
\end{align}
Model \eqref{eqn:model} posits \emph{independent Poisson} samples $\Xtime_{ij}$ with expected value $\lambda_{ij}=e^{\paramrow_i}\Xagg_{ij}e^{-\paramcol_j}$ wherever $\Xagg_{ij} > 0$ (we suppress time-indexing on $\paramrow$ and $\paramcol$ to simplify notation). Note that the model is unique up to normalization, e.g., $(u,v)^T\mathbf{1}_{m+n}=0$, since adding a constant $c$ to $(u,v)$ yields the \emph{same} model. 
Since $\Xtime$ is not observed, it is \emph{a priori} not obvious if the model parameters can be recovered by maximum likelihood estimation. 
Our first result reveals that although only $\Xagg$, $\ptime$, and $\qtime$ are observed, MLEs of $u,v$ are well-defined, since $\ptime$ and $\qtime$ form the sufficient statistics of model \eqref{eqn:model}. 
More importantly, our result connects IPF to \eqref{eqn:model}, by showing that its MLEs are exactly the IPF solution to the matrix balancing problem, with \eqref{eqn:dual} equivalent to the negative log-likelihood. This connection also partially inspired the name ``biproportional Poisson'', as an alternative name of IPF is ``biproportional fitting'' \citep{bacharach1965estimating}.
\begin{theorem}
\label{thm:model}
Assume that the matrix balancing problem with $\Xagg$, $\ptime$, and $\qtime$ has a finite solution $(D^0,D^1)$. Then $\ipfparamrow$ and $\ipfparamcol$ are limits of the IPF iterations if and only if $\hat{\paramrow} = \log \ipfparamrow$ and $\hat{\paramcol} = -\log \ipfparamcol$ are solutions to the maximum likelihood estimation problem of \eqref{eqn:model} given $\Xagg$, $\ptime$, and $\qtime$, with log-likelihood $\ell\equiv -g(u,v)$ in \eqref{eqn:dual} modulo constants:
\begin{align}
\label{eqn:log-likelihood}
\ell(u,v)=\sum_{i}\ptime_{i}u_i-\sum_{j}\qtime_{j}v_{j}-\sum_{ij}\Xagg_{ij}e^{u_{i}-v_{j}}.
\end{align}
Moreover, maximizing $\ell(u,v)$ is equivalent to the maximum likelihood estimation of a Poisson regression model, with $\ptime,\qtime$ as the sufficient statistics.
\end{theorem}

We prove \cref{thm:model} in Appendix \ref{thm:Poisson-mle}. 
The core of our result lies in identifying a model whose log-likelihood is equivalent to $-g(u,v)$, since IPF minimizes $g(u,v)$ in \eqref{eqn:dual}. Our result is closely related to \citet{qu2023sinkhorn}, who recently established connections between IPF and choice modeling, and observed that \eqref{eqn:dual} reduces to the maximum likelihood objective of a general class of choice models.

Notably, our theorem does not require assumptions about how the network evolves over time, since our model includes scaling factors $u$ and $v$ per time step, which IPF directly estimates. 
However, performing network inference in this decoupled fashion potentially leaves out additional information, which is that $\Xagg = \sum_{t=1}^T \Xtime$.
We show in Appendix \ref{sec:joint} that performing network inference with this constraint 
reduces to the decoupled problems 
under the following mild stationarity assumption for some constant $c$:
\begin{align}
    \sum_{t=1}^T e^{u_{i}(t)-v_j(t)} \approx c, \label{eqn:stationarity}
\end{align}
for all $i,j$ where $\Xagg_{ij} > 0$. 
We also verify that the stationarity assumption approximately holds on real-world data (Appendix \ref{sec:test-assumptions}), thus justifying the decoupled approach. 
As two additional results, in Appendix~\ref{sec:thm2-proof}, we prove necessary and sufficient conditions for IPF to recover $\Xtime$ exactly, and in Appendix~\ref{sec:app-uniqueness}, we prove that among a larger class of generalized linear models, the Poisson model is the \emph{unique} one where the MLE is the IPF solution --- so from the perspective of IPF, our generative model is canonical. 

\paragraph{Implications of our result.}
Our model allows us to interpret IPF through the lens of a generative network model.
Since, in relevant applications, $\Xtime_{ij}$ is often the number of visits from node $i$ to $j$, \eqref{eqn:model} is consistent with queuing theory where rare events are modeled with a Poisson process due to the memory-less property. In addition, in the Poisson parameters $e^{\paramrow_i}\Xagg_{ij}e^{-\paramcol_j}$, $u_i$ can be interpreted as the emission intensity of node $i$, and $-v_j$ as the absorption intensity of node $j$. 
For example, in mobility networks between residential neighborhoods and public places \citep{chang2021nature}, $u$ captures when each neighborhood is likelier to go out (e.g., younger populations more at night) and $-v$ captures each place's visit propensity (e.g., schools visited more during the day while bars visited more at night).
Notably, the biproportional form of the model assumes that there are not time-varying interactions between rows and columns (e.g., if a place offers special discounts for seniors at this time, attracting neighborhoods with large senior populations), making explicit one of the key assumptions of using IPF in this network inference setting.
Moreover, since $\ipfparamrow_i \Xagg_{ij}\ipfparamcol_j$ from IPF corresponds to $e^{ \paramrow_i} \Xagg_{ij} e^{-\paramcol_j}$, we can interpret the matrix inferred by IPF as estimating the \textit{expected values} of the network-generating process.

We also note connections of our model to (pseudo-)Poisson maximum likelihood regression in economics \citep{gourieroux1984pseudo}, including for models of trade \citep{silva2006log} and matching \citep{galichon2022cupid,galichon2022estimating}. In particular, \cref{thm:model} implies that estimating the biproportional model is equivalent to solving a pseudo-Poisson maximum likelihood problem. Consequently, the results of \citet{gourieroux1984pseudo} guarantee that the minimizers of \cref{eqn:dual} are consistent estimators of $u,v$ even under \emph{misspecifications} of the distribution and homoskedasticity of $\Xtime_{ij}$ given $\lambda_{ij}=e^{\paramrow_i}\Xagg_{ij}e^{-\paramcol_j}$, highlighting the robustness of our biproportional Poisson model. 


Defining an explicit model also yields several advantages.
First, the equivalence between the IPF solution and this model's MLE enables us to analyze IPF estimates using tools from statistical theory. 
In Section \ref{sec:stats-theory}, we develop bounds on the MLE's estimation error and establish that finite MLEs exist with high probability. Second, our model clarifies previously implicit assumptions when using IPF to infer dynamic networks, such as the stationarity assumption \eqref{eqn:stationarity} or the lack of time-varying interactions.
Making such assumptions explicit allows practitioners to evaluate how reasonable the assumptions are given their domain and data; for example, we test several model assumptions on real-world bikeshare data (Appendix \ref{sec:test-assumptions}). 
Third, defining an explicit model reveals natural ways to extend the model, such as non-Poisson distributions or interaction terms between rows and columns.
These extensions allow us to test IPF under model misspecification (Appendix \ref{sec:synthetic-misspecification}) and create future opportunities for studying how changes in the model map back to changes in IPF. 
Finally, the model enables new empirical analyses of IPF (Section \ref{sec:empirics}), such as quantifying uncertainty in the parameter estimates and evaluating IPF's estimation of the network parameters, instead of only evaluating IPF's error on the marginals, which is how IPF tends to be evaluated \citep{lovelace2015microsim}.

\section{Statistical Theory of Biproportional Poisson}
\label{sec:stats-theory}
We have shown in \cref{thm:model} that when the matrix balancing problem has finite solutions, IPF recovers them as MLEs of our biproportional Poisson model.
Two important questions remain: how ``good'' are these MLEs, in terms of their estimation error relative to the true model parameters, and how often can we guarantee that the MLEs, solutions to the matrix balancing problem, are finite?  
In this section, we develop statistical theory to answer these questions.

\subsection{Structure-dependent MLE Error Bounds}
Intuitively, the more ``well-connected'' a biproportional Poisson network, the better quality its MLEs. 
We quantify network connectivity through the \emph{Fiedler eigenvalue} \citep{fiedler1973algebraic}, which is the second-smallest eigenvalue $\lambda_{-2}(\mathcal L)$ of the graph Laplacian $\mathcal{L}:=\mathcal{D}(A\mathbf{1})-A$, where $A$ is the (weighted) adjacency matrix and $\mathcal{D}(\cdot)$ denotes the diagonalization of a vector.
In this paper, $\mathcal{L}$ is the graph Laplacian of the \emph{weighted bipartite} graph $G_b$ induced by $\Xagg$ with $A :=\begin{bmatrix}0 & {\Xagg}\\
{\Xagg}^{T} & 0\end{bmatrix}$. 
The Fiedler eigenvalue is frequently used in graph theory
and distributed optimization to measure graph connectivity \citep{spielman2012spectral}. 
In recent years, its importance for the algorithmic and statistical efficiencies of Luce choice model estimation \emph{vis-{\`a}-vis} the topology of comparison structures has been extensively studied \citep{shah2015estimation,vojnovic2016parameter,seshadri2020learning,vojnovic2020convergence,hendrickx2020minimax,bong2022generalized}. It is important to note 
that the Fiedler eigenvalue used in our paper is based on a \emph{different} graph 
than 
in the choice literature, which are constructed from choice data and are not bipartite. 
An exception is \citet{qu2023sinkhorn},
who provided linear convergence analyses of IPF
quantified by the \emph{same} $\lambda_{-2}(\mathcal{L})$ as in this paper.

We now provide error bounds for normalized MLEs of the biproportional Poisson model quantified by the Fiedler eigenvalue, 
whenever the true parameters and MLEs are bounded by some constant $B$.  
\begin{theorem}
\label{thm:mse}
Suppose that the biproportional Poisson model \eqref{eqn:model} holds with ground truth parameters $u^*,v^*$.
Suppose $(\hat u,\hat v)$ is a maximizer of the log-likelihood \eqref{eqn:log-likelihood}
and that we have the normalization condition $(\hat u - u^*, \hat v - v^*) \in \mathbf{1}_{m + n}^{\perp}$ and $\|(\hat u,\hat v,u^*,v^*)\|_{\infty} \le B$. Then with $\kappa = \sum_{ij} e^{u^\ast_i} \Xagg_{ij} e^{-v^
\ast_j}$ the total rate, 
 in expectation we have
 \vspace{-0.2cm}
\begin{equation} 
\mathbb E\left[\|(\hat u - u^*, \hat v - v^*)\|^2 \mathbbm{1}_{\mathcal B}\right] \le \frac{8e^{4B} \kappa}{\lambda_{-2}(\mathcal L)^2},
\label{eq:expected-risk-bound} \end{equation}
where $\mathbbm{1}_{\mathcal B}$ is the event that the MLE is bounded above by $B$. 
\end{theorem}
A similar tail bound is given and proven in \cref{thm:bd-apdx}. The dependence of the error bounds on the Fiedler eigenvalue $\lambda_{-2}(\mathcal L)$ is clear: the more well-connected the bipartite graph $G_b$ (induced by $\Xagg$), the larger ${\lambda_{-2}(\mathcal L)}$, which in turn improves the estimation quality. Since $\lambda_{-2}(\mathcal L)>0$ if and only if $G_b$ is connected, the bounds blow up when $G_b$ becomes disconnected. Indeed, the biproportional Poisson model is not identified in this case (see \cref{rem:error-bound}). Analytically, $\lambda_{-2}(\mathcal L)$ also quantifies the \emph{strong concavity} of the log-likelihood function \eqref{eqn:log-likelihood}, which provides additional justification for its prominence in the bounds. 

Given recently established connections between choice modeling and matrix balancing and the well-known structure-dependent estimation error bounds in the choice literature \citep{shah2015estimation,seshadri2020learning,hendrickx2020minimax}, our error bounds in \cref{thm:mse} can be viewed as natural analogs of such results for the biproportional Poisson model. As with those results, the assumption that the ground truth parameters and the MLEs are bounded by some $B$ is standard. However, we note that since the dimension of observations $\ptime,\qtime$ is on the same order as the parameter dimension (both $m+n$), the biproportional Poisson model corresponds to a high-dimensional setting by design. Consequently, unlike bounds in the choice literature, there is no explicit dependence on the ``sample size'' in our bounds. We provide a simple example to illustrate this point. 

\textbf{Example with accurate recovery: complete graph.} Suppose that $\Xagg$ is the $m \times n$ all-ones matrix. Then $\kappa = \Theta(nm)$ and $\lambda_{-2}(\mathcal L) = \min(n,m)$ so the right hand side of \eqref{eq:expected-risk-bound} is of order $\Theta(\max(n/m,m/n))$. So if e.g. $n = m$, then the total error for recovering the \emph{entire vector} $(u^\ast,v^\ast)$ is $\mathcal{O}(1)$. Equivalently, the average error \emph{per coordinate} of $u^\ast,v^\ast$ is $\frac{1}{2n} \|(\hat u - u^*,\hat v - v^*)\|^2 = \mathcal{O}(1/n)$. 

\textbf{Bound improves with growing SNR.} A key feature of the error bound is that it improves by a factor of $1/c$ when the base matrix $\Xagg$ is \emph{scaled up} by a constant $c>1$, since both $\kappa$ and $\lambda_{-2}(\mathcal L)$ are scaled by $c$. This feature can be understood as a result of improved ``signal-to-noise'' ratio (SNR), 
since the Poisson rate $e^{u^\ast_i} \Xagg_{ij} e^{-v^\ast_j}$ scales with $\Xagg$.

\textbf{Impact of sparsity.} As graph connectivity is related to the sparsity of the graph, our results can inform us on how the sparsity of the network impacts estimation quality of IPF. In \cref{sec:synthetic-sparsity}, we evaluate the quality of MLEs with synthetic data under different sparsity rates $r\in [0,1)$. One implication of \cref{thm:mse} is that, with high probability, 
\[\|(e^{\hat{u}} - e^{u^\ast}, e^{-\hat{v}} - e^{-v^\ast})\|_2 = \mathcal{O}\left(\frac{1}{\sqrt{1 - r}}\right),\]
which matches the observed deterioration of estimation quality in Figure \ref{fig:ipf-sparsity} as sparsity $r$ increases from 0 to 1.

In \cref{rem:error-bound}, we provide more discussion on the error bounds in \cref{thm:mse}, including optimality of the dependence on $\kappa$ and necessity of ${\lambda_{-2}(\mathcal L)}$ and $e^{\Theta(B)}$.
Despite the usefulness of the error bounds we develop in \cref{thm:mse}, an important technical question remains on the existence of \emph{finite} MLEs. Even under correct specification, a (random) dataset may not yield a well-defined maximum likelihood problem, i.e., no finite maximizer of \eqref{eqn:log-likelihood} exists. We next address this issue and identify a sufficient condition that guarantees that finite MLEs exist with high probability.

\subsection{Well-posedness of Maximum Likelihood Estimation}
Under the biproportional Poisson model \eqref{eqn:model}, the maximum likelihood estimation problem on $\Xagg,\ptime,\qtime$ may not have a finite solution. 
As is well-known and discussed in \citet{qu2023sinkhorn}, the corresponding matrix balancing problem has a finite solution if and only if there exists a matrix with \emph{exactly} the same zero patterns as $\Xagg$ and has marginals $\ptime, \qtime$. Under correct specification of the biproportional Poisson model, the matrix $\Xtime$ has the right marginals $\ptime,\qtime$. If all Poisson entries $\Xtime_{ij}>0$, $\Xtime$ provides a certificate for the existence of a finite solution to the matrix balancing problem hence finite MLEs. 
However, as soon as any $\Xtime_{ij}$ equals 0, $\Xtime$ will have additional zero entries than $\Xagg$, and the bipartite network induced by $\Xtime$ could become disconnected. However, there could still exist \emph{another} matrix that solves the matrix balancing problem, yielding a finite MLE. Our task is to show that this happens often.

A similar challenge exists in the choice setting. When some subset of items are always preferred over its complement, no positive MLE exists in the Luce choice modeling framework. Recently, \citet{bong2022generalized} provided a simple sufficient condition on the Fisher information matrix in the Bradley--Terry--Luce model which guarantees that this event rarely happens. Our next theorem can be viewed as an analog of their result for the biproportional Poisson model.

\begin{theorem}
    \label{thm:finite-mle}
    Let $\mathcal{L}^\ast:= -\nabla^2 \ell(u^\ast,v^\ast)$ be the Hessian of the negative log-likelihood in \eqref{eqn:log-likelihood} evaluated at the true parameters $(u^\ast,v^\ast)$, and suppose its second smallest eigenvalue satisfies
    \begin{align}
\label{eq:Fisher-bound}
        \lambda_{-2}(\mathcal{L}^\ast)\geq 8{\log (m+n)}.
    \end{align} Then with probability at least $1-2/\sqrt{m+n}$, the maximum likelihood estimation of the model \eqref{eqn:model} has a unique normalized finite solution, and IPF converges to this solution. 
\end{theorem}
Note that the sufficient condition \eqref{eq:Fisher-bound} is stated in terms of the Hessian evaluated
at the \emph{true} parameters $(u^{\ast},v^{\ast})$, which is essentially
the Fisher information matrix modulo a constant factor. When is \eqref{eq:Fisher-bound} satisfied? As discussed before, for complete graphs $\lambda_{-2}(\mathcal{L}^\ast) = \mathcal{O}(\min\{m,n\})$, so \eqref{eq:Fisher-bound} is satisfied. As another example, an Erd{\"o}s-R{\'e}nyi graph with parameter $p$ of order $\Omega(\frac{\log(m+n)}{m+n})$ also satisfies \eqref{eq:Fisher-bound} with high probability. See for example \citet{bong2022generalized}. 


Our results in this section provide useful insights about the MLEs of our biproportional Poisson model.
However, IPF convergence is a prerequisite for IPF to recover these MLEs. 
Now, we move onto the next natural and important question: what should one do in practice if IPF does \textit{not} converge?

\section{Guaranteeing IPF Convergence}
\label{sec:convergence}

IPF non-convergence tends to occur when there are many zeros in the inputs 
\citep{wong1992geographer}, and such sparsity is very common in real-world network data, partially due to missing values.
In fact, we show that, when trying to infer $\Xtime$ from $\ptime$, $\qtime$, and $\Xagg$, for some aggregated time period that includes $t$, IPF will not converge \textit{only if} there are missing entries in the inputs (Corollary \ref{cor:timeagg-converge}).
For example, in mobility data, true visits may be missed due to noisy GPS signals, populations not carrying cell phones \citep{coston2021facct}, or data ``clipping'' where low values are replaced with zeros to preserve privacy \citep{safegraph-patterns}.
We discuss these mechanisms in depth in Appendix \ref{sec:app-safegraph}.

Thus, we approach IPF non-convergence from the perspective of missing data: specifically, in initial matrix $X$, and we resolve non-convergence by adding edges to $X$.
This is a similar view to the typical solution for IPF non-convergence, which is to replace all zeros in $X$ with very small amounts \citep{lomax2015geographer}. 
However, the typical approach may result in unrealistic positive entries, such as drivers under the age of 16 when inferring joint demographics \citep{lovelace2015microsim} or nonexistent routes in transportation networks.
Furthermore, in our setting, replacing all zeros with positive entries completely alters the sparsity structure of the inferred network, which can greatly affect downstream results, such as modeling epidemics \citep{wang2003eigenvalue}.
Instead, we introduce a new algorithm, \texttt{ConvIPF}, that guarantees IPF convergence by adding edges while minimizing changes to the network structure, where we explore two different definitions of network change below.
Even though we focus on network inference, \texttt{ConvIPF} can be used in any application of IPF where non-convergence may be caused by missing values in the initial matrix $\ipfmat$.

\paragraph{Overview of \texttt{ConvIPF}.}
Two equivalent conditions that define when IPF converges are \citep{pukelsheim2014}:
\begin{enumerate}[nolistsep]
    \item There exists a matrix $Y$ with row sums $\ipfrow$ and column sums $\ipfcol$ such that $Y_{ij} = 0$ wherever $X_{ij} = 0$.
    \item For all row subsets $S \subseteq [m]$, $\sum_{i \in S} \ipfrow_i \leq \sum_{j \in N_X(S)} \ipfcol_j$, where $N_\ipfmat(S)$ represents the set of columns connected to $S$ in $\ipfmat$. 
\end{enumerate}
Condition (1) yields an efficient algorithm, which we call \texttt{MAX-FLOW}, for testing whether IPF will converge.
The algorithm, as described in \citet{idel2016review} and Appendix \ref{sec:convergence-test}, requires one round of max-flow on a graph constructed from the IPF inputs. 
If the resulting flow is equal to $\sum_i \ipfrow_i$, then IPF converges. 
Condition (2) is useful because it allows us to diagnose why IPF is not converging: if IPF does not converge, there must be at least one ``blocking set'' of rows for which the condition is violated.

The key idea of \texttt{ConvIPF} is that we can unblock a blocking set of rows by adding new edges in $\ipfmat$ for those rows, 
but we seek edge additions that modify $\ipfmat$ \textit{as little as possible}.
After we unblock one blocking set, there may be more.
So, \texttt{ConvIPF} iteratively identifies a blocking set and modifies $\ipfmat$ to unblock it, until there are no blocking sets remaining.
Our algorithm thus repeats three subroutines, \texttt{MAX-FLOW}, \texttt{BLOCKING-SET}, and \texttt{MODIFY-X}, until IPF converges:
\begin{enumerate}[nolistsep]
    \item Run \texttt{MAX-FLOW} to test for convergence. If IPF converges, then the algorithm is finished. If IPF does not converge, move on to Step 2.
    \item Since IPF does not converge, run \texttt{BLOCKING-SET} to identify a blocking set of rows, $S$.
    \item Run \texttt{MODIFY-X} to unblock $S$ by minimally adding edges to $\ipfmat$.
\end{enumerate}

\texttt{ConvIPF} must terminate since (i) it is always possible to unblock a row set (by connecting it to all columns) and (ii) an unblocked row set cannot become blocked through subsequent edge additions.
Furthermore, even though there are $2^m$ possible row subsets, the algorithm will terminate within $mn$ iterations, since IPF converges when all entries in $\ipfmat$ are positive \citep{pukelsheim2014} and each \texttt{MODIFY-X} adds at least one positive entry to $\ipfmat$.
Thus, as long as each subroutine runs in polytime, the entire algorithm runs in polytime.
However, \texttt{BLOCKING-SET} and \texttt{MODIFY-X} both try to solve combinatorial problems, so the challenge is how to efficiently solve each one.
We provide brief sketches of each in this section, with details in Appendix \ref{sec:convergence-algo}.

\paragraph{\texttt{BLOCKING-SET}: identifying a blocking set.}
Given inputs $\ipfmat$, $\ipfrow$, and $\ipfcol$, where we know IPF does not converge, this subroutine identifies a blocking set of rows $S$ for which Condition (2) is violated.
The naive approach to iterate through all subsets until a violation is found, but this approach is extremely inefficient, as there are $2^m$ possible subsets.
Instead, our subroutine imports ideas from constricted sets in bipartite matching \citep{hall1935subset,easleykleinberg} to design a much more efficient algorithm. 
First, construct a bipartite graph $B$ where the nodes are the rows and columns of $\ipfmat$ and they are connected wherever $\ipfmat_{ij} > 0$.
From running \texttt{MAX-FLOW} to test for convergence, we have flow values for each row/column.
Since IPF does not converge, there must be at least one row $i$ whose flow is less than its capacity, $\ipfrow_i$.
Run the following variant of breadth-first search (BFS) on $B$ starting from node $n_i$.
When progressing from a column node $n_C$ to its neighboring row node $n_R$, only include row nodes where $n_R \rightarrow n_C$ has non-zero flow in \texttt{MAX-FLOW}.
When progressing from row nodes to column nodes, include all (unvisited) neighbors.
When BFS terminates, the set of row nodes visited forms a blocking set, which we prove in Appendix \ref{sec:app-blocking-set}.

\paragraph{\texttt{MODIFY-X}: unblocking a blocking set.}
\label{sec:modify-x}
Given a blocking set $S$, this subroutine minimally adds edges to $\ipfmat$ to unblock $S$. 
Let $\ipfmat^K$ represent the modified $\ipfmat$ after adding new edges $K$ and let $f(\ipfmat, \ipfmat^K)$ represent the change in $\ipfmat$ that we are trying to minimize.
Then, our goal is to find $K^*$ that minimizes $f(\ipfmat, \ipfmat^{K^*})$, subject to $S$ being unblocked under $\ipfmat^{K^*}$, i.e., $\sum_{i \in S} p_i \leq \sum_{j \in N_{\ipfmat^{K^*}}(S)} q_j$.
We consider two natural definitions of $f(\ipfmat, \ipfmat^K)$.

\textit{1. Number of new edges.}
Let $\bar{N}_\ipfmat(S)$ represent the set of columns \textit{not} connected to $S$ in $\ipfmat$, and let $\delta$ represent the gap in marginals, $\delta := \sum_{j \in N_\ipfmat(S)} \ipfcol_j - \sum_{i \in S} \ipfrow_i$.
Take the top-$k$ columns in $\bar{N}_\ipfmat(S)$, ordered by $\ipfcol_j$ in descending order, that satisfy $\sum_{j=1}^k \ipfcol_j \geq \delta$.
Then, any set of edges between a row in $S$ and these $k$ columns will unblock $S$, while minimizing the number of new edges added.

\textit{2. Change in $\lambda_1$.}
A more nuanced objective minimizes the change in the relevant spectral properties of $\ipfmat$.
Motivated by the application of epidemic spread, recall that the epidemic threshold of a network is closely related to the largest eigenvalue $\lambda_1$ of its adjacency matrix \citep{wang2003eigenvalue} and attempts to reduce spreading aim to minimize $\lambda_1$ through edge removals \citep{saha2015radius,li2023sdm}.
So, to preserve a ``similar'' network from a spreading standpoint, we seek to minimize change in $\lambda_1$.
We show in Appendix \ref{sec:app-modify-x} that, with reasonable approximations of change in $\lambda_1$ \citep{tong2012gel}, unblocking $S$ while minimizing change in $\lambda_1$ reduces to the following: first, find the row $i^*$ in $S$ with the smallest $\vec{u}_1(i)$, where $\vec{u}_1$ and $\vec{v}_1$ are the left and right eigenvectors of $\lambda_1(\ipfmat)$, respectively.
Then, solve an integer linear program to find the set of columns $J \subseteq \bar{N}_\ipfmat(S)$ that minimize $\sum_{j \in J} \vec{v}_1(j)$, subject to $\sum_{j \in J} \ipfcol_j \geq \delta$.
The set of new edges is $\{(i^*, j) | j \in J\}$, which unblocks $S$ and approximately minimizes change in $\lambda_1$.
\section{Experiments with Data}
\label{sec:empirics}
We now summarize our experiments with synthetic and real-world data. Our experiments reveal the utility of our theoretical and algorithmic contributions, and demonstrate IPF's capability to infer networks in practice. 
\begin{figure}
    \centering
    \includegraphics[width=\linewidth]{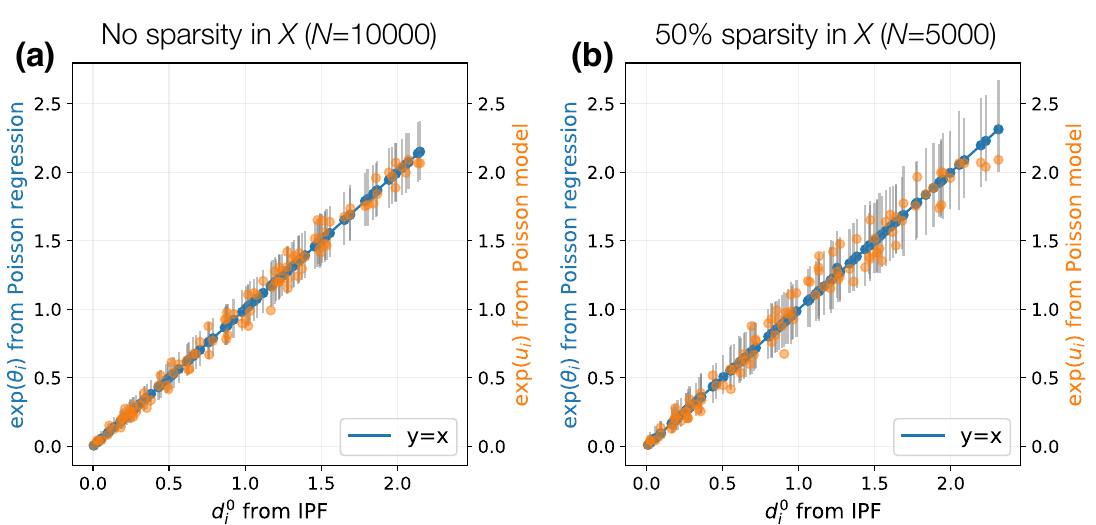}
    \caption{Comparing inferred parameters from IPF (x-axis) against inferred parameters from Poisson regression (left y-axis, blue) and true parameters from Poisson model (right y-axis, orange). 
    Grey bars indicate 95\% CIs from Poisson regression.
    $N$ represents the number of nonzero entries in $X$, so $N$ is halved with 50\% sparsity.
    Under both networks, estimated parameters from IPF and Poisson regression are perfectly aligned (Theorem \ref{thm:model}), but their estimation quality worsens with greater sparsity (Theorem \ref{thm:mse}).
    }
    \label{fig:synthetic-ipf-poisson}
\end{figure}

\paragraph{Testing IPF with synthetic data.} 
In our first set of experiments, we use synthetic data generated from our network model \eqref{eqn:model} to confirm the correspondence between IPF and Poisson regression.
In Figure \ref{fig:synthetic-ipf-poisson}, we show that the parameters inferred by IPF and Poisson regression align perfectly, validating Theorem \ref{thm:model}.
Furthermore, our Poisson model enables us to quantify uncertainty in the parameter estimates under the model, adding valuable interpretability missing from IPF output alone. 
We display 95\% confidence intervals (CIs) in Figure \ref{fig:synthetic-ipf-poisson}; while these CIs are only asymptotically valid, they provide useful measures of uncertainty. 

\begin{figure*}
    \centering
    \includegraphics[width=0.75\linewidth]{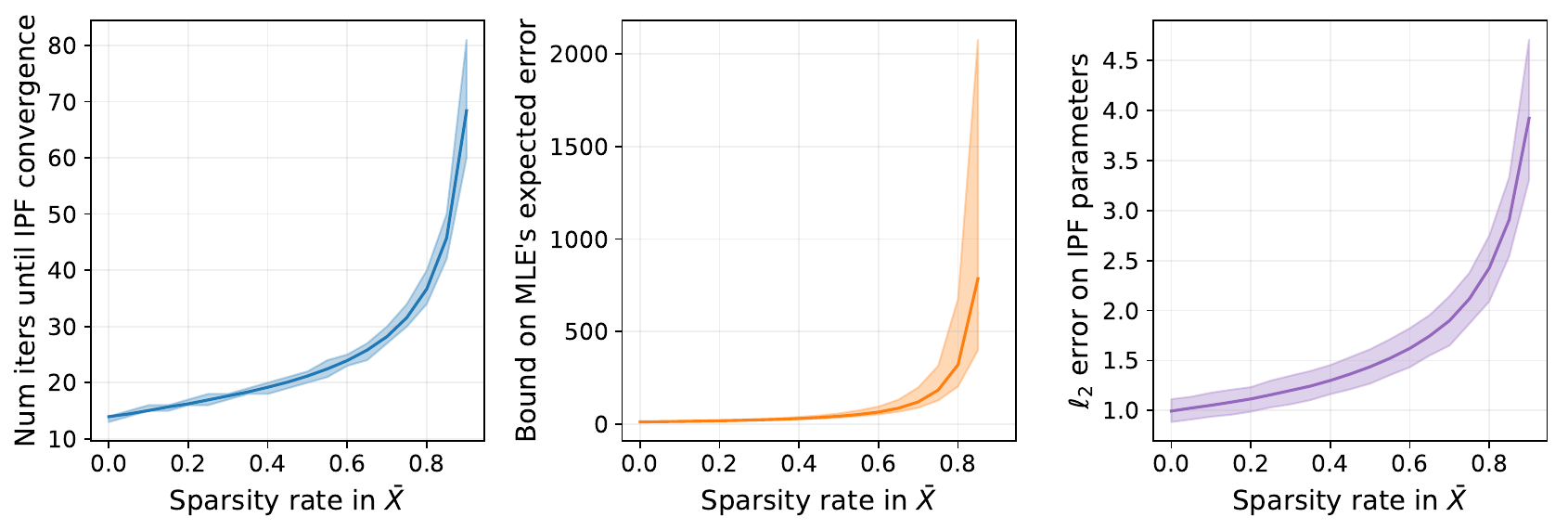}
    \caption{Comparing sparsity rate in $\Xagg$ to number of IPF iterations (left), bound on MLE's expected estimation error, without constants (middle), and observed $\ell_2$ error of IPF estimates (right). Lines represent mean and shaded region represents 95\% CIs over 1000 trials.}
    \label{fig:ipf-sparsity}
\end{figure*}
We also test IPF in more difficult settings, such as with increased sparsity in $\Xagg$.
We generate $\Xagg$ with a given sparsity rate $r$ by randomly selecting $r \cdot mn$ entries and setting them to 0, and test IPF on varying sparsity rates.
As expected, we find that the widths of the CIs increase with greater sparsity (Figure \ref{fig:synthetic-ipf-poisson}b), since greater sparsity results in fewer Poisson observations to fit our model \eqref{eqn:model}.
We also find that, despite matching the target marginals in all cases, the $\ell_2$ distance between IPF's inferred parameters and the true parameters increases quickly as we increase the sparsity in $\Xagg$ (Figure \ref{fig:ipf-sparsity}, right).
These findings align with Theorem \ref{thm:mse}, where we showed that the bound on the MLEs' expected estimation error \eqref{eq:expected-risk-bound} improves as $\Xagg$ becomes more well-connected.

We also use synthetic data to test IPF under model misspecification and find that IPF is reasonably robust to model modifications in this network inference setting (Appendix \ref{sec:synthetic-misspecification}). 
When the model is correctly specified (i.e., data is generated from our biproportional Poisson model), the cosine similarity between the true network $\Xtime$ and IPF estimate of the network $\estXtime = D^0 \Xagg D^1$ is, on average, $0.911$ (Table \ref{tab:ipf-non-poisson}). 
If we replace the Poisson with an exponential distribution, the cosine similarity only decreases to $0.855$; if we replace it with a negative binomial distribution ($\gamma = 0.5$), it decreases to $0.843$.
IPF is similarly robust when we test it on data from an ``interaction model'' \eqref{eqn:interaction}, which allows $\Xtime_{ij}$ to additionally depend on interaction terms between rows and columns, such as distance (Figure \ref{fig:interaction-model-cosine}).

\paragraph{IPF convergence on mobility data.}
To test IPF convergence, we use mobility data from SafeGraph \citep{safegraph-patterns}.
Here, we seek to infer the hourly visit network from neighborhoods to points-of-interest (POIs), so the marginals represent hourly total visits from neighborhoods and to POIs,  and $\Xagg$ represents the time-aggregated network. 
SafeGraph data provides a real-world example where only the hourly marginals and time-aggregated network are provided, with missing data in the time-aggregated network due to underreporting and data clipping (see Appendix \ref{sec:app-safegraph}).
We use mobility data from the Richmond metro area in Virginia, which has 9917 POIs and 1098 neighborhoods \citep{chang2021kdd}.
Despite aggregating over 10 months (January to October 2020), $\Xagg$ remains sparse: only 8\% of its entries are non-zero. 
From running IPF on two days (48 hours), we find that IPF does not converge for three nighttime hours when POI marginals are particularly sparse.


Using one of these hours---2AM on March 2, 2020---as an example, we apply our algorithm \texttt{ConvIPF} and evaluate the change in  $\Xagg$. 
We compare our solutions to the typical solution for IPF non-convergence, which is to replace all zeros with a very small value $\epsilon$ \citep{lomax2015geographer,lovelace2015microsim}.
If minimizing the number of new edges added, \texttt{ConvIPF} only adds $2$ edges, while the typical solution results in $10012193$ new edges.
If minimizing the change in $\lambda_1$, \texttt{ConvIPF} only changes it by $5.60 \cdot 10^{-9}$, while the typical solution changes $\lambda_1$ by $30.04$.
The magnitude of reduction achieved by \texttt{ConvIPF} is similar for the other two hours (Table \ref{tab:convergence-results}).
Furthermore, we find that the choice of objective in \texttt{ConvIPF} makes a difference:  choosing to minimize the number of edges always results in the fewest number of edges added and choosing to minimize the change in $\lambda_1$ always results in the smallest change in $\lambda_1$, by an order of magnitude compared to the other \texttt{ConvIPF} variants.
Finally, \texttt{ConvIPF} terminates quickly in practice, typically requiring 5 or fewer iterations.

\paragraph{Ground-truth networks from bikeshare data.}
Using data from New York City's Citibike system, we can construct \textit{ground-truth} hourly networks that record the number of bike trips between stations.
Acquiring ground-truth networks enables us to test IPF's ability to infer these networks, compared to baselines, which assesses the downstream utility of our work in bridging IPF to this network inference problem.
We test IPF's inferred networks, given hourly marginals and time-aggregated networks at the month-, week-, and day-level. 
As a baseline, we try the classic gravity model, which assumes that the amount of travel between two regions is proportional to their inverse distance \citep{zipf1946intercity,erlander1990gravity}.
We fit the ``doubly constrained gravity model'' \citep{navick1994distance}, which is equivalent to running IPF on a distance matrix, instead of the time-aggregated network, and the hourly marginals (Appendix \ref{sec:eval-ground-truth}).
This baseline enables us to test how much additional information the time-aggregated network provides, beyond what is captured by common geographical assumptions.
We also test other baselines:
\begin{itemize}[nolistsep]
    \item An ablation that removes $\Xagg$ and distributes $\sum_i \ptime_i$ proportional to $\ptime_i \qtime_j$;
    \item An ablation that removes $\qtime$ and distributes $\ptime_i$ within each row proportional to $\Xagg_{ij} / \sum_j \Xagg_{ij}$;
    \item An ablation that removes $\ptime$ and distributes $\qtime_j$ within each column proportional to $\Xagg_{ij} / \sum_i \Xagg_{ij}$.
\end{itemize}
\begin{figure}
    \centering
    \includegraphics[width=0.99\linewidth]{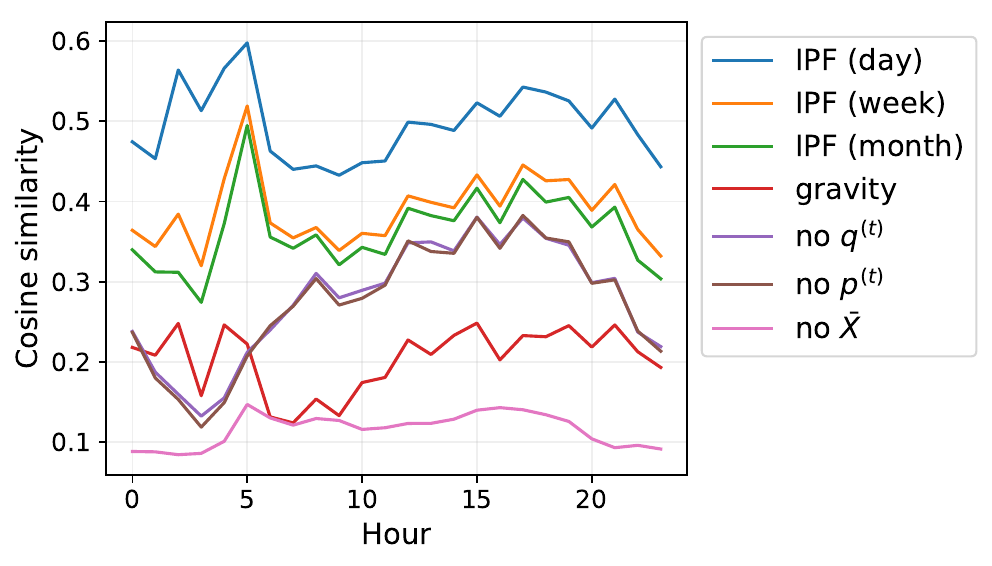}
    \caption{Cosine similarity between ground-truth hourly networks from bikeshare data and inferred networks from IPF and baselines.}
    \label{fig:ipf-bikeshare}
\end{figure}
We present main results in Figure \ref{fig:ipf-bikeshare}. 
First, we find that IPF strongly outperforms the gravity model, with a $78\%$ improvement in cosine similarity even when using the month-aggregated network.
Second, IPF also outperforms the ablation baselines: when $\Xagg$ is the month-aggregated network, IPF outperforms the ablation without $\Xagg$ by $214\%$ and the ablations without $\ptime$ or $\qtime$ by around $31\%$. 
Third, finer temporal granularity in the time-aggregated network improves IPF's performance, but the relative improvement is much larger from week to day than from month to week.
This suggests that bike trips vary more at a daily scale within the week (e.g., weekday vs. weekend) than a weekly scale within the month.
Overall, these results demonstrate the effectiveness of IPF on this difficult network inference problem as well as the benefits of more granular data.
\section{Conclusion}
In this work, we have established a statistical framework for using IPF to infer a dynamic network from its marginals. 
Our primary contribution is the biproportional Poisson model. We show that IPF uniquely recovers the MLEs of this model and derive statistical results on the MLEs.
Our model not only provides justification for using IPF to infer dynamic networks, but also enables new analyses of IPF, clarifies implicit network assumptions, and creates natural paths for testing alternate models and misspecification.
We also introduce \texttt{ConvIPF}, a principled algorithm that guarantees IPF convergence on sparse data, which can be broadly applied wherever IPF is used. 
Our empirics confirm our theoretical results, demonstrate the value of our modeling and algorithmic contributions, and reveal IPF's ability to infer networks in practice.
Given the long history of IPF and richness of the network inference problem, connecting them generates many avenues for future research.
Some future directions include characterizing how IPF's ability to infer networks varies based on the network's temporal dynamics (e.g., periodicity),
designing variants of IPF that outperform traditional IPF on this network inference problem, and providing a statistical interpretation of our convergence algorithm. 
Lastly, we note that the biproportional Poisson model is applicable more generally to other instances of the matrix balancing problem beyond network settings.

\section*{Acknowledgements}
S.C.\ was supported in part by an NSF Graduate Research Fellowship, the Meta PhD Fellowship, and NSF award CCF-1918940. F.K.\ was supported in part by NSF award CCF-1704417, NSF award IIS1908774, and N.~Anari’s Sloan Research Fellowship. Z.Q.\ was supported in part by ONR Grant N00014-19-1-2468 and NSF CAREER Award IIS-2143176. J.U.\ was supported in part by NSF CAREER Award IIS-2143176. J.L.\ was supported in part by NSF awards OAC-1835598, CCF-1918940, DMS-2327709, and Stanford Data Applications Initiative.
The authors thank Alfred Galichon, Emma Pierson, and Arjun Seshadri for helpful discussions and comments as well as Devin Caughey, Guido Imbens, Pang Wei Koh, and members of Jure Leskovec's lab for helpful feedback on early versions of this work.

\section*{Impact Statement}
The primary goal of our work is to advance understanding of IPF and dynamic network inference, with contributions to machine learning and applied statistics. 
However, our work is deeply motivated by problems of societal importance, such as epidemic response and transportation planning.
Prior work has required inferring dynamic networks from their marginals in order to model COVID-19 spread over detailed mobility networks, inform pandemic interventions, and analyze socioeconomic disparities in infection rates \citep{chang2021nature,chang2021kdd,chaudhuri2022overdispersion,li2023sdm,alimohammadi2023working}.
This dynamic network inference problem also appears in transportation \citep{carey1981method}, communication \citep{kruithof1937telefoonverkeersrekening}, and human migration \citep{plane1982information,pham2022migrant}.
Our work provides broadly applicable insights for practitioners in all of these domains, allowing them to fully leverage the empirical advantages of IPF with a deeper understanding of its assumptions and behavior (e.g., when it is justified, how to interpret its estimates, how to ensure convergence).
Another advantage of IPF is \textit{privacy}: it enables practitioners to estimate detailed networks while data providers only need to release aggregated data, i.e., only the time-varying marginals and the time-aggregated network.
Due to such limited information, IPF will only reconstruct the true network under very restrictive assumptions (Theorem \ref{thm:special}), but otherwise, it will infer a \emph{realistic} network that is correlated with the real network, enabling analyses that require time-varying networks without revealing private information.

\bibliography{references}
\bibliographystyle{icml2024}

\newpage
\appendix
\onecolumn
\section*{Appendix}
\renewcommand\thefigure{\thesection.\arabic{figure}}   
\setcounter{figure}{0}
\renewcommand\thetable{\thesection.\arabic{table}}
\setcounter{table}{0}

\section{Details on IPF Implementation}
\label{sec:app-ipf}
In this section, we provide details on our implementation of IPF, along with pseudocode in Algorithm \ref{alg:ipf}.
As discussed in the main text, IPF aims to solve the \textit{matrix balancing problem} \citep{deming1940ipf,schneider1990comparative}:
\begin{quote}
    Given positive vectors $\ipfrow \in\mathbb{R}_{++}^{m}, \ipfcol \in\mathbb{R}_{++}^{n}$ with
    $\sum p_{i}=\sum q_{j}$ and non-negative matrix $\ipfmat \in\mathbb{R}_{+}^{m\times n}$,
    find positive diagonal matrices $D^{0}$, $D^{1}$ satisfying the
    conditions $D^{0} \ipfmat D^{1}\cdot\revised{\mathbf{1}_n}=\ipfrow$ and $D^{1} \ipfmat^{T} D^{0}\cdot\revised{\mathbf{1}_m}=\ipfcol$. 
\end{quote}
IPF learns the scaling factors $\ipfparamrow$ and $\ipfparamcol$, which are diagonals of $D^0,D^1$, by alternating between scaling the rows to match $\ipfrow$, then scaling the columns to match $\ipfcol$:
\begin{align}
    \ipfparamrow_i(k+1) &= \frac{\ipfrow_i}{\sum_j \ipfmat_{ij} \ipfparamcol_j(k)}, \quad 
    \ipfparamcol_j(k+2) = \frac{\ipfcol_j}{\sum_i X_{ij} \ipfparamrow_i(k+1)}. \nonumber 
\end{align}
We denote by $\ipfestmat(k):= D^0(k)XD^1(k)$ the scaled matrix after the $k$-th iteration. 
The convergence behavior depends on the problem structure: $(D^0(k),D^1(k))$ can converge to a solution of the matrix balancing problem; $(D^0(k),D^1(k))$ can diverge but $\ipfestmat(k)$ converges; or $\ipfestmat(k)$ oscillates between accumulation points \citep{pukelsheim2009ipf}.
Furthermore, it is known that there are at most \textit{two} accumulation points, so if IPF does not converge, it oscillates between two solutions, one that matches the target row marginals $\ipfrow$ and one that matches the target column marginals $\ipfcol$ \citep{csiszar1984alternating,aas2014limits}.
Thus, in our implementation of IPF (Algorithm \ref{alg:ipf}), we check for two stopping conditions, one of which must be met eventually: either IPF converges, such that $\ipfestmat(k) \approx \ipfestmat(k+1)$, or IPF exhibits period-2 oscillation, such that $\ipfestmat(k) \approx \ipfestmat(k-2)$ and $\ipfestmat(k-1) \approx \ipfestmat(k-3)$. 

\paragraph{IPF with zeros in marginals.}
IPF returns a matrix of the form $D^0 X D^1$, where $D^0$ and $D^1$ are positive diagonal matrices.
So, unless the entire row or column of $X$ is 0, IPF solutions cannot naturally match zeros in the target row marginals $p$ or target column marginals $q$.
However, zero marginals are common in real-world data: for example, in mobility data (Section \ref{sec:app-safegraph}), many points-of-interest have zero visits at night, and in bikeshare data (Section \ref{sec:app-bikeshare}), some bike stations have zero trips at night.
So, our implementation of IPF modifies it slightly to allow for non-negative, instead of strictly positive, marginals.
Our version sets $d^0_i = 0$, for all $p_i = 0$, and $d^1_j = 0$, for all $q_j = 0$, then updates all other entries in $d^0$ and $d^1$ as usual, as described in \eqref{eqn:ipf-update}.
We show below that this is still a valid IPF procedure and all guarantees of IPF hold, because this procedure is \textit{equivalent} to running the original IPF procedure on $\tilde{X}$, $\tilde{p}$, and $\tilde{q}$, where $\tilde{X}$ is a submatrix of $X$ that leaves out the rows and columns with zero marginals, $\tilde{p}$ contains the nonzero entries in $p$, and $\tilde{q}$ contains the nonzero entries in $q$.

For some row $i$ where $p_i > 0$, let $d^0_i$ represent IPF's inferred parameter under our modified IPF procedure on $X$, $p$, $q$, and let $\tilde{d}^0_i$ represent IPF's inferred parameter under the original IPF procedure on $\tilde{X}$, $\tilde{p}$, and $\tilde{q}$.
Let $d^1_j$ and $\tilde{d}^1_j$ be defined analogously.
We will prove by induction that, for all iterations $k$, $d^0_i(k) = \tilde{d}^0_i(k)$, $\forall i$ s.t. $p_i > 0$, and  $d^1_j(k) = \tilde{d}^1_j(k)$, $\forall j$ s.t. $q_j > 0$.
First, in the base case, $d^0_i(0)$, $\tilde{d}^0_i(0)$, $d^1_j(0)$, and $\tilde{d}^1_j(0)$ are all initialized to 1.
Now, assuming the statement holds up to iteration $k$, the next IPF update is
\begin{align}
    d^0_i(k+1) &= \frac{p_i}{\sum_j X_{ij} d^1_j(k)} = \frac{p_i}{\sum_{j; q_j > 0} X_{ij} \tilde{d}^1_j(k)} = \tilde{d}^0_i(k+1). \label{eqn:ipf-nonnegative} 
\end{align}
In the denominator, we can drop all the terms where $q_j = 0$, since in our modified algorithm, we set $d^1_j = 0$ if $q_j = 0$.
Furthermore, since we are only considering $j$ where $q_j > 0$, then we can replace $d^1_j(k)$ with $\tilde{d}^1_j(k)$, based on the inductive hypothesis.
A similar proof follows to show the inductive step for $d^1_j$ and $\tilde{d}^1_j$.

Recall that in the connection between IPF and our network model \eqref{eqn:model}, the number of nonzero entries in $\bar{X}$ is our number of Poisson observations.
One implication of this equivalence between modified IPF for non-negative marginals and original IPF on the submatrix is that zero marginals substantially reduce our number of observations, since the submatrix drops entire rows and columns.
Thus, for a given hour $t$, our set of observations consists of $\{(i,j) | i \in [m], j \in [n], \Xagg_{ij} > 0, \ptime_i > 0, \qtime_j > 0\}$, which can be much fewer than $mn$ observations, given high levels of sparsity in real-world $\Xagg$, $\ptime$, and $\qtime$.
Fewer observations result in less accurate and more uncertain parameter estimates, making explicit the connection between data sparsity and quality of IPF estimates, which we also demonstrate empirically in Section \ref{sec:app-empirics}, especially Figure \ref{fig:ipf-sparsity}.

\begin{algorithm}[tb]
   \caption{Our implementation of the iterative proportional fitting procedure.}
   \label{alg:ipf}
\begin{algorithmic}
   \STATE {\bfseries Input:} matrix $X$, row marginals $p$, column marginals $q$, tolerance $\epsilon$
   \STATE Initialize $converged = false$, $oscillate = false$, $\tau = 1$, $d^0 = \mathbf{1}_m$, $d^1 = \mathbf{1}_n$
   \REPEAT
   \IF{$\tau$ is $odd$}
       \FOR{$i=1$ {\bfseries to} $m$}
           \IF{$p_i = 0$}
           \STATE $d^0_i \leftarrow 0$
           \ELSE 
           \STATE $d^0_i \leftarrow \frac{p_i}{\sum_j X_{ij} d^1_j}$
           \ENDIF
       \ENDFOR
   \ELSE 
       \FOR{$j=1$ {\bfseries to} $n$}
           \IF{$q_j = 0$}
           \STATE $d^1_j \leftarrow 0$
           \ELSE 
           \STATE $d^1_j \leftarrow \frac{q_j}{\sum_i X_{ij} d^0_i}$
           \ENDIF
       \ENDFOR
   \ENDIF
   \STATE $\ipfestmat(\tau) = \textbf{diag}(d^0) X \textbf{diag}(d^1)$
   \IF{$||\ipfestmat(\tau)-\ipfestmat(\tau-1)||_1 < \epsilon$}
   \STATE $converged \leftarrow True$
   \ELSIF{$||\ipfestmat(\tau)-\ipfestmat(\tau-2)||_1 < \epsilon$ \textbf{and} $||\ipfestmat(\tau-1)-\ipfestmat(\tau-3)||_1 < \epsilon$}
   \STATE $oscillate \leftarrow True$
   \ENDIF
   \STATE $\tau \leftarrow \tau+1$
   \UNTIL{$converged$ is $true$ or $oscillate$ is $true$}
\end{algorithmic}
\end{algorithm}
\section{Deriving Our Generative Network Model}
In this section, we provide details on our generative network model. 
In Appendix \ref{sec:dual}, we provide details on the known KL divergence duality result that is key to our following proof.
In Appendix \ref{thm:Poisson-mle}, we derive the log-likelihood of our model, prove that IPF recovers the MLEs of our model (Theorem \ref{thm:model}), and define the equivalent Poisson regression problem.
In Appendix \ref{sec:joint}, we formalize the ``joint'' network inference problem, where $\Xagg = \sum_t \Xtime$, and characterize when the joint problem reduces to the decoupled problem solved by our model.
In Appendix \ref{sec:thm2-proof}, we provide necessary and sufficient conditions for IPF to recover the true network exactly (Theorem \ref{thm:special}).
Finally, in Appendix \ref{sec:app-uniqueness}, we prove that among a larger class of generalized linear models, the Poisson model is the unique one where the MLE is the IPF solution --- so from the perspective of IPF, our generative model is canonical (Theorem \ref{thm:uniqueness}). 
In Figure \ref{fig:summary}, we also summarize many of the conditions we discuss in this work (e.g., when IPF converges, when MLEs are finite) and visualize how they fit together.


\subsection{Duality result for KL divergence minimization}
\label{sec:dual}
For completeness, we provide the details for deriving the dual problem \eqref{eqn:dual} of the KL minimization problem \eqref{eq:KL-minimization}.  Let
$u$ and $v$ be the multipliers of the constraints $\hat{\klmat}\mathbf{1}_{n}=p,\hat{\klmat}^{T}\mathbf{1}_{m}=q$,
respectively. Applying Sion's minimax theorem \citep{sion1958general}, the problem is equivalent to 
\begin{align*}
\min_{\hat{\klmat}}\max_{u,v}\sum_{ij}\hat{\klmat}_{ij}\log\frac{\hat{\klmat}_{ij}}{\ipfmat_{ij}}-\sum_{i}u_{i}(\hat{\klmat}\mathbf{1}_{n}-p)_{i}+\sum_{j}v_{j}(\hat{\klmat}^{T}\mathbf{1}_{m}-q)_{j} & =\\
\max_{u,v}\min_{\hat{\klmat}}\sum_{ij}\hat{\klmat}_{ij}\log\frac{\hat{\klmat}_{ij}}{\ipfmat_{ij}}-\sum_{i}u_{i}(\hat{\klmat}\mathbf{1}_{n}-p)_{i}+\sum_{j}v_{j}(\hat{\klmat}^{T}\mathbf{1}_{m}-q)_{j}
\end{align*}
where strong duality holds because both problems are feasible and bounded. 
Taking the first order condition with respect to $\hat{\klmat}_{ij}$, we obtain 
\begin{align*}
\log\hat{\klmat}_{ij} & =\log \ipfmat_{ij}-1+u_{i}-v_{j},
\end{align*}
 and substituting this back into the objective, we obtain 
\begin{align*}
\max_{u,v}\sum_{ij}\ipfmat_{ij}e^{-1+u_{i}-v_{j}}(-1+u_{i}-v_{j})&-\sum_{i}u_{i}(\sum_{j}\ipfmat_{ij}e^{-1+u_{i}-v_{j}}-p_{i})+\sum_{j}v_{j}(\sum_{i}\ipfmat_{ij}e^{-1+u_{i}-v_{j}}-q_{j})  \\
=\max_{u,v}&-\sum_{ij}\ipfmat_{ij}e^{-1+u_{i}-v_{j}}+\sum_{i}u_{i}p_{i}-\sum_{j}v_{j}q_{j}.
\end{align*}
Finally, using the 
change of variable  $u_{i}=u_i-\frac{1}{2}$ and $v_{j}=v_{j}+\frac{1}{2}$,
we obtain 
\begin{align*}
\max_{u,v}-\sum_{ij}\ipfmat_{ij}e^{u_i-v_j}+\sum_{i}p_{i}u_i-\sum_{j}q_{j}v_{j}-\frac{\sum_{i}p_{i}+\sum_{j}q_{j}}{2} & \Leftrightarrow\\
\min_{u,v}\sum_{ij}\ipfmat_{ij}e^{u_i-v_{j}}-\sum_{i}p_{i}u_{i}+\sum_{j}q_{j}v_{j},
\end{align*}
which we recognize as $g(u,v)$. 

\subsection{Connections to Entropy Regularized Optimal Transport}
\label{subsec:entropy-ot}
In many applications of IPF, the initial matrix $\ipfmat$ arises from some transportation cost function $C$ of the entropic-regularized optimal transport problem associated with \eqref{eq:KL-minimization}. More precisely, with the transformation $\ipfmat =\exp(-C/\varepsilon)$, the KL minimization problem 
\begin{align*}
  \min_{\hat Y\in \mathbb{R}^{n\times m}_+} D_{\text{KL}}(\hat{Y}\| \ipfmat)\\
\hat{Y}\mathbf{1}_m =p, \quad
\hat{Y}^{T}\mathbf{1}_n  =q
\end{align*}
is equivalent to the following regularized optimal transport problem
\begin{align*}
  \min_{\hat Y\in \mathbb{R}^{n\times m}_+} \langle \hat{Y},C\rangle+&\varepsilon D_{\text{KL}}(\hat{Y}\| p\otimes q)\\
\hat{Y}\mathbf{1}_m =p, &\quad
\hat{Y}^{T}\mathbf{1}_n  =q.
\end{align*}
For example, $C$ could be the distance between the origin and the destination or total travel time. In our setting, $\ipfmat$ is given as $\bar{X}$, the time-aggregated network (e.g., number of bike trips between stations over time or number of visits from neighborhoods to places), but we could still interpret it as a surrogate measure for the above mentioned physical quantities. For example, we can expect that the longer it takes to travel between a particular pair of locations, the fewer people on average will travel between those nodes. Therefore, when we do not have access to information about distance or travel cost, $\bar{X}$ can be used as a reasonable surrogate.

We also prefer to directly use $\bar{X}$ instead of writing it as $\exp(-C/\epsilon)$ since in real-world data, $\bar{X}$ contains many zero elements. If we want to interpret it in terms of $C$, we need to allow ``infinite transportation cost'' for some origin-destination pairs. In \cref{sec:synthetic-misspecification}, we also briefly discuss a possible extension of our biproportional Poisson model based on the gravity model of \citet{navick1994distance}, which uses the distance between origin-destination pairs as the cost function $C$, and with $\bar{X}$ the aggregate number of trips as a constant multiplier instead. When we have access to additional information (besides aggregate trips) that can serve as the cost $C$, such a model could be preferable.

\subsection{Proof of Theorem \ref{thm:model}}
\label{thm:Poisson-mle}
\begin{figure}[t]
    \centering
    \includegraphics[width=\linewidth]{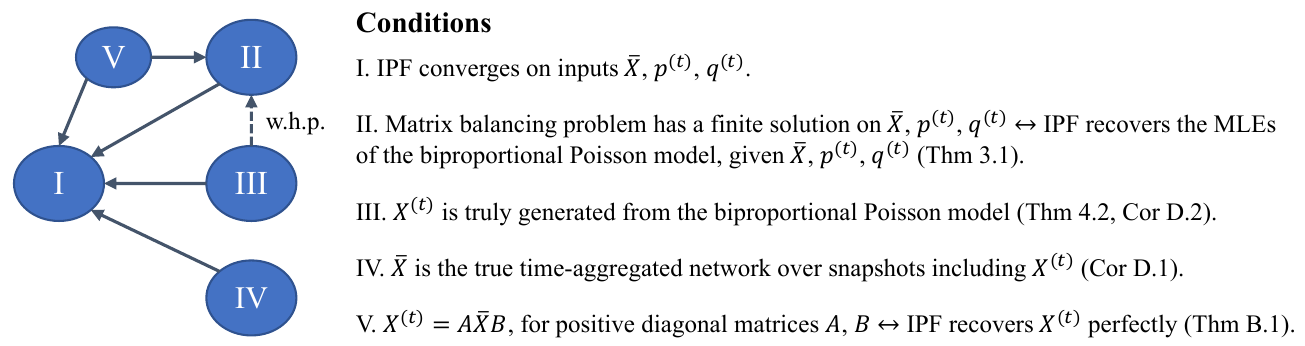}
    \caption{Summary of how the different conditions we discuss in this work fit together. Arrows indicate that one condition implies (i.e., is a sufficient condition for) another.
    Conditions II-V are all sufficient, but not necessary, conditions for IPF to converge. Prior work has also defined several necessary \textit{and} sufficient conditions for IPF to converge \citep{pukelsheim2014}; see Section \ref{sec:convergence-algo} for details. 
    }
    \label{fig:summary}
\end{figure}

For convenience, we restate our biproportional Poisson model \eqref{eqn:model} and Theorem \ref{thm:model} below.
Our model is
\begin{align*}
    \Xtime_{ij} &\sim \begin{cases}
        \mathrm{Poisson}(e^{\paramrow_i} \Xagg_{ij} e^{-\paramcol_j})\textrm{, if }\Xagg_{ij}>0,\\
        0\textrm{, otherwise.}
    \end{cases} \\
    \ptime_i &= \sum_j \Xtime_{ij}, \quad
    \qtime_j = \sum_i \Xtime_{ij}.
\end{align*}
where $\Xagg$ is the time-aggregated network, $\Xtime$ is the time-varying network, $u$ and $v$ are scaling factors, and $\ptime$ and $\qtime$ are the row and column sums of $\Xtime$, respectively.
Our theorem states:
\begin{quote}
Assume that the matrix balancing problem on $\Xagg$, $\ptime$, and $\qtime$ has a finite solution $(D^0,D^1)$. Then $\ipfparamrow$ and $\ipfparamcol$ are limits of the IPF iterations if and only if $\hat{\paramrow} = \log \ipfparamrow$ and $\hat{\paramcol} = -\log \ipfparamcol$ are solutions to the above maximum likelihood problem given $\Xagg$, $\ptime$, and $\qtime$, with log-likelihood $\ell=-g(u,v)$ in \eqref{eqn:dual}, i.e.,
\begin{align*}
\ell(u,v)=\sum_{i}\ptime_{i}u_i-\sum_{j}\qtime_{j}v_{j}-\sum_{ij}\Xagg_{ij}e^{u_{i}-v_{j}}.
\end{align*}
Moreover, this problem is equivalent to the maximum likelihood estimation of a Poisson regression model, and $\ptime,\qtime$ are the sufficient statistics.
\end{quote}

\begin{proof}
Since we know that IPF minimizes $g(u,v)$ in \eqref{eqn:dual}, our goal is to identify a network model whose likelihood is equivalent to $-g(u,v)$.
Several intuitions guide our construction of the model. 
First, \citet{qu2023sinkhorn} recently established connections between IPF and choice modeling, and observed that \eqref{eqn:dual} reduces to the maximum likelihood objective of a general class of choice models based on the Luce choice axiom, which suggests the same may be true for a network model.
Second, it is well-known that finding maximum likelihood parameters $\hat{\theta}$ is asymptotically equivalent to minimizing the KL divergence to the data-generating distribution, i.e., $\min_{\theta} D_{\textrm{KL}}(p_{\mathrm{data}} || p_{\theta})$ \citep{huber1967mle}.
Since the parameters appear on the right side of this KL problem and $X$ appears in the right side of the KL problem implied by IPF \eqref{eq:KL-minimization}, and KL divergence is not symmetric,  then our intuition is that the time-aggregated network $\Xagg$ (which takes the place of $\ipfmat$) should appear in the \textit{parameters} of our network model.
Lastly, the choice of Poisson distribution is natural given its close connections with KL divergence \citep{harremoes2001binomial,renner2013equivalence}, similar to the associations between $\ell_1$ and Laplace, as well as $\ell_2$ and Gaussian distributions.

With these intuitions, we arrived at our biproportional Poisson model. 
Now, we derive the log-likelihood of our model to verify that it matches $-g(u,v)$.
To simplify notation, we will employ the notation $(Y, X, p, q)$ as shorthand for $(\Xtime, \Xagg, \ptime, \qtime)$ throughout. The model's likelihood is given by
\begin{align}
    \mathcal{L}(Y | X, u, v) = \Pi_{i,j; X_{ij} > 0} \frac{(e^{\paramrow_i} \ipfmat_{ij} e^{-\paramcol_j})^{\klmat_{ij}} \exp(-(e^{\paramrow_i} \ipfmat_{ij} e^{-\paramcol_j}))}{\klmat_{ij}!}.
\end{align}
When maximizing the likelihood with respect to $\paramrow$ and $\paramcol$, we can drop the denominator, which is constant in the parameters.
Maximizing the log-likelihood yields the following problem:
\begin{align}
    \label{eqn:model-ll}
    \max_{\paramrow,\paramcol} \sum_{i,j; X_{ij} > 0} \klmat_{ij} \cdot ( \paramrow_i + \log \ipfmat_{ij} - \paramcol_j) - e^{\paramrow_i} \ipfmat_{ij} e^{-\paramcol_j}.
\end{align}
We can also drop $\klmat_{ij} \log(\ipfmat_{ij})$ since it does not depend on $u,v$.
The resulting problem is
\begin{align}
    \max_{\paramrow,\paramcol} &\sum_{i,j; X_{ij} > 0} \klmat_{ij}  \paramrow_i - \klmat_{ij} \paramcol_j - e^{\paramrow_i} \ipfmat_{ij} e^{-\paramcol_j} \\
    &= \max_{\paramrow,\paramcol} \sum_i \paramrow_i (\sum_{j; X_{ij} > 0} \klmat_{ij}) - \sum_j \paramcol_j (\sum_{i; X_{ij} > 0} \klmat_{ij}) -\sum_{i,j; X_{ij} > 0} e^{\paramrow_i} \ipfmat_{i,j} e^{-\paramcol_j} \\
    &= \max_{\paramrow,\paramcol} \sum_i \paramrow_i \ipfrow_i - \sum_j \paramcol_j \ipfcol_j - \sum_{ij} e^{\paramrow_i} \ipfmat_{ij} e^{-\paramcol_j}, \label{eqn:poisson-marginals}
\end{align}
which is equivalent to maximizing $-g(u,v)$.
Since IPF minimizes $g(u,v)$ (when the matrix balancing problem has a finite solution), it must also maximize the likelihood of our model, so we have proven the first part of our statement.
Furthermore, observe that \eqref{eqn:poisson-marginals} implies that the marginals of $\klmat$ are sufficient statistics for our network model, conveniently aligning with the problem constraints from IPF.

Next, we show that the network inference problem is equivalent to a Poisson \emph{regression} problem. Although our Poisson network model and Poisson regression are, not surprisingly, close relatives, it is instructive to precisely illustrate their connections and distinctions. Poisson regression is a generalized linear model which defines the \textit{logarithm} of a Poisson variable's expected value as a linear model of input features.
Generically, let $\theta \in \mathbb{R}^d$ represent the parameters of a Poisson regression model, $\mathbf{x} \in \mathbb{R}^d$ represent input features, and $y$ represent the observed non-negative count data.
Ignoring constants, the log-likelihood of observing $y$ under the Poisson regression model is
\begin{align}
    y \cdot \theta^T\mathbf{x} - e^{\theta^T\mathbf{x}}.
\end{align}
To match this to the log-likelihood of our network model in \eqref{eqn:model-ll}, for each sample indexed by $i,j$ with $\ipfmat_{ij}>0$, we set the input features $\mathbf{x}$ to be $[\mathbf{e}_i, \mathbf{e}_j, \log \ipfmat_{ij}]$, where $\mathbf{e}_i \in \{0, 1\}^m$ is a vector of all zeros aside from a 1 in the $i$-th position and $\mathbf{e}_j \in \{0, 1\}^n$ is a vector of all zeros aside from a 1 in the $j$-th position. 
Observe again that since this construction relies on $\log \ipfmat_{ij}$, the Poisson regression model only applies to $\klmat_{ij}$ where $\ipfmat_{ij} > 0$. We then set the dependent variable to be $y=\klmat_{ij}$. Lastly, it is obvious that we should set the parameter $\theta \in \mathbb{R}^d$ with $d = m + n + 1$ as $\theta=[ \paramrow, -\paramcol,1]$. We can verify that the log-likelihood of this Poisson regression model is equal to the objective in \eqref{eqn:model-ll}, and that $p,q$ are the sufficient statistics of the model. To perform Poisson regression, we may need a set of values $\klmat_{ij}$ for $\ipfmat_{ij}>0$ that are consistent with the marginals. This can be achieved exactly by running the max flow algorithm on the bipartite graph defined by $X$ \citep{idel2016review}. For details on the maximum flow algorithm, see Section \ref{sec:convergence-test}.
\end{proof}

\paragraph{Other settings where IPF recovers MLEs.}
A popular application of IPF is to estimate contingency tables, where it is known that IPF can be used to derive MLEs for certain log-linear models of contingency tables \citep{bishop1969contingency,bishop1974discrete,little1991models,little1993poststrat}. 
For example, a related formulation using $\lambda_{ij} = e^{\paramrow_i} \Xagg_{ij} e^{-\paramcol_j}$ for contingency tables is studied by \citet{little1991models}. However, instead of Poisson random variables, $\Xtime_{ij}$ are treated as probability parameters of a target population. Data from a two-step sampling procedure fitted using IPF then yields the MLE. 
Recently, \citet{qu2023sinkhorn} revealed interesting connections of IPF to algorithms in the choice modeling literature, where it is shown that IPF can be used to compute MLEs of a general class of choice models based on Luce's axiom of choice  \citep{luce1959individual}. 
Our network model is closely related to such choice models, and our statistical theory results also draw inspirations from corresponding works in the choice literature \citep{seshadri2020learning,bong2022generalized}. 

\paragraph{Comparison to collective graphical models.}
\citet{sheldon2011cgm} introduce the problem of collective graphical models (CGMs): how to fit a model of individual behavior from aggregate data (e.g., counts). 
The general problem that CGM addresses shares strong similarities with the dynamic network inference problem in our work, where we are also trying to estimate finer-grained information (time-varying networks) based on coarser, aggregate data (time-varying marginals and time-aggregated network) due to privacy or data collection costs. 
A common feature is that the available aggregate data provides the sufficient statistics of the model, which enables valid estimation and inference. An additional algorithmic connection between our work and CGM is that IPF has also been used in the inference of CGM \citep{singh2020opt}.

However, a notable difference is that CGM is typically applied in settings where only the time-varying marginals are known, without access to the time-aggregated network. 
For example, several papers use CGMs to estimate population flows given the number of people in each area over time \citep{iwata2017flow,akagi2018ijcai,iwata2019neural}. 
Furthermore, CGM seeks to relate population-level observations to a model of individual behavior (e.g., a single person’s mobility or a bird’s migration patterns), while we are not trying to estimate individual models. 
Instead of disaggregating over populations, we seek to disaggregate over time, by taking the time-aggregated network and distributing it over smaller time intervals (e.g., from monthly to hourly). 
Finally, the specific mathematical formulations of the problems in CGM and in the network inference problem seem to be distinct. 
CGM is built on probabilistic graphical models of independent and identically distributed individuals. It aims to estimate the parameters of random vectors on nodes of a graph given only aggregate counts of specific states using MAP inference. 
In contrast, our biproportional Poisson model is used to model the random population flows between nodes of a graph using maximum likelihood estimation, which is solved via IPF.

\subsection{The ``joint'' network inference problem}
\label{sec:joint}

In our biproportional Poisson model \eqref{eqn:model}, we used the time-aggregated network $\Xagg$ as the baseline intensity in the Poisson parameter $e^{u_i} \Xagg_{ij} e^{-v_j}$. We then solve many ``decoupled'' problems, where we fit the model separately for every time step $t$.
However, in doing so we are potentially leaving out additional information implied by the constraint that the Poisson variables aggregated over time should be \emph{equal} to the observed time aggregated traffic $\Xagg$.
This constraint is only applicable when we are inferring $\Xtime$ for the same set of time steps $t$ over which $\Xagg$ is aggregated, so it would not apply, for example, if $\Xagg$ is aggregated from historical data or we are only inferring $\Xtime$ for some subset of time steps.
In cases where it does apply, we consider instead the following ``joint'' model
\begin{align}
\label{eq:joint-Poisson}
Y_{ijt} & =\begin{cases}
\text{Poisson}(X_{ij}e^{-v_{it}+u_{jt}}) & X_{ij}>0\\
0 & X_{ij}=0
\end{cases},
\end{align}
where $Y_{ijt}$ describes the traffic between nodes $i,j$ at time $t$, $X_{ij}$ is an \emph{unknown} parameter that describes the time-invariant propensity of traffic between $i,j$, and $v_{it},u_{jt}$ are now time-dependent intensity parameters that describe the level of activity at different nodes. 

In the cross-sectional problem, we use the observed $\Xagg_{ij}$ which
aggregates over different times in place of $X_{ij}$. In \eqref{eq:joint-Poisson}, instead
of having $X_{ij}$ as an observable, we posit that it is
a model parameter that needs to be estimated, and we observe the
three marginals quantities $\Xagg,P,Q$:
\begin{align*}
    \Xagg_{ij} &=\sum_{t}Y_{ijt},\\ Q_{it}&=\sum_{j}Y_{ijt},\\
P_{jt}&=\sum_{i}Y_{ijt}.
\end{align*}
Importantly, we no longer require that $X_{ij}$
and $\Xagg_{ij}$ have the same zero patterns, as $\Xagg_{ij}=0$ could result from $\text{Poisson}(X_{ij}e^{-v_{it}+u_{jt}})=0$ for all $t$. The relevant part of the log-likelihood function of the joint model \eqref{eq:joint-Poisson} is given by
\begin{align*}
\sum_{ijt}Y_{ijt}\log(X_{ij}e^{-v_{it}+u_{jt}})-\sum_{ijt}X_{ij}e^{-v_{it}+u_{jt}} =&\\
\sum_{ij}\Xagg_{ij}\log X_{ij}-\sum_{it}Q_{it}v_{it}+\sum_{jt}P_{jt}u_{jt}-\sum_{ijt}X_{ij}e^{-v_{it}+u_{jt}},
\end{align*}
where again the marginal quantities $\Xagg,P,Q$ are the sufficient statistics of the model. 
The first order conditions are given by 
\begin{align*}
e^{-v_{it}}\cdot\sum_{j}X_{ij}e^{u_{jt}} & =Q_{it}\\
e^{u_{jt}}\cdot\sum_{i}X_{ij}e^{-v_{it}} & =P_{jt}\\
X_{ij}\cdot\sum_{t}e^{-v_{it}+u_{jt}} & =\Xagg_{ij}.
\end{align*}
 Compare this system with that of the cross-sectional problem, where for each fixed $t$, we solve 
\begin{align*}
e^{-v_{it}}\cdot\sum_{j}\Xagg_{ij}e^{u_{jt}} & =Q_{it}\\
e^{u_{jt}}\cdot\sum_{i}\Xagg_{ij}e^{-v_{it}} & =P_{jt}.
\end{align*}
 In contrast, for the joint problem we instead solve 
\begin{align*}
e^{-v_{it}}\cdot\sum_{j}X_{ij}e^{u_{jt}} & =Q_{it}\\
e^{u_{jt}}\cdot\sum_{i}X_{ij}e^{-v_{it}} & =P_{jt}
\end{align*}
where $X_{ij}$ satisfies the additional constraint 
\begin{align*}
X_{ij}\cdot\sum_{t}e^{-v_{it}+u_{jt}} & =\Xagg_{ij}.
\end{align*}
Therefore, the implicit assumption used to reduce the joint problem to the cross-sectional problems is that 
\begin{align}
\label{eq:stationarity}
    \sum_{t}e^{-v_{it}+u_{jt}} \approx c
\end{align}
for all $i,j$ where $\Xagg_{ij} > 0$; otherwise, $X_{ij} = 0$ or $\sum_{t}e^{-v_{it}+u_{jt}} = 0$. In other words, the aggregated time-dependent intensities for each pair $(i,j)$ does not depend on $i,j$. We may justify \eqref{eq:stationarity} in several ways. For example, if we assume $v_{it}$ and $u_{jt}$ are drawn i.i.d. from two distributions with the same mean, then $\frac{1}{T} \sum_{t}e^{-v_{it}+u_{jt}} \rightarrow 1$ by the continuous mapping theorem. Alternatively, if $X_{ij}$ is an accurate description of the long-run traffic propensity between $i,j$, then the intensity of the aggregate traffic Poisson variables $\sum_tY_{ijt}$ should largely be captured by $X_{ij}$, i.e., \emph{independent} of $\sum_{t}e^{-v_{it}+u_{jt}}$.

\subsection{When IPF recovers the true network exactly} \label{sec:thm2-proof}
If we swap out the Poisson distribution in our biproportional Poisson model \eqref{eqn:model} for an identity function, we find that this condition is necessary and sufficient for IPF to exactly recover the true network, $\Xtime$, given $\Xagg$, $\ptime$, $\qtime$. \vspace{0.2cm}
\begin{theorem}
    \label{thm:special}
    IPF returns $\Xtime$ if and only if $\Xtime = A \Xagg B$, for some positive diagonal matrices $A$ and $B$.
\end{theorem}
\begin{proof}
First, we prove that, if IPF returns $\Xtime$, then $\Xtime = A \Xagg B$.
All IPF solutions take the form $D^0 \Xagg D^1$, where $D^0$ and $D^1$ are positive diagonal matrices.
If IPF returns $\Xtime$, then $\Xtime$ can be written as $A \Xagg B$, with $A = D^0$ and $B = D^1$.

Second, we prove that, if $\Xtime = A \Xagg B$, then IPF will return $\Xtime$.
First, if there is a finite solution to the matrix balancing problem, IPF will converge to a solution \citep{pukelsheim2014}.
We know that $A \Xagg B$ is a solution, so IPF will return $D^0 \Xagg D^1$, where $(D^0 \Xagg D^1) \cdot \mathbf{1}_n= \ptime$ and $(D^1 \Xagg^T D^0) \cdot \mathbf{1}_m = \qtime$.
Furthermore, biproportional scalings are \textit{unique} with respect to marginals, meaning if two biproportional scalings $M^1$ and $M^2$ of $M^0$ have the same marginals, then $M^1 = M^2$ \citep{pukelsheim2014}.
Thus, IPF will return $\Xtime$, since IPF will converge to $D^0 \Xagg D^1$, with marginals $\ptime$ and $\qtime$, and $\Xtime$ is the unique biproportional scaling of $\Xagg$ that matches the marginals.
\end{proof}

While this statement about $\Xtime$, that $\Xtime = A \Xagg B$, is rarely true in practice, this result allows us to pin down when IPF works ``perfectly'' for dynamic network inference.
Notably, this result implies that the scaling matrix from $\Xagg$ to $\Xtime$, i.e., $\Xtime_{ij}/\Xagg_{ij}$ where $\Xagg_{ij} > 0$, must be a rank-1 matrix with entries $a_i b_j$ (the diagonals of $A$ and $B$).
If we interpret $\Xtime$ as a dynamic network and $\Xagg$ as its time-aggregated form, then we are essentially constraining the complexity of the network's temporal variation.

\subsection{Poisson links uniquely recover IPF within a class of generalized linear models}\label{sec:app-uniqueness}
In Theorem \ref{thm:model}, we showed that IPF recovers the MLEs of the biproportional Poisson model, and then we argued that spelling out this model clarifies implicit assumptions when using IPF to infer dynamic networks from their marginals. 
However, this argument requires some degree of \textit{uniqueness} in the model we identified, since if IPF recovers the MLEs of many different models, then using IPF to infer dynamic networks may be justified under the assumptions of many different models, not only the biproportional Poisson.
To analyze uniqueness, in this section we consider a very natural family of generalized linear models which generalizes the biproportional Poisson model. 
We prove a rigorous result which shows that if we want to recover IPF as the solution to maximum likelihood estimation, then we are \textit{forced} to choose scaled Poisson distributions for the link\footnote{Note that if the inputs to the IPF algorithm are uniformly scaled, the outputs of IPF are scaled in the same way, so the fact that scaling doesn't matter is expected. This is a useful property of IPF, because it means it works even when data is scaled/normalized.}. This shows that in some sense our model is canonical, and it is unique within this general family. 

Throughout this section, time superscripts are omitted. 

\paragraph{Class of models where row and column statistics are sufficient.} 
Suppose that $(p_{\alpha})_{\alpha \in A}$ is an exponential family with canonical parameter\footnote{See e.g. \cite{mccullagh2019generalized} for a reference on exponential families and generalized linear models. As an example, for a Poisson distribution with rate $\lambda$ the canonical parameter is $\log \lambda$.}  $\alpha$ and sufficient statistic $x$. Assuming without loss of generality that $0 \in A$, this means that
\begin{equation}\label{eqn:glm-link}
\frac{dp_{\alpha}}{dp_0}(x) = \exp(\alpha x - \Phi(\alpha)). 
\end{equation}
We suppose the observations are generated by the following generalized linear model: independently for each $i$ and $j$ we sample
\begin{equation}\label{eqn:glm-model} Y_{ij} \sim p_{\alpha_{ij}}, \qquad \alpha_{ij} = u_i + \log X_{ij} - v_j. 
\end{equation}
Note that the algebraic form of \eqref{eqn:glm-link} and the weights $\alpha_{ij}$ in \eqref{eqn:glm-model} is exactly what leads in the derivation of \eqref{eqn:poisson-marginals} to the row and column sums being sufficient statistics for the overall model. This is also illustrated below in the computation of the log-likelihood.
In other words, the key desiderata that the row and column statistics are sufficient does not \emph{by itself} force us to consider the biproportional Poisson model --- rather, it is the fundamental property of this larger class of GLMs. Many  important distribution families satisfy \eqref{eqn:glm-link}: for example, the class of unit-variance Gaussians $N(\mu,1)$ or the class of exponential distributions $\operatorname{Exp}(\lambda)$.

Thus, this is the natural class of models to consider if we want to see if IPF is also the MLE for a variant of our model, or in other words test how unique the connection between IPF and the biproportional Poisson model is. As we will show, this connection is unique --- the only measures $p_{\alpha}$ that will recover IPF are scalings of the Poisson family, and this is true even if we only require the measure to satisfy a significantly generalized version of the IPF equations. 

\paragraph{Log-likelihood under general model.} Recall that we use $(y, X, p, q)$ as shorthand for $(\Xtime, \Xagg, \ptime, \qtime)$. 
The log-likelihood of our model, assuming $X$ is known, is
\begin{align*} 
L(u,v;y) 
&= \sum_{ij} \log p_{\alpha_{ij}}(y_{ij}) = \sum_{ij} (u_i + \log X_{ij} - v_j) y_{ij} - \sum_{ij} \Phi(u_i + \log X_{ij} - v_j) \\
&= \sum_i u_i p_{i} - \sum_j v_j q_j - \sum_{ij} \Phi(u_i + \log X_{ij} - v_j) + C(y). 
\end{align*}
where $C(y) = \sum_{ij} y_{ij} \log X_{ij}$ does not depend on the model parameters $u$ or $v$. As promised, we see that the row sums $p_i = \sum_j y_{ij}$ and column sums $q_j = \sum_i y_{ij}$ are sufficient statistics for this model. 

Differentiating, we see that
\begin{equation}\label{eqn:generalized-critical} \partial_{u_i} L(u,v;y) = p_i - \sum_j \Phi'(u_i + \log X_{ij} - v_j) = p_i - \sum_j \Phi'(u_i - w_{ij}), 
\end{equation}
where for convenience we defined 
\[ w_{ij} = \log X_{ij} - v_j \]
to be the term in the linear model added to $u_i$. The partial derivative with respect to $v_j$ satisfies a symmetrical equation. Because of this symmetry, it will suffice for us to consider the equation for the row weights in what follows.
\paragraph{(Generalized) IPF is uniquely recovered by a scaled Poisson link.} We say that the MLE satisfies a \emph{generalized proportional fitting equation} if
\begin{equation}\label{eqn:gps-def}
\partial_{u_i} L(u,v;y) = 0 \iff h(u_i) = \frac{p_i}{\sum_j f(w_{ij})}
\end{equation}
for some smooth functions $f,h$. This definition captures and generalizes the fact that in IPF, the row weight, parameterized here by $u_i$, is determined by the ratio of the empirical count $p_i$ and the corresponding row sum. In the special case of the Poisson distribution, we have $f = h = \exp$.

Now we show that if the critical point equation has the generalized form \eqref{eqn:gps-def}, the distribution corresponding to the link must be a scaled Poisson family  --- in which case, we know that we get exactly the IPF equations discussed in the main body of the paper. So for this family of generalized linear models, we recover the IPF equations as a characterization of the MLE exactly when the link corresponds to a scaled Poisson distribution.
\begin{theorem} \label{thm:uniqueness}
In the above generalized linear model, the MLE satisfies a generalized proportional fitting equation for all parameters $u,v$ iff the exponential family $\{p_{\alpha}\}_{\alpha}$ is the set of Poisson distributions scaled by some factor $a \in \mathbb R$. In other words, $\{p_{\alpha} : \Phi(\alpha) < \infty\} = \{ a \operatorname{Poisson}(\lambda) : \lambda \ge 0 \}$ for some $a$.
\end{theorem}
\begin{proof}
Substituting \eqref{eqn:generalized-critical}, we see that \eqref{eqn:gps-def} is equivalent to asking that
\begin{equation}\label{eqn:gps} 
h(u_i) \sum_j f(w_{ij}) = \sum_j \Phi'(u_i + w_{ij}). 
\end{equation}
Taking the partial derivative with respect to $v_{j}$ (equivalently, $w_{ij}$), this implies that
\[ h(u_i) f'(w_{ij}) = \Phi''(u_i + w_j). \]
This holding for all possible $u,v$ is the same as requiring that
\[ \Phi''(u + t) = h(u) \cdot f'(t). \]
Letting $t = 0$ gives that $h(u) = \Phi''(u)/f'(0)$ and letting $u = 0$ gives that $f'(t) = \Phi''(t)/h(0) = f'(0) \Phi''(t)/\Phi''(0)$. So
\[ \Phi''(u + t) = \frac{\Phi''(u) \Phi''(t)}{\Phi''(0)}. \]
Taking the partial derivative with respect to $t$ on both sides and setting $t = 0$, we have
\[ \Phi'''(u) = \Phi''(u) \frac{\Phi'''(0)}{\Phi''(0)}. \]
Equivalently, letting $a = \frac{\Phi'''(0)}{\Phi''(0)}$, we have
\[ \frac{d}{du} \log \Phi''(u) = \frac{\Phi'''(u)}{\Phi''(u)} = a\]
Integrating yields
\[ \log \Phi''(u) = au + b\]
so
\[ \Phi''(u) = e^{au + b}. \]
Integrating twice more yields
\[ \Phi(u) = e^{au + b} + cu + d. \]
From the definition, we must have $\Phi(0) = 0$ so
\begin{equation}\label{eqn:phi-nearly}
\Phi(u) = e^{au + b} + cu - e^{b}. 
\end{equation}
Recalling that $f'(t) = f'(0) \Phi''(t)/\Phi''(0)$ and $h(u) = \Phi''(u)/f'(0)$,
returning to \eqref{eqn:gps} yields that
\[ \frac{\Phi''(u_i)}{\Phi''(0)} \sum_j (\Phi'(w_{ij}) + f(0)) = \sum_j \Phi'(u_i + w_{ij}),\]
and plugging in \eqref{eqn:phi-nearly} yields
\[ a e^{au_i} \sum_j (e^{a w_{ij} + b} + f(0)) = \sum_j (ae^{a(u_i + w_{ij}) + b} + c), \]
so
\[ c = ae^{au_i} f(0)\]
for all $u_i$. Either $a = 0$, in which case this implies $c = 0$, or if $a \ne 0$ this implies $f(0) = 0$ and so $c = 0$ regardless. Hence
\[ \Phi(u) = e^{au + b} - e^b. \]
Recall that the cumulant generating function of the Poisson distribution with canonical parameter $b$ (i.e. with rate $\lambda = e^b$) is $u \mapsto e^b(e^u - 1)$. So the above $\Phi$ is exactly the cumulant generating function of a scaled Poisson distribution $a \operatorname{Poisson}(e^b)$. We know that such distributions satisfy the IPF equations, so this is  indeed an if-and-only-if statement.
\end{proof}

\section{Statistical theory of the biproportional Poisson model}
\label{sec:app-theory}
In this section, we provide a comprehensive presentation including proofs of our statistical results on the biproportional Poisson model. In Appendix \ref{sec:app-mle-error}, we derive bounds on the estimation error of the model's MLEs (when bounded).
In Appendix \ref{sec:app-mle-finite}, we prove that, under the correct specification of the biproportional Poisson model, the maximum likelihood estimation problem has a unique bounded normalized solution with high probability.

We begin with a precise definition of the graph Laplacian used in this work. Given the time-aggregated matrix $\Xagg$, we define the bipartite graph $G_b$ with adjacency matrix $A({\Xagg})$:
\begin{align}
A({\Xagg}):=\begin{bmatrix}0 & {\Xagg}\\
{\Xagg}^{T} & 0
\end{bmatrix}, \mathcal{L}:=\begin{bmatrix}\mathcal{D}({\Xagg}\mathbf{1}_{m}) & -{\Xagg}\\
-{\Xagg}^{T} & \mathcal{D}({\Xagg}^{T}\mathbf{1}_{n})
\end{bmatrix}.
\label{eq:Laplacian}
\end{align}
We can verify that $\mathcal{L}$ is also the Hessian of the negative log-likelihood of the biproportional Poisson model evaluated at $(u,v)=(0,0)$.

\subsection{Structure-dependent Bounds on Estimation Error of MLE}
\label{sec:app-mle-error}

 We begin with a lemma from \citet{qu2023sinkhorn}, which bounds the Hessian of the negative log-likelihood \eqref{eqn:dual} in a bounded domain by the Fiedler eigenvalue. Recall that we use $X$ as shorthand for $\Xagg$.

\begin{lemma}[Appendix F.3 of \cite{qu2023sinkhorn}]\label{lem:laplacian}
For all $(u,v) \in 1_{m + n}^{\perp}$ with $\|(u,v)\|_{\infty} \le B$,
\[ \lambda_{-2}(\nabla^2 g(u,v)) \ge e^{-2B} \lambda_{-2}(\mathcal L) \]   
where 
\[ \mathcal L = \mathcal D\left(\begin{bmatrix} 0 & X \\ X^T & 0 \end{bmatrix} \begin{bmatrix} 1_n \\ 1_m \end{bmatrix} \right) - \begin{bmatrix} 0 & X \\ X^T & 0 \end{bmatrix} \]
is the graph Laplacian of the weighted bipartite graph given by $X$.
\end{lemma}
Using the lemma above, we can obtain bounds on the expected squared error for the normalized MLEs  quantified by the Fiedler eigenvalue, 
as long as the true parameters and MLEs are all bounded by some constant $B$. 
\begin{theorem}\label{thm:bd-apdx}
Suppose that the Poisson network model holds with ground truth parameters $u^*,v^*,X$.
Suppose $(u,v)$ is a maximizer of the likelihood (minimizer of $g$) and that we have the normalization condition $(u - u^*, v - v^*) \in 1_{m + n}^{\perp}$ and $\|(u,v,u^*,v^*)\|_{\infty} \le B$. Then over the randomness of $p^{(t)},q^{(t)}$ we have in expectation
\begin{equation}\label{eqn:expectation-bd}
\mathbb E \|(u - u^*, v - v^*)\|^2 1_{\mathcal B} \le \frac{8 e^{4B}}{\lambda_{-2}(\mathcal L)^2} \kappa 
\end{equation}
where  $\mathcal B$ is the event that the MLE exists and is bounded by $B$, and
\[ \kappa = \sum_{i,j} e^{u_i} X_{ij} e^{-v_j}. \]
Furthermore, with probability at least $1 - \delta$, we have that whenever the MLE exists and is bounded by $B$, it satisfies
\begin{align*} 
\|(u - u^*, v - v^*)\|^2 \le \frac{8e^{4B}}{\lambda_{-2}(\mathcal L)^2} \left(\lambda +  \sqrt{6\log(4/\delta) \lambda (1 + M)} + C\log^2(12/\delta) \log^2(1 + m + n) (1 + M)\right)
\end{align*} 
where $C > 0$ is an absolute constant, and
\[M:=\max\{\max_{i \in [m]} \sum_j  e^{u^\ast_i} X_{ij} e^{-v^\ast_j}, \max_{j \in [n]}\sum_i e^{u^\ast_i} X_{ij} e^{-v^\ast_j} \}. \]
is the maximum of the row and column sums of the matrix $(e^{u_i} X_{ij} e^{-v_j})_{ij}$.
\end{theorem}
\begin{remark}[Interpretation of Theorem~\ref{thm:bd-apdx}]
\label{rem:error-bound}
We illustrate the bound with some examples, discuss the meaning and tightness of this bound, and give a simple sufficient condition for the MLE to be bounded.
\begin{enumerate}
    \item \emph{Examples with strong recovery guarantees: complete network.} Suppose that $X$ is the $n \times m$ all-ones matrix and that $B$ is a constant (e.g. $B = 5$).  Then $\kappa = \Theta(nm)$ and $\lambda_{-2}(\mathcal L) = \min(n,m)$ so the right hand side of \eqref{eqn:expectation-bd} is of order $\Theta(\max(n/m,m/n))$. So if e.g. $n = m$, then the total error for recovering the entire vector is $O(1)$. Equivalently, the average error per coordinate of $u,v$ is $\frac{1}{2n} \|(u - u^*,v - v^*)\|^2 = O(1/n)$.  

    For a more sophisticated example on a random graph model, see Appendix~\ref{sec:synthetic-sparsity}.
    \item \emph{Meaning of normalization condition.} The normalization condition $\sum_i u_i + \sum_i v_i = \sum_i u^*_i + \sum_i v^*_i$ eliminates a scaling ambiguity in the model: if we add a constant $c$ to $u^*$ and to $v^*$, it does not change the distribution of the output of the model.  
    \item \emph{Dependence on $\kappa$ is optimal.} Consider the special case where $X : m \times 1$ with all-ones entries, and we fix the scaling convention that $v^*_1 = 0$ (because any value of $v^*_1$ is consistent with some shift of $u^*$). Then the generative model reduces to observing $m$ independent Poisson variables with rates $e^{u^*_i}$ for $i = 1$ to $m$, and our estimator $u$ reduces to solving the equation $e^{u_i} = p_i^{(t)}$. Then $\kappa = \sum_{ij} e^{u_i}$ is the total variance of estimating the quantities $e^{u^*_i}$. Since $\lambda_{-2} = 1$ in this example, we see that the bound is optimal up to constants and the dependence on $B$ in this example.
    \item \emph{Dependence on $\lambda_{-2}$ is required; non-identifiability when $\lambda_{-2} = 0$.} If $\lambda_{-2} = 0$ no guarantee of this form is possible for the estimator, because the parameters of the model are not identifiable. The reason is that this corresponds to the case where the graph is disconnected, in which case the true parameters are not identifiable. To illustrate the reason why, consider the case where $X = I$. Then for every $i \in [n]$, we can replace $u^*_i,v^*_i$ with $u^*_i + c, v^* + c$ for any $c \in \mathbb R$ without changing the distribution of observations from the model. Unless $n = 1$, this results in nonidentifiability in the model which is not fixed by the normalization condition. 
    More generally, whenever the graph is disconnected, we can similarly shift the parameters corresponding to one of the components without changing the distribution of observations.

    We briefly remark that some quantitative dependence on $\lambda_{-2}$ is required even if we restrict to the case $\lambda_{-2} > 0$. This is because the above argument can be modified to show that the parameters are statistically indistinguishable if the graph is not strictly disconnected but effectively so (e.g. if we add $\epsilon > 0$ to all of the entries of $X$ for a very tiny $\epsilon$). This can be formally proven by a similar KL calculation to the one in the example below. 
    \item \emph{Example where $e^{\Theta(B)}$ term is required.} We show that the term with exponential dependence on $B$ in \eqref{eqn:expectation-bd} cannot be removed --- this is by establishing a lower bound which holds not just for our estimator, but for any possible estimator. 
    Consider a very special case of the model, where we have a single observation from a Poisson distribution with parameter $e^{u^*_1}$. Consider two such models: in the first one $u^*_1 = u'_1 = -B/2$ and in the second one $u^*_1 = u''_1 = -B$. The KL divergence between the two Poisson distributions $Poi(e^{u'_1})$ and $Poi(e^{u''_1})$ is $e^{u'_1}(u'_1 - u''_1) + e^{u'_1} - e^{u'_2} = O(Be^{-B/2}) = O(e^{-B/4})$. By Pinsker's inequality and the Neyman-Pearson Lemma (see e.g. \cite{rigollet2023high}), this means that it is impossible to distinguish between these two distributions from a single sample with probability of success better than $1/2 + e^{-\Omega(B)}$. This means any estimator for the true parameter will have to make an error of size $\Omega(B)$ with probability at least $1/2 - e^{-\Omega(B)}$ in one of these models. Because $\kappa = e^{-\Omega(B)}$ and $\lambda_{-2}(\mathcal L) = 1$, the right hand side of \eqref{eqn:expectation-bd} would be exponentially small in $B$ without the presence of the term $e^{4B}$, which would contradict our lower bound. 

    \item \emph{Generalization of conclusion.} From the proof of the theorem, we can see that the final high probability guarantee doesn't hold specifically for $u,v$ the MLE, but for any parameters $u,v$ which achieve at least as high log-likelihood than the ground truth $u^*,v^*$ on the training data. (So in this form, the theorem could be applied even if the MLE doesn't exist.)

    \item \emph{MLE is bounded with high probability if entries of $X$ large enough.}
    Consider scaling $X$ by $\beta>1$. When $\beta$ is large enough, one can actually show that with high probability the MLE must be close to $(u^\ast,v^\ast)$ --- in particular it exists and it is bounded. From the proof of \cref{thm:mse} and Markov, with 99\% probability we have
\[ g(u,v) \ge g(u^*,v^*) - \sqrt{200 \sum_{ij} e^{u^*_i} X_{ij} e^{-v^*_j}} + e^{-2b} \lambda_{-2}(\mathcal L)(\|u - u^*\|^2 + \|v - v^*\|^2)/2. \]
Suppose that $X = \beta Z$ and let $\mathcal L_Z$ be the Laplacian for $Z$. Then this is
\[ g(u,v) \ge g(u^*,v^*) - \sqrt{\beta} \sqrt{200 \sum_{ij} e^{u^*_i} X_{ij} e^{-v^*_j}} \|(u,v) - (u^*,v^*)\| + e^{-2b} \beta \lambda_{-2}(\mathcal L_Z)(\|(u,v) - (u^*,v^*)\|^2)/2. \]
Define $B = \|(u^*,v^*)\|_{\infty}$ and observe $b \le B + \|(u - u^*, v - v^*)\|$ so letting $r = \|(u - u^*, v - v^*)\|$ we have the inequality
\[ g(u,v) \ge g(u^*,v^*) - \sqrt{\beta} \sqrt{200 \sum_{ij} e^{u^*_i} Z_{ij} e^{-v^*_j}}r + e^{-2B - 2r} \beta \lambda_{-2}(\mathcal L_Z)r^2/2. \]
This implies $g(u,v) > g(u^*,v^*)$ if
\[ \sqrt{\beta} e^{-2B - 2r} r \lambda_{-2}(\mathcal L_Z) > 40 \sqrt{\kappa} \]
where $\kappa = \sum_{ij} e^{u^*_i} X_{ij} e^{-v^*_j}$. In particular, if $r = 1$ this reduces to the condition
\[ \sqrt{\beta} e^{-2B} \lambda_{-2}(\mathcal L_Z) > 40 e^2 \sqrt{\kappa} \]
so taking
\[ \beta > 400^2 \frac{e^{4 B}}{\lambda_{-2}(\mathcal L_Z)^2} \kappa  \]
suffices. Finally, observe by convexity that this implies the minimum of $g$ is attained within a ball of radius $1$ around $(u^*,v^*)$, since $g(u,v) > g(u^*,v^*)$ for all $u,v$ on the sphere of radius $1$ centered at $(u^*,v^*)$.
\end{enumerate}  
\end{remark}
\begin{proof}[Proof of Theorem~\ref{thm:bd-apdx}]
In this proof, we omit the time superscripts so $p = p^{(t)}$ and $q = q^{(t)}$.

We have by \cref{lem:laplacian}, Taylor's theorem, and the Cauchy-Schwarz inequality that 
\begin{align*} 
g(u,v) 
&\ge g(u^*,v^*) + \langle \nabla g(u^*,v^*), (u - u^*, v - v^*) \rangle 
+ e^{-2B} \lambda_{-2}(\mathcal L) (\|u - u^*\|^2 + \|v - v^*\|^2)/2 \\
&\ge  g(u^*,v^*) - \|\nabla g(u^*,v^*)\| \|(u - u^*, v - v^*)\|
+ e^{-2B} \lambda_{-2}(\mathcal L) (\|u - u^*\|^2 + \|v - v^*\|^2)/2.
\end{align*}
Therefore, using that $g(u^*,v^*) \ge g(u,v)$ and rearranging, we have that
\[  \|\nabla g(u^*,v^*)\| \|(u - u^*, v - v^*)\| \ge e^{-2B} \lambda_{-2}(\mathcal L) (\|u - u^*\|^2 + \|v - v^*\|^2)/2. \]
Dividing through by $\|(u - u^*, v - v^*)\|$  gives
\begin{equation} \|\nabla g(u^*,v^*)\| \ge e^{-2B} \lambda_{-2}(\mathcal L) \|(u - u^*, v - v^*)\|/2.
\label{eqn:taylor} \end{equation}
It remains to control $\|\nabla g(u^*,v^*)\|$.
First observe that
\[ \partial_{u_i} g = \sum_j X_{ij} e^{-v_j + u_i} - p_i = \lambda_i - p_i \]
and $p_i \sim \text{Poisson}(\lambda_i)$ if we define $\lambda_i = \sum_{j} e^{u_i} X_{ij} e^{-v_j}$,
so $\mathbb E \nabla_u g(u^*, v^*) = 0$. 
Next, note that
\begin{align*} 
\|\nabla_u g(u^*,v^*)\|^2 
&= \sum_i (\partial_{u_i} g(u^*,v^*))^2
= \sum_i \left(\lambda_i - p_i\right)^2.
\end{align*}
Since $p_i \sim \text{Poisson}(\lambda_i)$, we have that
\begin{align*} 
\mathbb E \left(\lambda_i - p_i\right)^2 = \lambda_i, \qquad
\mathbb E \left(\lambda_i - p_i\right)^4 = 3\lambda_i^2 + \lambda_i 
\end{align*}
so $\text{Var}\left(\left(\lambda_i - p_i\right)^2\right) = 2\lambda_i^2 + \lambda_i$. Therefore,
\begin{align*} 
\mathbb E\|\nabla_u g(u^*,v^*)\|^2 = \sum_i \lambda_i = \sum_{ij} e^{u_i} X_{ij} e^{-v_j},  \qquad 
\text{Var}\left(\|\nabla_u g(u^*,v^*)\|^2\right) = \sum_j (2\lambda_i^2 + \lambda_i).
\end{align*}
Applying Lemma~\ref{lem:poisson-bernstein} and Lemma~\ref{lem:maximal} yields that
\[ \left\|\max_i (|\lambda_i - p_i|^2 - \lambda_i) \right\|_{\psi_{1/2}} \le K \log^2(1 + m) \max_i (1 + \lambda_i) \]
where $K$ is an absolute constant. 
Hence, applying Theorem~\ref{thm:adam} yields that with probability at least $1 - \delta/2$,
\begin{align*} 
\|\nabla_u g(u^*,v^*)\|^2 - \sum_{i} \lambda_i 
&\le \sqrt{3\log(4/\delta) \sum_i (2\lambda_i^2 + \lambda_i)} + K' \log^2(12/\delta) \log^2(1 + m) \max_i (1 + \lambda_i)   \\
&\le \sqrt{6\log(4/\delta) \sum_i \lambda_i \max_i (1 + \lambda_i)} + K' \log^2(12/\delta) \log^2(1 + m) \max_i (1 + \lambda_i).
\end{align*}
By a completely symmetrical argument, an analogous bound holds for $\|\nabla_v g(u^*,v^*)\|^2$. Summing the two bounds and using the union bound yields that with total probability at least $1 - \delta$,
\[ \|\nabla g(u,v)\|^2 \le 2\sum_{ij} e^{u_i} X_{ij} e^{-v_j} +  \sqrt{24\log(4/\delta) \sum_{ij} e^{u_i} X_{ij} e^{-v_j} (1 + M)} + 2K'\log^2(12/\delta) \log^2(1 + m + n) (1 + M)  \]
where we recall that
\[ M = \max \left\{\sum_j  e^{u_i} X_{ij} e^{-v_j} \mid i \in [m]\right\} \cup \left\{\sum_i e^{u_i} X_{ij} e^{-v_j} \mid j \in [n]\right\} \]
is the maximum of the row and column sums of the matrix $(e^{u_i} X_{ij} e^{-v_j})_{ij}$. Finally, recalling \eqref{eqn:taylor} we see that
\begin{align*} 
&\|(u - u^*, v - v^*)\|^2 \\
&\le 4 \frac{e^{4B}}{\lambda_{-2}(\mathcal L)^2} \|\nabla g(u^*,v^*)\|^2 \\
&\le \frac{8e^{4B}}{\lambda_{-2}(\mathcal L)^2} \left(\sum_{ij} e^{u_i} X_{ij} e^{-v_j} +  \sqrt{6\log(4/\delta) \sum_{ij} e^{u_i} X_{ij} e^{-v_j} (1 + M)} + K'\log^2(12/\delta) \log^2(1 + m + n) (1 + M)\right).
\end{align*}
A simpler version of this argument proves the in-expectation bound (simply combine \eqref{eqn:taylor} with the calculation for $\mathbb E \|\nabla g(u^*,v^*)\|^2$). 
\end{proof}
    
\subsubsection{Concentration inequalities with Orlicz norms}
We make use of powerful concentration of measure estimates from the literature in terms of Orlicz norms. These are useful because we need to give concentration estimates for sums of squares of Poisson random variables, and a standard ``Chernoff bound'' argument cannot be applied to these because their tails are too heavy for moment generating functions to exist. 
\begin{definition}[Orlicz norm]
Given a function $\psi : \mathbb R_{\ge 0} \to \mathbb R_{\ge 0}$, we define the corresponding Orlicz norm of a random variable $X$ by
\[ \|X\|_{\psi} = \inf \{\lambda > 0 : \mathbb E \psi(|X|/\lambda) \le 1 \}. \]
\end{definition}
The Orlicz norm may not be a norm in the usual sense (in particular, satisfy the triangle inequality) if $\psi$ is not convex, but this terminology is commonly used anyway.
\begin{definition}[$\alpha$-exponential Orlicz norm]
For $\alpha > 0$, we define 
\begin{equation}
\psi_{\alpha}(x) = \exp(x^{\alpha}) - 1
\end{equation}
and refer to $\|\cdot\|_{\psi_{\alpha}}$ as the $\alpha$-exponential Orlicz norm.
\end{definition}
The following result arises by specializing the upper tail bound of Theorem 4 of \cite{adamczak2008tail} to the case of the identity function class with $\delta = 1$ and $\eta = 1/2$.
\begin{theorem}[Special case of Theorem 4 from \cite{adamczak2008tail}]\label{thm:adam}
Let $\alpha > 0$. There exists a constant $C_{\alpha} > 0$ such that the following is true.
Suppose that $X_1,\ldots,X_n$ are independent, mean-zero random variables with finite $\alpha$-exponential Orlicz norm. Then
\[ \Pr\left(\left|\sum_{i = 1}^n X_i\right| \ge t\right) \le \exp\left(-\frac{t^2}{3 \sigma^2}\right) + 3\exp\left(-\left(\frac{t}{C_{\alpha} \|\max_{1 \le i \le n} |X_i|\|_{\psi_{\alpha}}}\right)^{\alpha}\right) \]
where
\[ \sigma^2 = \sum_{i = 1}^n \mathbb E X_i^2 \]
\end{theorem}
\begin{remark}
The conclusion of the above theorem directly implies the following. With probability at least $1 - \delta$,
\[ \left|\sum_{i = 1}^n X_i\right| \le \sigma \sqrt{3\log(2/\delta)} + C_{\alpha} \|\max_{1 \le i \le n} |X_i|\|_{\psi_{\alpha}} \sqrt[\alpha]{\log(6/\delta)}. \]
\end{remark}
While $\psi_{\alpha}$ may not be a norm it has similar properties:
\begin{lemma}[Property 4.ii and 4.vi of \cite{chamakh2020orlicz}]
For $c \in \mathbb{R}$,
\[ \|c\|_{\psi_{\alpha}} = |c|/\sqrt[\alpha]{\log 2}. \]
For any random variables $X,Y$ and $\alpha > 0$
\[ \|X + Y\|_{\psi_{\alpha}} \le 2(\|X\|_{\psi_{\alpha}} + \|Y\|_{\psi_{\alpha}}). \]
\end{lemma}
\begin{lemma}[Property 4.viii of \cite{chamakh2020orlicz}]\label{lem:maximal}
For any $\alpha > 0$ and random variables $X_1,\ldots,X_n$,
\[ \left\|\max_{1 \le i \le n} |X_i|\right\|_{\psi_{\alpha}} \le \left[\frac{\log(1 + n)}{\log(3/2)}\right]^{1/\alpha} \max_{1 \le i \le n} \|X_i\|_{\psi_{\alpha}}. \]
\end{lemma}
We also use the following standard tail bound for Poisson random variables (it can be derived from Bernstein's inequality via the Poisson CLT, see \cite{vershynin2018high}).
\begin{lemma}[Standard, see e.g. \cite{vershynin2018high}]\label{lem:poisson-bernstein}
For any $\lambda \ge 0$, if $X \sim \text{Poisson}(\lambda)$ then
\[ \Pr(|X - \lambda| \ge t) \le 2\exp\left(\frac{-t^2}{2(\lambda + t/3)}\right). \]
\end{lemma}
The tail bound implies an upper bound on the subexponential norm of Poisson variables, which can be reinterpreted as a bound on $1/2$-exponential Orlicz norm of their square.  
\begin{lemma}
There exists an absolute constant $C > 0$ such that the following holds.
Suppose that $X \sim \text{Poisson}(\lambda)$ for some $\lambda \ge 0$. Then
\[ \|(X - \lambda)^2\|_{\psi_{1/2}} = \|X - \lambda\|_{\psi_{1}}  \le C(1 + \lambda). \]
\end{lemma}
\begin{proof}
    Recall from \cref{lem:poisson-bernstein}, we have that
    \[ \Pr(|X - \lambda| \ge t) \le 2\exp\left(\frac{-t}{2(\lambda/t + 1/3)}\right). \]
    For any $t \ge 1$, this implies that $\Pr(|X - \lambda| \ge t) \le 2\exp\left(\frac{-t}{2(\lambda + 1/3)}\right) < 2\exp\left(\frac{-t}{6(\lambda + 1/3)}\right)$ and so in fact $\Pr(|X - \lambda| \ge t) \le 2\exp\left(\frac{-t}{6(\lambda + 1/3)}\right)$ for all $t \ge 0$, because the right hand side is larger than $1$ for $t < 1$.
    Therefore the conclusion follows from an equivalent characterization of subexponential random variables, more preccisely Proposition 2.7.1 of \cite{vershynin2018high}.
\end{proof}
\subsection{Well-posedness of maximum likelihood estimation of biproportional Poisson}
\label{sec:app-mle-finite}

Given the biproportional Poisson model \eqref{eqn:model} restated below with $(Y,X,p,q)$ as shorthand for $(\Xtime,\Xagg, \ptime,\qtime)$:
\begin{align*}
    Y_{ij} &\sim \mathrm{Poisson}(e^{u_i} X_{ij} e^{-v_j})\text{ for } X_{ij}>0, \\
    p_i &= \sum_{j,X_{ij}>0} Y_{ij} \\
    q_j &= \sum_{i, X_{ij}>0} Y_{ij},
\end{align*}
we briefly explained in \cref{sec:stats-theory} why the maximum likelihood estimation problem
\begin{align*}
\min g(u,v) & =\sum_{ij}e^{u_{i}}X_{ij}e^{-v_{j}}-\sum_{i}u_{i} p_{i}+\sum_{j}v_{j} q_{j}
\end{align*}
given data $X,p,q$ may not have a finite solution. Here we provide a bit more detail on this matter. The key to the well-posedness problem is the gap between the strong and the weak existence conditions discussed in \citet{qu2023sinkhorn}. For completeness, the conditions are:

\begin{quote}\textbf{Strong Existence.}
    For every pair of sets of indices $I \subsetneq [m]$ and $J \subsetneq [n]$ such that $X_{ij}=0$ for $i\notin I$ and $j\in J$, $\sum_{i\in I}p_i \geq \sum_{j\in J}q_j$, with equality iff $X_{ij}=0$ for all $i \in I$ and $j \notin J$ as well.

   \textbf{Weak Existence.}
   For every pair of sets of indices $I \subsetneq [n]$ and $J \subsetneq [m]$ such that $X_{ij}=0$ for $i\notin I$ and $j\in J$, $\sum_{i\in I}p_i \geq \sum_{j\in J}q_j$.
\end{quote}

Suppose all the Poisson random variables $Y_{ij}$ happen to be non-zero. Then the strong existence condition is satisfied, i.e., there exists a matrix (in fact $Y$) with the same zero pattern as the base matrix $X$ with marginals $p,q$, so the ML estimation problem has a unique (normalized) finite solution. However, as soon as any of the Poisson variables $Y_{ij}$ is zero, only the weak condition is guaranteed, since the matrix $Y$ inherits all zeros of $X$, but now contains additional zeros. We know that in this case, the balancing problem with $(X,p,q)$ does not admit a finite solution. Sinkhorn still converges, but the estimated parameters $u,v$ diverge to $\pm \infty$.
 In practice, we only observe $p,q$, so there is no easy way to check whether any of the \emph{latent} individual Poisson variables $Y_{ij}$ is zero, but in principle, it is possible that the generated data $(p,q)$ is not consistent with the base matrix $X$. The corresponding case in the choice setting is exactly when some set of items always wins. We need to actually show that this event happens with small probability. We now give a proof of \cref{thm:finite-mle}. 

\begin{proof}[Proof of \cref{thm:finite-mle}]
We will show that the strong existence condition holds with high probability. Recall that the negative log-likelihood of the biproportional Poisson model is equivalent to
   \begin{align*}
g(u,v) & =\sum_{ij}e^{u_{i}}X_{ij}e^{-v_{j}}-\sum_{i}u_{i}p_{i}+\sum_{j}v_{j}q_{j}.
\end{align*}
 The Hessian is given by 
\begin{align*}
\mathcal{L}(u,v):=\begin{bmatrix}\mathcal{D}(e^{u}Xe^{-v}\mathbf{1}_{m}) & -e^{u}Xe^{-v}\\
-(e^{u}Xe^{-v})^{T} & \mathcal{D}((e^{u}Xe^{-v})^{T}\mathbf{1}_{n})
\end{bmatrix},
\end{align*}
 and the gradient is given by 
\begin{align*}
(\sum_{j}e^{u_{i}}X_{ij}e^{-v_{j}}-Y_{ij},-\sum_{i}e^{u_{i}}X_{ij}e^{-v_{j}}+Y_{ij}).
\end{align*} 

Note that each $Y_{ij}$ is an independent Poisson random variable
with mean and variance parameter $e^{u_{i}}X_{ij}e^{-v_{j}}$. Now
for each non-empty $I\subsetneq[m]$ and $J\subsetneq[n]$ with $X_{I^{C}J}=0$,
i.e., $X_{ij}\equiv0$ for any $(i,j)$ satisfying $i\in I^{C},j\in J$,
let $E_{IJ}$ be the event that $\sum_{j\in J}q_{j}=\sum_{i\in I}p_{i}$.
The event $E_{IJ}$ corresponds to the case where the realized network
of Poisson variables being \emph{disconnected}, hence the MLE being
unbounded. Our goal is to bound the union of events $E_{IJ}$
for \emph{all} $I,J$ where $X_{I^{C}J}=0$. We have 

\begin{align*}
\mathbb{P}[E_{IJ}] & =\mathbb{P}\left[\sum_{i\in I,j\in J}Y_{ij}=\sum_{i\in I,j\in[n]}Y_{ij}\right]\\
 & =\mathbb{P}\left[\sum_{i\in I,j\in J^{C}}Y_{ij}=0\right]\\
 & =\mathbb{P}\left[\sum_{i\in I,j\in J^{C}}Y_{ij}-e^{u_{i}^{\ast}}X_{ij}e^{-v_{j}^{\ast}}=-\sum_{i\in I,j\in J^{C}}e^{u_{i}^{\ast}}X_{ij}e^{-v_{j}^{\ast}}\right]\\
 & \leq\mathbb{P}\left[\sum_{i\in I,j\in J^{C}}Y_{ij}-e^{u_{i}^{\ast}}X_{ij}e^{-v_{j}^{\ast}}\leq-\sum_{i\in I,j\in J^{C}}e^{u_{i}^{\ast}}X_{ij}e^{-v_{j}^{\ast}}\right]\\
 & \leq\exp(-\frac{1}{2}\sum_{i\in I,j\in J^{C}}e^{u_{i}^{\ast}}X_{ij}e^{-v_{j}^{\ast}}),
\end{align*}
 where in the last step we used Bennett's inequality \citep{bennett1962probability}.\footnote{See also Terry Tao's blog post on improved Bennett's inequality for the Poisson distribution:
\url{https://terrytao.wordpress.com/2022/12/13/an-improvement-to-bennetts-inequality-for-the-poisson-distribution/}.} 
 
Now we connect the bound above to the Hessian evaluated at the true
parameters. Note that for any $i\in[m]$ and $j\in[n]$, 
\begin{align*}
e^{u_{i}^{\ast}}X_{ij}e^{-v_{j}^{\ast}} & =-\mathcal{L}_{i(j+m)}(u^{\ast},v^{\ast})=-\mathcal{L}_{(i+n)j}(u^{\ast},v^{\ast}),
\end{align*}
 and letting $e_{I}\in\mathbb{R}^{m}$ be the indicator vector of
$I$ and $e_{J}\in\mathbb{R}^{n}$ be the indicator vector of $J$,
we have 
\begin{align*}
\sum_{i\in I,j\in J^{C}}e^{u_{i}^{\ast}}X_{ij}e^{-v_{j}^{\ast}} & =-\begin{bmatrix}e_{I}\\
e_{J}
\end{bmatrix}^{T}\mathcal{L}(u^{\ast},v^{\ast})\begin{bmatrix}\mathbf{1}_{m}-e_{I}\\
\mathbf{1}_{n}-e_{J}
\end{bmatrix}
\end{align*}
where we have used the fact that $X_{ij}\equiv0$ for any $(i,j)$
satisfying $i\in I^{C},j\in J$. Now using once again that $\mathbf{1}_{m+n}$
spans the null space of $\mathcal{L}(u,v)$ for any $(u,v)$, we have 
\begin{align*}
\sum_{i\in I,j\in J^{C}}e^{u^\ast_{i}}X_{ij}e^{-v^\ast_{j}} & =\begin{bmatrix}e_{I}\\
e_{J}
\end{bmatrix}^{T}\mathcal{L}(u^{\ast},v^{\ast})\begin{bmatrix}e_{I}\\
e_{J}
\end{bmatrix}\\
 & \geq\lambda_{-2}(\mathcal{L}(u^{\ast},v^{\ast}))\frac{(|I|+|J|)(m+n-|I|-|J|)}{m+n}.
\end{align*}

As a result, we have the following bound on the probability of event
$E_{IJ}$:
\begin{align*}
\mathbb{P}[E_{IJ}] & \leq\exp(-\frac{1}{2}\begin{bmatrix}e_{I}\\
e_{J}
\end{bmatrix}^{T}\mathcal{L}(u^{\ast},v^{\ast})\begin{bmatrix}e_{I}\\
e_{J}
\end{bmatrix})\\
 & \leq\exp(-\frac{1}{2}\lambda_{-2}(\mathcal{L}(u^{\ast},v^{\ast}))\frac{(|I|+|J|)(m+n-|I|-|J|)}{m+n}).
\end{align*}
 Now applying a union bound on all $I,J$ with $X_{I^{C}J}=0$ and following the arguments from Lemma 1 in \citet{simons1999asymptotics}, we
have 
\begin{align*}
\mathbb{P}(\text{MLE does not exist}) & \leq\sum_{I\subsetneq[m],J\subsetneq[n],X_{I^{C}J}=0}\exp(-\frac{1}{2}\lambda_{-2}(\mathcal{L}(u^{\ast},v^{\ast}))\frac{(|I|+|J|)(m+n-|I|-|J|)}{m+n})\\
 & \leq2\left[\left(1+\exp(-\frac{1}{4}\lambda_{-2}(\mathcal{L}(u^{\ast},v^{\ast})))\right)^{m+n}-1\right].
\end{align*}
 Finally, plugging in $\lambda_{-2}(\mathcal{L}(u^{\ast},v^{\ast}))\geq8\log(m+n)$
gives 
\begin{align*}
\mathbb{P}(\text{MLE does not exist}) & \leq2\left[\left(1+\exp(-2\log(m+n))\right)^{m+n}-1\right]\\
 & =2\left[\left(1+(m+n)^{-2}\right)^{m+n}-1\right]\\
 & \leq\frac{2}{\sqrt{m+n}}.
\end{align*} 
\end{proof}
\textbf{Intuition on sufficient condition \eqref{eq:Fisher-bound}.} We provide some intuition why a ``large'' $\lambda_{-2}(\mathcal{L}^\ast)$ ensures finite solutions exist. Consider again scaling $\Xagg$ by a constant $c>1$, which also scales up $\lambda_{-2}(\mathcal{L}^\ast)$ and the Poisson parameters $e^{u^\ast_i} \Xagg_{ij} e^{-v^\ast_j}$ by $c$. As $c$ increases, the probability of the Poisson entries $\Xtime_{ij}$ drawing zero decreases exponentially. As a result, $\Xtime$ is much more likely to have the same zero patterns as $\Xagg$, which implies bounded MLEs \citep{qu2023sinkhorn}.
\section{Details on our convergence algorithm, \texttt{ConvIPF}} 
\label{sec:convergence-algo}
In this section, we provide details on our algorithm, \texttt{ConvIPF}, for achieving IPF convergence.
First, we review the conditions for IPF convergence and discuss two implications: (1) IPF will always converge given a true time-aggregated network (Corollary \ref{cor:timeagg-converge}), (2) IPF will always converge on data generated from our network model (Corollary \ref{cor:model-converge}).
Then, we discuss our algorithm, \texttt{ConvIPF}, which repeats three subroutines: \texttt{MAX-FLOW}, \texttt{BLOCKING-SET}, and \texttt{MODIFY-X}.
In Appendix \ref{sec:convergence-test}, we describe \texttt{MAX-FLOW}, which is also described in \citet{idel2016review}.
In Appendix \ref{sec:app-blocking-set}, we provide our algorithm for \texttt{BLOCKING-SET} and prove that it efficiently finds a blocking set of rows.
In Appendix \ref{sec:app-modify-x}, we formalize the minimization objective in \texttt{MODIFY-X} and present two principled approaches to solving it.

\paragraph{Conditions for IPF convergence.}
Three equivalent conditions \citep{pukelsheim2014} that define exactly when IPF converges, given inputs $\ipfmat$, $\ipfrow$, and $\ipfcol$, are:
\begin{enumerate}
    \item There exists a matrix $M$ with row sums $\ipfrow$ and column sums $\ipfcol$ such that $M_{ij} = 0$ wherever $X_{ij} = 0$.
    \item For all row subsets $S \subseteq [m]$, $\sum_{i \in S} \ipfrow_i \leq \sum_{j \in N_X(S)} \ipfcol_j$, where $N_\ipfmat(S)$ represents the set of columns connected to $S$ in $\ipfmat$. 
    \item There exist positive diagonal matrices $D^0$, $D^1$ such that $D^0 X D^1$ has row sums $\ipfrow$ and column sums $\ipfcol$.
\end{enumerate}
Condition (1) yields the \texttt{MAX-FLOW} algorithm for testing whether IPF will converge, which we describe in Appendix \ref{sec:convergence-test}.
Condition (2) motivates our \texttt{BLOCKING-SET} algorithm (Appendix \ref{sec:app-blocking-set}), which efficiently identifies a ``blocking set'' of rows for which the condition is violated.
Conditions (1) and (3) yield the following useful corollaries, which establish two settings where IPF is guaranteed to converge.

\begin{corollary}
\label{cor:timeagg-converge}
Let $\mathcal{T} = \{t_1, t_2, \cdots, t_T\}$ represent a set of timesteps and $\Xagg = \sum_{t' \in \mathcal{T}} X^{(t')}$ represent the time-aggregated network.
Given the time-varying network $\Xtime$, with row marginals $\ptime$ and column marginals $\qtime$, IPF will always converge on inputs $\Xagg$, $\ptime$, and $\qtime$ if $t \in \mathcal{T}$.
\end{corollary}
\begin{proof}
    We know that IPF must converge on $\Xtime$, $\ptime$, and $\qtime$, since $\Xtime$ has marginals $\ptime$ and $\qtime$, so simply setting $d^0 = \mathbf{1}_m$ and $d^1 = \mathbf{1}_n$ solves the matrix balancing problem, satisfying Condition (3). 
    Then, Condition (2) must also be satisfied for $\Xtime$, $\ptime$, and $\qtime$. 
    Since $\Xagg$ sums over a set of matrices including $\Xtime$, for any row subset $S$ in $[m]$, its set of connected columns under $\Xagg$ must be a superset of its connected columns under $\Xtime$.
    Thus, if Condition (2) is satisfied for $\Xtime$, $\ptime$, and $\qtime$, then it must be satisfied for $\Xagg$, $\ptime$, and $\qtime$; therefore, IPF converges on $\Xagg$, $\ptime$, and $\qtime$.
\end{proof}

\begin{corollary}
\label{cor:model-converge}
Let $\Xtime, \ptime, \qtime$ represent data generated from our network model \eqref{eqn:model}, given any $\Xagg$ and scaling parameters $u$ and $v$.
IPF will always converge on $\Xagg$, $\ptime$, and $\qtime$.
\end{corollary}
\begin{proof}
    Here, we do not require $\Xagg$ to be a sum over some set of matrices that includes $\Xtime$.
    However, a similar logic applies: first, by the same argument as above, we know that IPF must converge on $\Xtime$, $\ptime$, and $\qtime$.
    Furthermore, under our model \ref{eqn:model}, $\Xtime$ adopts all zeros in $\Xagg$ (with potential additional zeros).
    So, again, for any row subset $S$ in $[m]$, its set of connected columns under $\Xagg$ must be a superset of its connected columns under $\Xtime$.
    Thus, if Condition (2) is satisfied for $\Xtime$, $\ptime$, and $\qtime$, then it must be satisfied for $\Xagg$, $\ptime$, and $\qtime$; therefore, IPF converges on $\Xagg$, $\ptime$, and $\qtime$.
\end{proof}

These corollaries establish useful facts relating IPF, our network inference problem, and our model.
Note that for Corollary \ref{cor:timeagg-converge}, we do not require the data to be generated from our model, only that the true time-aggregated network is given.
In contrast, for Corollary \ref{cor:model-converge}, we do not require that the true time-aggregated network is given, only that the data is generated from the model.
These corollaries also serve as ``certificates'' of data quality and model correctness.
If IPF does \textit{not} converge, then the statements above cannot be true: we do not observe the true time-aggregated network $\Xagg$ (at least one that includes $t$) and the data $\ptime$ and $\qtime$ is not generated from the model given $\Xagg$.
This motivates our perspective of IPF non-convergence as an issue of missing data in $\Xagg$, and the need for additional, but minimal, new edges. 
In Appendix \ref{sec:app-safegraph}, we also discuss real-world reasons for missingness in $\Xagg$, with mobility data as an example.

Our algorithm, \texttt{ConvIPF}, achieves IPF convergence by iteratively identifying a blocking set of rows, for which the second convergence condition is violated, and unblocking those rows by minimally adding edges that connect those rows to new columns.
Since our algorithm can be applied generally to any application of IPF, not only in the network inference setting, we use generic notation of $X$, $p$, and $q$ in the following sections, to represent general IPF inputs.
Our algorithm repeats three subroutines, \texttt{MAX-FLOW}, \texttt{BLOCKING-SET}, and \texttt{MODIFY-X}, until IPF converges:
\begin{enumerate}
    \item Run \texttt{MAX-FLOW} to test for convergence. If IPF converges, then the algorithm is finished. If IPF does not converge, move on to Step 2.
    \item Since IPF does not converge, run \texttt{BLOCKING-SET} to identify a blocking set of rows, $S$.
    \item Run \texttt{MODIFY-X} to unblock $S$ by minimally adding edges to $\ipfmat$.
\end{enumerate}
Our algorithm is similar in structure to the multi-item auction procedure described by \citet{demange1986auction} for finding market-clearing prices. Their procedure iteratively (1) checks if there is a perfect matching of item-buyer based on current prices, (2) if not, finds the set of constricted buyers $S$ and their neighbors $N(S)$, (3) adjusts prices in $N(S)$ to attempt to fix the matching. These steps map directly onto our steps 1-3 and step 2 is especially related because of the connection between blocking row sets in IPF and constricted sets in bipartite matching \citep{hall1935subset}, which we discuss in Appendix \ref{sec:app-blocking-set}.
Below, we describe each of our subroutines in detail.

\subsection{\texttt{MAX-FLOW}} 
\label{sec:convergence-test}
For completeness, we describe the algorithm from \citet{idel2016review} for testing whether IPF will converge on inputs $\ipfmat$, $\ipfrow$, and $\ipfcol$.
Condition (1) described that IPF converges if and only if there exists a matrix $A$ with row sums $\ipfrow$ and column sums $\ipfcol$ such that $A_{ij} = 0$ wherever $X_{ij} = 0$, i.e., $A$ inherits the zeros of $X$.
Note that this matrix is \textit{more general} than the set of possible solutions to the matrix balancing problem, since $A$ does not have to be a biproportional scaling of $\ipfmat$.
Now, the following algorithm will check for the existence of $A$.
This algorithm is closely related to the max-flow-based algorithm from \citet{kumar2015wsdm} to test their concept of graph consistency, further establishing connections between IPF and discrete choice.

\paragraph{\texttt{MAX-FLOW} algorithm.}
Create a new directed graph $G_f$ that has a source node $s$ connected to one node $n_i$ for each row and set the capacity of the edge $s \rightarrow n_i$ to $\ipfrow_i$.
Create a sink node $t$ connected to one node $n_j$ for each column and set the capacity of the edge $n_j \rightarrow t$ to $\ipfcol_j$.
Finally, include an edge $n_i \rightarrow n_j$, with capacity $\infty$, wherever $\ipfmat_{ij} > 0$.
Compute the maximum flow on $G_f$.
If the maximum flow is equal to $\sum_i \ipfrow_i = \sum_j \ipfcol_j$, then the desired matrix $A$ exists, meaning IPF converges for $\ipfmat$, $\ipfrow$, and $\ipfcol$; otherwise, it does not converge.

\begin{proof}
The reasoning is as follows: the matrix $A$ can be constructed from the flow values, where $A_{ij}$ is set to $flow(n_i, n_j)$ (i.e., the flow along edge $n_i \rightarrow n_j$) wherever $X_{ij} > 0$ and $0$ otherwise.
This satisfies the constraint that $A$ inherits the zeros of $X$.
Since the total flow is equal to $\sum_i \ipfrow_i$ and the source node $s$ is only connected to the row nodes, with capacity $p_i$ along each edge, then each row node $n_i$ must have exactly $p_i$ flowing through it.
Due to the conservation of flow, it must be true that  $p_i = \sum_j flow(n_i, n_j)$ for all rows $i$, so the row marginals are satisfied.
A similar argument can be made for the columns: since the sink node $t$ is only connected to the column nodes, with capacity $q_j$ along each edge, then each column node $n_j$ must have exactly $q_j$ flowing through it, so $\sum_i flow(n_i, n_j) = q_j$ for all columns $j$.
Thus, $A$ is a matrix that inherits the zeros of $X$ and satisfies the row and column marginals.
\end{proof}

\subsection{\texttt{BLOCKING-SET}}
\label{sec:app-blocking-set}

Given inputs $\ipfmat$, $\ipfrow$, and $\ipfcol$, where we know IPF does not converge, this subroutine identifies a blocking set of rows $S$ for which Condition (2) is violated, i.e., $\sum_{i \in S} \ipfrow_i > \sum_{j \in N_\ipfmat(S)} \ipfcol_j$.
The naive approach to iterate through all subsets until a violation is found, but this approach is extremely inefficient, as there are $2^m$ possible subsets.
Instead, our subroutine imports ideas from matching theory and constricted sets to design a much more efficient algorithm. 

\paragraph{Connection to constricted sets.}
The idea of blocking subsets in IPF is closely related to the concept of \textit{constricted sets} in bipartite matching problems.
Given a bipartite graph $G = L \cup R$, a constricted set is a set of nodes $S \subseteq L$ such that $|N(S)| < |S|$, where $N(S)$ represents the set of neighbors in $R$ who are connected to $S$.
Then, a perfect matching on $G$ (where all nodes are matched) exists if and only if no constricted set in $G$ exists \citep{hall1935subset}.

Given a graph without a perfect matching, one can efficiently find a constricted set via an alternating breadth-first search (BFS) algorithm.
First, start with the maximal matching from the graph.
Then, run BFS from an unmatched node on the \textit{right} side.
Since the graph is bipartite, BFS will alternate between nodes on the right side and nodes on the left side.
The key is that we only keep \textit{alternating} edges during BFS: when we move from the right to the left side, we keep edges that are unused in the matching; when we move from the left to the right side, we keep edges that are used in the matching.
When alternating BFS terminates, the set of nodes visited on the \textit{right} side forms a constricted set.
This is due to the following reasons, as explained in \citet{easleykleinberg}:
\begin{enumerate}
    \item All of the nodes visited on the \textit{left} side must be matched. Otherwise, we would have an ``augmenting path'', i.e., one that starts with an unmatched node on the right and ends with an unmatched node on the left, alternating between unused and used edges. This is called an augmenting path because by flipping the edges from unused to used and used to unused, we can increase the size of the matching. We can assume that no augmenting paths appear in alternating BFS, since we already found the maximal matching; thus, all visited nodes on the left must be matched.
    \item Each odd layer (left side) has the same number of nodes as the subsequent layer (right side), since we follow used edges from left to right and each visited node on the left side is matched, so in the subsequent layer, we take each left node's partner. So, there are strictly more nodes in even (right) layers than odd (left) layers, since we start with a node on the right.
    \item Every node in an even (right) layer has all of its neighbors in the graph visited during BFS, since it was added by its match in the previous layer then we add all of its non-match neighbors in the following layer.
\end{enumerate}
Thus, the nodes visited on the right side form a constricted set, since all of their neighbors are included in the odd layers during BFS, but there are strictly more nodes in the even than odd layers.

\paragraph{\texttt{BLOCKING-SET} algorithm.}
We propose the following algorithm, inspired by the alternating BFS algorithm for constricted sets, for finding a blocking row set in IPF.
First, from running \texttt{MAX-FLOW} on the flow graph $G_f$ to test for convergence, we have flow values for each row/column.
We say that a row $i$ reached its capacity if the flow passing through it is $p_i$ and a column $j$ reached its capacity if the flow passing through it is $q_j$.
Since the maximum flow is not equal to $\sum_i p_i$, there must be at least one row $i$ that does not reach its capacity. 
Construct an undirected bipartite graph $G$ where the nodes on the left and right are the rows and columns of $X$, respectively, and they are connected wherever $X_{ij} > 0$ (this graph is similar to but distinct from the flow graph $G_f$, which is directed and only has edges in the direction of rows to columns).
Run the following variant of BFS on $G$ starting from node $n_i$.
As in the alternating BFS algorithm, each layer of BFS here will alternate between nodes on the left (rows) and nodes on the right (columns).
When progressing from a column node $n_C$ to its neighboring row node $n_R$, only include row nodes where $n_R \rightarrow n_C$ has non-zero flow in $G_f$.
When progressing from row nodes to column nodes, include all (unvisited) neighbors.
When this version of BFS terminates, the set of row nodes visited forms a blocking subset of rows.

\begin{proof}
First, we can use an argument similar to that of augmenting paths from bipartite matching.
Here, imagine that we encountered a column $j$ during BFS that has \textit{not} reached capacity.
Then, there is a path from node $n_i$ to node $n_j$ in $G$; for example, $n_i - n_{C_1} - n_{R_1} - n_{C_2} - n_{R_2} - n_j$, where $C_1$, $C_2$ and $R_1$, $R_2$ are other columns and rows, respectively.
\begin{figure}
    \centering
    \includegraphics[width=0.5\linewidth]{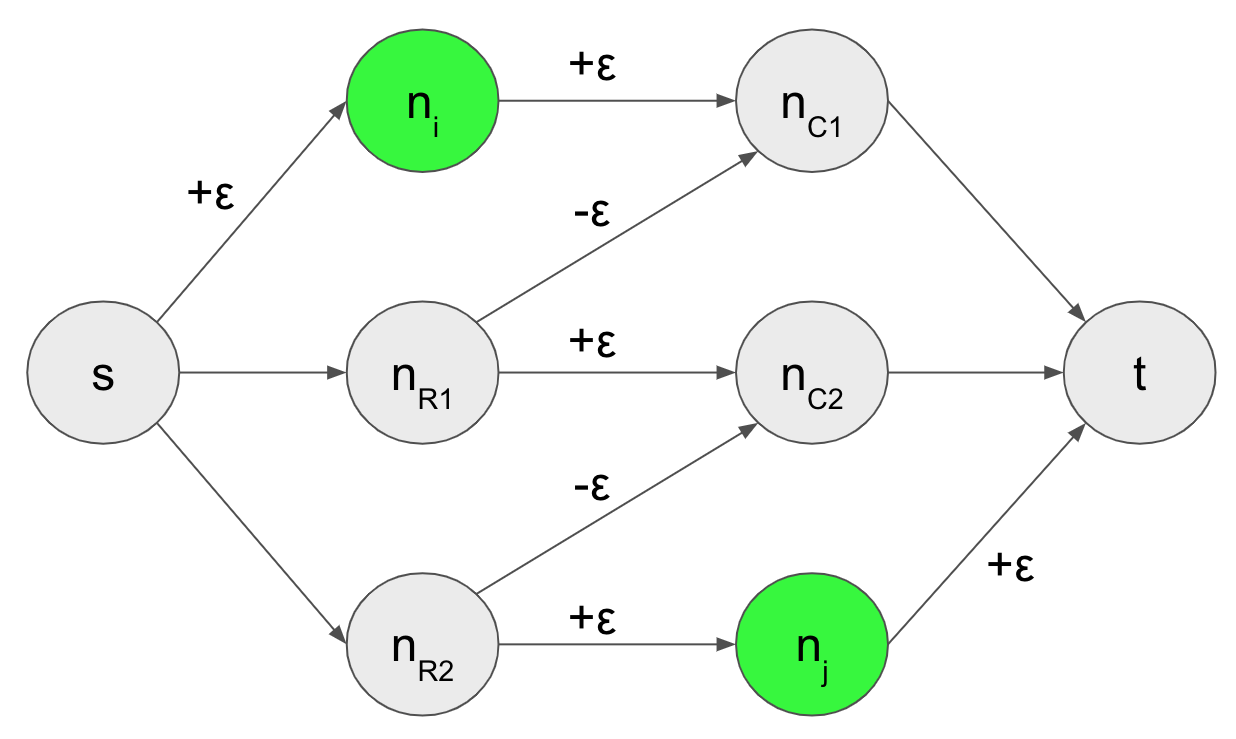}
    \caption{Example of how we would modify $G_f$ to increase the overall flow by $\epsilon$, given nodes $n_i$ and $n_j$ (in green) that have \textit{not} reached capacity.}
    \label{fig:augmenting-path}
\end{figure}
Then, we could increase the overall flow in $G_f$ by adding $\epsilon$ to $s \rightarrow n_i$ and $n_i \rightarrow n_{C_1}$, 
removing $\epsilon$ from $n_{R_1} \rightarrow n_{C_1}$ and adding $\epsilon$ to $n_{R_1} \rightarrow n_{C_2}$,
removing $\epsilon$ from $n_{R_2} \rightarrow n_{C_2}$ and adding $\epsilon$ to $n_{R_2} \rightarrow n_j$,
and adding $\epsilon$ to $n_j \rightarrow t$ (Figure \ref{fig:augmenting-path}).
In other words, the flows at $n_{R_1}$, $n_{R_2}$, $n_{C_1}$, and $n_{C_2}$ remain the same, but we increase the overall flow by rerouting additional flow through $n_i$ and $n_j$.
Note that this is only possible since $n_i$ and $n_j$ are under capacity, and we restricted our BFS to paths where, for consecutive nodes $n_C$ and $n_R$, $n_R \rightarrow n_C$ has non-zero flow in $G_f$  (so we can remove $\epsilon$ from those edges). 
For any such path, we can do this rerouting to increase the overall flow. 
However, since we have found the maximum flow, such a path cannot exist.
Therefore, all of the columns we visited during BFS must have reached capacity.

The remaining arguments are as follows:
\begin{enumerate}
    \item Due to the conservation of flow in $G_f$, the flow going through the visited rows must be equal to the flow going through the visited columns, since our BFS includes all rows that are contributing non-zero flow to these columns.
    \item The total capacity of the visited rows must be strictly greater than the total capacity of the visited columns, since we established above that all visited columns must have reached capacity, but we start with a row node that did \textit{not} reach capacity. 
    \item Furthermore, all of the visited rows' neighbors are included among the visited columns, since our BFS only restricted edges from columns to rows, but not from rows to columns.
\end{enumerate}
Since capacity is equivalent to the original IPF marginals, this means we have identified a set of rows whose total marginal is larger than the total marginal of their connected columns.
Thus, the set of rows visited during this version of BFS forms a blocking subset.    
\end{proof}

\subsection{\texttt{MODIFY-X}} 
\label{sec:app-modify-x}
Given a blocking set $S$, this subroutine minimally adds edges to $\ipfmat$ to unblock $S$. 
Let $\ipfmat^K$ represent the modified $\ipfmat$ after adding new edges $K$ and let $f(\ipfmat, \ipfmat^K)$ represent the change in $\ipfmat$ that we are trying to minimize.
Then, our goal is to find the set of edges $K^*$ that minimizes $f(\ipfmat, \ipfmat^K)$,
\begin{align}
    K^* &= \min_{K} f(X, X^K),
\end{align}
subject to unblocking $S$,
\begin{align}
    \sum_{i \in S} p_i \leq \sum_{j \in N_{\ipfmat^K}(S)} q_j. \label{eqn:unblock}
\end{align}
Given a current set of edges, each additional edge $(i, j)$ will only contribute towards satisfying \eqref{eqn:unblock} if row $i$ is in $S$ and no row in $S$ is already connected to column $j$.
Thus, it is clear that we should only consider edge sets $K$ such that each column in $K$'s edges are unique and not in the original set of $S$'s neighbors, $N_X(S)$.
Furthermore, under such $K$, the constraint \eqref{eqn:unblock} reduces to
\begin{align}
    \sum_{(i, j) \in K} q_j \geq \delta, \label{eqn:qj-remainder}
\end{align}
where $\delta := \sum_{i \in S} p_i - \sum_{j \in N_X(S)} q_j$ is the gap in marginals we are trying to make up to unblock $S$.

Below, we discuss two natural definitions of $f(\ipfmat, \ipfmat^K)$, along with appropriate algorithms for these objectives.
Once we have a set of selected edges $\hat{K}$, we add them to $X$ with a small uniform weight.
Since we are trying to minimize $f(X, X^K)$, which may be affected by the choice of edge weight, while we are trying to increase $\sum_{j \in N_{\ipfmat^K}(S)} q_j$, which is not affected by edge weight (only depends on the structure of the graph), we should keep the weight as small as possible.

\subsubsection{Minimizing number of edges added}
One simple objective would be to minimize the number of new edges added, so $f(\ipfmat, \ipfmat^K) = |K|$.
The solution in this case is straightforward. 
Let $\bar{N}_\ipfmat(S)$ represent the set of columns \textit{not} connected to $S$ in $\ipfmat$.
Take the top-$k$ columns in $\bar{N}_\ipfmat(S)$, ordered by $\ipfcol_j$ in descending order, that satisfy $\sum_{j=1}^k \ipfcol_j \geq \delta$.
Then, any set of edges between a row in $S$ and these $k$ columns will unblock $S$, while minimizing the number of new edges added.
For example, we could arbitrarily choose a row $i^* \in S$ as the row with the largest $p_i$, then set $K^* = \{(i^*, j)\}_{j=1}^k$.

\subsubsection{Minimizing change in $\lambda_1$}
Minimizing the number of edges in $K$ may feel insufficient, since the objective is coarse and leaves many degrees of freedom, such as the choice of row in $S$.
A more nuanced objective minimizes the \textit{spectral change} in $X$.
Motivated by the application of epidemic spread, recall that the epidemic threshold of a network is closely related to the largest eigenvalue $\lambda_1$ of its adjacency matrix \citep{wang2003eigenvalue}.
In fact, attempts to reduce spreading on networks often aim to minimize $\lambda_1$ through budgeted edge removals \citep{tong2012gel,saha2015radius,li2023sdm}.
So, our goal is to modify $X$ to unblock $S$, but otherwise change $\lambda_1$ \textit{as little as possible}.
This goal is related to spectral sparsification, which aims to sparsify a graph (i.e., greatly reduce its edges) while approximately preserving the spectrum of the graph \citep{spielman2010spectral}.
In contrast, we aim to add edges to $X$ while minimizing change in its largest eigenvalue.
So, we have
\begin{align}
    f(X, X^K) = |\lambda_1(X) - \lambda_1(X^K)| = \lambda_1(X^K) - \lambda_1(X).
\end{align}
We are able to make this simplification since, by the Perron-Frobenius Theorem, a connected graph's largest eigenvalue strictly increases when an edge is added.

To solve this problem, we use a common approximation:
\begin{align}
    \lambda_1(X^K) - \lambda_1(X) \approx \sum_{(i,j) \in K} \vec{u}_1(i) \vec{v}_1(j), \label{eqn:eigen-approx}
\end{align}
where $\vec{u}_1$ and $\vec{v}_1$ are the left and right eigenvectors of $\lambda_1(X)$, respectively.
In the closely related problem of removing $k$ edges from a graph $G$ to minimize $\lambda_1(G)$, \citet{tong2012gel} prove that there is only an $O(k)$ gap between approximate impact and actual impact on $\lambda_1$.
Let $\lambda_1(G)$ represent the original largest eigenvalue and $\lambda_1(G-R)$ represent the eigenvalue after removing the edges in $R$.
Then, \citet{tong2012gel} show that
\begin{align}
    \lambda_1(G) - \lambda_1(G-R) = c \sum_{(i,j) \in R} \vec{u}_1(i) \vec{v}_1(j) + O(k).
\end{align}
Since they approximate each edge's impact as \textit{independent}, then their solution simply removes the top-$k$ edges with the largest $\vec{u}_1(i) \vec{v}_1(j)$.
In our case, we want to find the opposite, i.e., the  edges that minimize $\vec{u}_1(i) \vec{v}_1(j)$, while unblocking $S$.
Our algorithm first finds the row $i^*$ in $S$ with the smallest $\vec{u}_i(i)$, which we should use for all edges in $\hat{K}$ since unblocking $S$ is agnostic to which row in $S$ is chosen but $\lambda_1$ approximately increases with larger $\vec{u}_i(i)$.
Now, based on our prior analysis, we want to find a set of columns $J \subseteq \bar{N}_X(S)$, where $\sum_{j \in J} q_j \geq \delta$ \eqref{eqn:qj-remainder}, while minimizing $\sum_{j \in J} \vec{v}_1(j)$.

We can formulate this objective as an integer linear program (ILP), by using $x \in \mathbb{R}^{|\bar{N}_X(S)|}$ as an indicator representing which columns in $\bar{N}_X(S)$ should be included:
\begin{align}
    x^* = &\min_{x} \sum_{j \in \bar{N}_X(S)} \vec{v}_1(j) \cdot x_j, \label{eqn:ilp} \\
    &\textrm{such that } \sum_{j \in \bar{N}_X(S)} q_j \cdot x_j \geq \delta \textrm{ and } x_j \in \{0, 1\}.\nonumber
\end{align}
Then, we keep the set of columns $J$ where $x^*_j = 1$ and our algorithm returns $\hat{K} = \{(i^*, j) | j \in J\}$, which is guaranteed to unblock $S$ while approximately minimizing the change in the largest eigenvalue of $X$.

Our ILP in fact reduces to the classic Knapsack Problem, where, given a set of $\eta$ items, each with a weight $w_i$ and value $v_i$, the goal is to maximize the total value of the items chosen subject to the total weight being under a given capacity $W$.
By flipping the signs of $v_1(j)$ and $q_j$ in \eqref{eqn:ilp}, our ILP becomes identical to the Knapsack Problem.
The Knapsack Problem, like ILPs in general, is NP-hard, but it can be solved in pseudo-polynomial time using dynamic programming, which yields a time complexity of $O(\eta W)$ (pseudo because $W$ is not the size of input).
In our case, $\eta$ corresponds to $|\bar{N}_X(S)|$, which is worst-case $n$ if $S$ is not connected to any columns, and $W$ corresponds to $\delta$, which is worst-case $\sum_i p_i$ if $r_p(S) = \sum_i p_i$ and $c_{X,q}(S) = 0$.
So, we can solve our problem in pseudo-polynomial time too, or, if strictly polynomial time is desired, we can use the Fully Polynomial-Time Approximation Scheme (FPTAS), which finds a solution in $O(n^3/\epsilon)$ time that is at least $(1 - \epsilon)$ times the optimal value.

\paragraph{Further justification of eigenscore approximation.}
Here, we further justify the eigenscore approximation in \eqref{eqn:eigen-approx} with a derivative-based approach. 
We consider minimizing the increase in $\lambda_{1}$ in the limit when the added edge weight $\epsilon\rightarrow0$.
In other words, we consider minimizing the infinitesimal change in
the largest eigenvalue when adding edges. 
The first observation is that for each candidate column $j\in\bar{N}_{X}(S)$
that we want to add an edge, it is optimal to select only one row $i \in S$
to add the edge. To see this, consider the derivative 
\begin{align*}
\frac{\partial\lambda_{1}}{\partial X{}_{ij}} & =\lim_{\epsilon\rightarrow0}\frac{\lambda_{1}(X+\epsilon e_{i}e_{j}^{T})-\lambda_{1}(X)}{\epsilon}
\end{align*}
 and for any $j\in\bar{N}_{X}(S)$ define 
\begin{align*}
\theta_{j} & :=\min_{i\in S}\frac{\partial\lambda_{1}}{\partial X{}_{ij}}.
\end{align*}
In other words, $\theta_{j}$ is the \emph{minimal} infinitesimal
increase (cost) in $\lambda_{1}$ when adding an edge from $S$ to
$j$, and we only need to consider the row $i\in S$ that achieves $\theta_{j}$. 
We can therefore consider the following problem: 
\begin{align*}
&\min_{x}\sum_{j\in\bar{N}_{X}(S)}\theta_{j}\cdot x_{j}\\
&\textrm{such that }\sum_{j\in\bar{N}_{X}(S)}q_{j}\cdot x_{j}\geq\delta \textrm{ and }x_{j}\in\{0,1\}.
\end{align*}

This is again a canonical knapsack problem. Moreover, Corollary 2.4
of \citet{stewart1990matrix} gives 
\begin{align*}
\theta_{j} & =\min_{i\in S}\frac{\vec{u}_{1}^{T}e_{i}e_{j}^{T}\vec{v}_{1}}{\vec{u}_{1}^{T}\vec{v}_{1}}
\end{align*}
 where $\vec{u}_{1},\vec{v}_{1}$ are the left and right eigenvectors of $\lambda_{1}$, respectively.
The problem can be simplified because we just choose $i$ with minimal
$\vec{u}_{1}^{T}e_{i}$ for any $j\in\bar{N}_{X}(S)$, and the problem reduces
to the same knapsack problem as above \eqref{eqn:ilp}:
\begin{align*}
&\min_{x}\sum_{j\in\bar{N}_{X}(S)}\vec{v}_{1}(j)\cdot x_{j}\\
&\textrm{such that }\sum_{j\in\bar{N}_{X}(S)}q_{j}\cdot x_{j}\geq\delta \textrm{ and }x_{j}\in\{0,1\}.
\end{align*}
We therefore have a rigorous justification of the above ILP problem
in terms of minimizing the instantaneous increase in largest eigenvalue
while unblocking the subset $S$.

\subsection{Toy example}
To build intuition, here we provide a toy example of data where IPF does not converge:
\begin{align}
    X &= \begin{bmatrix}
    1 & 0 & 0\\
    1 & 1 & 0\\
    0 & 1 & 0\\
    0 & 1 & 1
    \end{bmatrix} \\
    p &= [1, 1, 1, 1], q = [1, 1, 2]. \nonumber
\end{align}
IPF will not be able to converge on these inputs, since rows 0, 1, and 2 form a blocking subset: their total row marginal is 3 but their total column marginal (for columns 0 and 1) is 2.
Intuitively, the problem is that row 0 is only connected to column 0, so all of its marginal must be used on column 0, leaving no marginal left in column 0. So, even though row 1 is also connected to column 0, it must use all of its marginal on its remaining connection, column 1, which leaves no room for row 2's marginal, where row 2 is only connected to column 1.

First, we verify that \texttt{BLOCKING-SET} returns $S = \{0, 1, 2\}$.
The marginal gap in this case is $\delta = 1$.
For $\texttt{MODIFY-X}$, there is only one column that $S$ is not connected to, which is column 2. 
If we use the objective of minimizing number of new edges, we would add one edge based any row in $S$ and column 2, which unblocks $S$ and allows IPF to converge.
In comparison, if we added $\epsilon$ to all zeros, as prior works have done to enable IPF to converge, we would add five edges.
Alternatively, if we use the objective of minimizing change in $\lambda_1$, our algorithm identifies row 0 as the row in $S$ with the smallest $\vec{u}_1(i)$.\footnote{Following \citet{tong2012gel}, we use $-\vec{u}_1(i)$ when $\min_i \vec{u}_1(i)$ is negative, and similarly use $-\vec{v}_1(j)$ when $\min_j \vec{v}_1(j)$ is negative, to ensure that all eigenscores are non-negative. According to the Perron-Frobenius theorem, such eigenvectors always exist.}
In this case, there is only one column, column 2, that is not connected to $S$, so the algorithm adds an edge between row 0 and column 2 with weight $\epsilon = 0.01$.
The change in $\lambda_1$ from this modification is 0.0007. 
In comparison, if we connected row 1 to column 2 instead, the change is $\lambda_1$ is 0.0019 ($2.8\times$), and if we connected row 2 to column 2, the change is $\lambda_1$ is 0.0013 ($1.8\times$).
If we added $\epsilon$ to all zero entries, the change in $\lambda_1$ is 0.010 ($14.3\times$), which is much larger.
So, this toy example demonstrates how, even on a tiny matrix $X$ with relatively few zero values, the difference between our algorithm and the naïve solution is sizable, and our approximation of change in $\lambda_1$ still effectively minimizes that change in practice.
On real-world data with much larger $X$, the differences are far more dramatic: in Section \ref{sec:mobility-ipf}, we show on real-world mobility data that our algorithm reduces change in $\lambda_1$ by orders of billions.

\section{Experiments with data} 
\label{sec:app-empirics}
In this section, we provide details on our experiments with data and describe additional experimental results.
In Appendix \ref{sec:synthetic}, we describe experiments with synthetic data, which enables us to test IPF's ability to recover \textit{true} network parameters $u$ and $v$.
We test IPF under correct model specification and model misspecification, trying alternate models with non-Poisson distributions or interaction terms.
In Appendix \ref{sec:app-safegraph}, we describe experiments with mobility data from SafeGraph, which provide a realistic example of our network inference setting, where only the hourly marginals and time-aggregated network are provided, with missingness in the time-aggregated network.
We show that IPF fails to converge several times on this data, demonstrating the importance of principled solutions for non-convergence, and evaluate our new convergence algorithm \texttt{ConvIPF} on this data, revealing dramatic improvements over prior solutions.
In Appendix \ref{sec:app-bikeshare}, we describe experiments with bikeshare data from New York City's Citibike system, where we are able to construct ground-truth hourly networks.
These ground-truth networks allow us to evaluate IPF ability to estimate hourly networks in our network inference setting (with hourly marginals and time-aggregated network) and compare it to baseline methods.
We also use these networks as an opportunity to test our model assumptions on real-world data.

\subsection{Synthetic data}
\label{sec:synthetic}
\paragraph{Generating synthetic data.}
We generate synthetic data based on our biproportional Poisson model \eqref{eqn:model}.
In these experiments, we set $m = n = 100$, and generate data in the following order:
\begin{enumerate}
    \item We sample the row scaling factors $e^{u} \in \mathbb{R}^m$ and column scaling factors $e^{-v} \in \mathbb{R}^n$ from Uniform(0, 4).
    \item We sample $\Xagg \in \mathbb{R}^{m \times n}$ from Uniform(0, 1).
    \item For a given sparsity level $r \in [0, 1)$, we randomly select $r \cdot mn$ entries from $\Xagg$ (without replacement) and set them to 0.
    \item We sample each $\Xtime_{ij}$ from Poisson($e^{u_i} \Xagg_{ij} e^{-v_j})$.
    \item We set $\ptime$ and $\qtime$ to the row sums and column sums of $\Xtime$, respectively.
\end{enumerate}

\subsubsection{Comparing IPF and Poisson regression}
First, we run IPF on $\Xagg$, $\ptime$ and $\qtime$, producing parameters $d^0 \in \mathbb{R}^m$ and $d^1 \in \mathbb{R}^n$.
Then, we fit a Poisson regression model on all observations where $\Xagg_{ij} > 0$, following the construction in Section \ref{thm:Poisson-mle}. 
Poisson regression returns parameters $\{\theta_i\}_{i=1}^m$ corresponding to rows and parameters $\{\theta_j\}_{j=1}^n$ corresponding to columns.
Based on Theorem \ref{thm:model}, we expect $d^0_i = \exp(\theta_i)$, for all $i \in [m]$, and $d^1_j = \exp(\theta_j)$, for all $j \in [n]$, subject to arbitrary scaling between rows and columns (i.e., scaling row factors by $k$ and scaling column factors by $1/k$).
To control for such scaling, we normalize both sets of parameters by dividing by their means.
In Figures \ref{fig:synthetic-ipf-poisson}, we plot the normalized IPF parameters versus the Poisson regression parameters.
We find that they lie perfectly on the $y=x$ line, validating Theorem \ref{thm:model}.
We also plot the 95\% confidence intervals, as provided by Python's statsmodels package for fitting generalized linear models.\footnote{\url{https://www.statsmodels.org/stable/glm.html}.}
Finally, we also plot the true parameter values $e^u$ and $e^{-v}$ in Figure \ref{fig:synthetic-ipf-poisson}, which we can only include in our synthetic setting since we know the true network model's parameters.

\subsubsection{Evaluating IPF over varying sparsity in $\Xagg$}
\label{sec:synthetic-sparsity}
We also evaluate IPF over varying levels of sparsity in the time-aggregated matrix $\Xagg$, since sparsity is common in real-world data, as we find in our later experiments on real-world mobility data (Section \ref{sec:app-safegraph}) and bikeshare data (Section \ref{sec:app-bikeshare}).
For each sparsity rate in $r \in \{0, 0.05, \cdots, 0.9\}$, we run 1000 random trials, where in each trial, we repeat steps 2-5 of our generative process with that sparsity rate in $\Xagg$ (i.e., we fix $e^u$ and $e^{-v}$ but resample all other variables) and run IPF on the newly generated $\Xagg$, $p$, and $q$.
For each trial, we evaluate three metrics:
\begin{itemize}
    \item The number of iterations that IPF takes to converge, following our implementation in Algorithm \ref{alg:ipf} (recall that under the model, IPF will always converge; see Corollary \ref{cor:model-converge}).
    \item The bound on the MLE's expected estimation error, $\mathbb{E}[||(\hat{u}-u, \hat{v}-v)||^2_2]$, from Theorem \ref{thm:mse}, without constants:
    \begin{align*}
        \frac{\sum_{ij} e^{u_i} \Xagg_{ij} e^{-v_j}}{\lambda_{-2}^2(\mathcal{L})}.
    \end{align*}
    \item The observed $\ell_2$ errors between the IPF estimates and the true model parameters:
    \begin{align}
        ||(d^0 - e^u, d^1 - e^{-v}||_2 = \sqrt{\sum_i (d^0_i - e^{u_i})^2 + \sum_j (d^1_j - e^{-v_j})^2}, \label{eqn:l2-def}
    \end{align}
    where we use the mean-normalized version of all parameters.
\end{itemize}

In Figure \ref{fig:ipf-sparsity}, we visualize our results, showing the mean and 95\% CIs (from 2.5th to 97.5th percentiles over random trials) per metric over sparsity rates.
We find that IPF is negatively affected by sparsity in multiple ways: the number of iterations until convergence increases, and both the expected bound and observed estimation error on the network parameters worsens.
However, the rate at which these metrics change differ, with the bound growing most quickly at high levels of sparsity. 
To better understand why, below we analyze the connection between the bound on the MLE's expected error, which we derived in Theorem~\ref{thm:mse}, and the observed $\ell_2$ error that we evaluate in these experiments.

Recall from Theorem~\ref{thm:mse} that in this example, we have an in-expectation bound on the error of the MLE of the form 
\[ \mathbb E \|(\hat u - u, \hat v - v)\|_2^2 = O\left(\frac{\sum_{ij} X_{ij}}{\lambda_{-2}^2(\mathcal L)}\right) \]
where $\hat u = \log d^0$ and $\hat v = - \log d^1$. By the mean-value theorem and the fact that the exponential function is Lipschitz on the interval $[-4,4]$, this implies an analogous in-expectation bound on the exponentiated version which corresponds to the plots:
\[ \mathbb E[\|(d^0 - e^u, d^1 - e^{-v})\|_2^2 \mid X] = O\left(\frac{\sum_{ij} X_{ij}}{\lambda_{-2}^2(\mathcal L)}\right). \]
Since
\[ \mathbb E X_{ij} = 2(1 - r) \]
we have that in expectation (over the randomness of $X$)
\[ \mathbb E \mathcal L = 2(1 - r)\begin{bmatrix} n I & -11^T \\ -11^T & n I \end{bmatrix} \]
and so $\lambda_{-2}(\mathbb E \mathcal L) = 2(1 - r)n$.
Using Weyl's inequality and standard tools from random matrix theory (sketched below, see e.g. \cite{gross2010note} or \cite{anderson2010introduction}) we know that $\lambda_{-2}(\mathcal L) = \lambda_{-2}(\mathbb E \mathcal L) + o(n)$ with high probability for any fixed sparsity rate $r \in [0,1)$. So with high probability over the randomness of $X$,
\begin{equation}  \mathbb E[\|(d^0 - e^u, d^1 - e^{-v})\|_2^2 \mid X] = O\left(\frac{(1 -r)n^2}{(1 - r)^2 n^2}\right) = O\left(\frac{1}{1 - r}\right). 
\end{equation}
By Markov's inequality, this shows that with 99\% probability overall,
\[ \|(d^0 - e^u, d^1 - e^{-v})\|_2 = O\left(\frac{1}{\sqrt{1 - r}}\right) \]
and the right hand side qualitatively matches the shape observed in Figure~\ref{fig:ipf-sparsity} (right).

\emph{Sketch of concentration argument.} We only include a sketch of this argument since it is standard in the random graph/matrix literature and not part of any of the main results of this work.  
By standard concentration estimates (i.e. Bernstein's inequality), the number of entries of $X$ that are deleted from each row will be $rn \pm O(\sqrt{rn \log(n)})$. Equivalently, the number of entries that remain in each row will be $(1 - r)n \pm O(\sqrt{rn \log(n)})$. Next we can compute that
\[ \text{Var} X_{ij} = (1-r)\frac{16}{3} + 4r(1 - r). \]
Hence
\[ \begin{bmatrix} 0 & X \\
X^T & 0 \end{bmatrix} = (1-r) 11^T + W \]
where $W$ is a symmetric mean-zero matrix, where the variance of each coordinate above the diagonal is $O(1 - r)$. So by matrix concentration (see e.g. \cite{gross2010note})
we have that $\|W\|_{OP} = O(\sqrt{1 - r}\sqrt{n \log(n)})$ with high probability. Since all of the error terms in this analysis were $O(\sqrt{n} \log(n))$, they are certainly $o(n)$ as claimed before.

\subsubsection{Testing IPF under model misspecification}
\label{sec:synthetic-misspecification}
One advantage of defining an explicit model (the biproportional Poisson) that IPF corresponds to is that we can now explore related, modified models.
First, we show that our model can be written as a multinomial model, as it is well-known that the distribution of a set of independent Poisson variables can be decomposed as a multinomial distribution over these variables, conditioned on their sum, multiplied by the distribution of their sum (which is a Poisson variable).
Let $\lambda := \sum_{ij; \Xagg_{ij} > 0} e^{u_i} \Xagg_{ij} e^{-v_j}$.
Then, our biproportional Poisson model \eqref{eqn:model} is equivalent to the following model:
\begin{align*}
    N^{(t)} &\sim \mathrm{Poisson}(\lambda)\\
    \{\Xtime_{ij}\}_{\Xagg_{ij} > 0} &\sim \mathrm{Mult}(N^{(t)}, \mathbf{\pi}) \\
    \pi_{ij} &= \frac{e^{u_i} \Xagg_{ij} e^{-v_j}}{\sum_{i'j'; \Xagg_{i'j'}>0} e^{u_{i'}} \Xagg_{i'j'} e^{-v_{j'}}}.
\end{align*}

Other models are not equivalent to the Poisson, but defining them allows us to test IPF's performance under model misspecification.
Below, we define three new models, discuss their MLEs, and conduct experiments where we run IPF on data generated from these models.
In all of these models, we redefine how the time-varying network $\Xtime_{ij}$ is generated, but $\ptime$ and $\qtime$ remain the row and column sums of $\Xtime$, respectively, and $\Xtime_{ij} = 0$ wherever $\Xagg_{ij} = 0$.

\paragraph{Non-Poisson distributions.}
In our biproportional Poisson model \eqref{eqn:model}, we assume each $\Xtime_{ij}$ is drawn from Poisson($\lambda_{ij}$), where $\lambda_{ij} = e^u \Xagg_{ij} e^{-v}$.
Here, we explore other models that keep the same expected value $\lambda_{ij}$, but replace the Poisson distribution with other distributions for count data.
These experiments build on what we proved in Theorem~\ref{thm:uniqueness}: within a family of generalized linear models, IPF uniquely recovers the MLEs of the Poisson model.
Now we test empirically how IPF performs under model misspecification, by trying non-Poisson distributions.

\emph{Exponential.}
First, we try an exponential distribution, since \citet{holfold1980rates} showed that maximum likelihood estimators are equivalent for Poisson and exponential models when the underlying \textit{rate} at which events occur has a log-linear relationship with covariates, and the exponential model is defined by the rate while the Poisson model by the rate multiplied by the population size.
In contrast, we show that replacing the Poisson with an exponential distribution in our model does not yield the same MLE, since the log-linear relationship is with the Poisson parameter instead of a rate.
Since we want the expected value of $\Xtime_{ij}$ to be $e^{\paramrow_i} \Xagg_{ij} e^{-\paramcol_j}$, then the exponential rate must be $\frac{1}{e^{\paramrow_i} \Xagg_{ij} e^{-\paramcol_j}}$.
So, our model is
\begin{align}
    \Xtime_{ij} \sim \mathrm{Exp}(\frac{1}{e^{\paramrow_i} \Xagg_{ij} e^{-\paramcol_j}})\textrm{, for }\Xagg_{ij} > 0. \label{eqn:exp-model}
\end{align}
The likelihood of this model is
\begin{align*}
    \mathcal{L}(\Xtime | \Xagg, u, v) = \prod_{i,j,\Xagg_{ij} > 0}  \frac{1}{e^{\paramrow_i} \Xagg_{ij} e^{-\paramcol_j}} \exp(-\frac{\Xtime_{ij}}{e^{\paramrow_i} \Xagg_{ij} e^{-\paramcol_j}}).
\end{align*}
Maximizing the log-likelihood yields
\begin{align}
    \max_{u,v} \sum_{i,j,\Xagg_{ij} > 0} -\paramrow_i - \log \Xagg_{ij} + \paramcol_j - \frac{\Xtime_{ij}}{e^{\paramrow_i} \Xagg_{ij} e^{-\paramcol_j}}. \label{eqn:exp-ll}
\end{align}
We can see that this log-likelihood \eqref{eqn:exp-ll} clearly does not match our Poisson model's: the signs on $\paramrow_i$ and $\paramcol_j$ are flipped and, unlike in our derivation of the Poisson model's log-likelihood \eqref{eqn:poisson-marginals},  we cannot marginalize out $\Xtime_{ij}$, meaning that the marginals $\ptime$ and $\qtime$ do not suffice as sufficient statistics.
Thus, it is not possible for IPF to derive the maximum likelihood estimates of $u,v$ when the distribution is exponential, instead of Poisson.

\emph{Negative binomial.}
We also try a negative binomial distribution, which is a common alternative to the Poisson model for count data \citep{gardner1995poisson}, since it allows for overdispersed data while the Poisson assumes equidispersion (equal mean and variance).
Our random variable here is $\Xtime_{ij}$, and the model is defined in terms of parameters $s$ (the number of successes) and $\gamma$ (the success probability).
To maintain an expected value of $e^{\paramrow_i} \Xagg_{ij} e^{-\paramcol_j}$, we have
\begin{align*}
    s = \frac{\gamma \cdot e^{\paramrow_i} \Xagg_{ij} e^{-\paramcol_j}}{1-\gamma},
\end{align*}
which yields the following model:
\begin{align}
    \Xtime_{ij} \sim \mathrm{NB}(\frac{\gamma \cdot e^{\paramrow_i} \Xagg_{ij} e^{-\paramcol_j}}{1-\gamma}, \gamma)\textrm{, for }\Xagg_{ij} > 0. \label{eqn:nb-model}
\end{align}
Each random variable has variance
\begin{align*}
    \frac{s(1-\gamma)}{\gamma^2} = \frac{e^{\paramrow_i} \Xagg_{ij} e^{-\paramcol_j}}{\gamma}.
\end{align*}
The negative binomial model becomes identical to our Poisson model in the limit $\gamma \rightarrow 1$, but otherwise clearly will not result in equivalent MLEs due to the change in variance and additional parameter $\gamma$.

Although our statistical theory is developed for MLE under correct specification, we remark that it applies more generally to non-Poisson distributions, as long as the conditional mean of $X_{ij}$ is given by $e^{\paramrow_i} \Xagg_{ij} e^{-\paramcol_j}$. This is because pseudo-Poisson maximum likelihood is known to yield a consistent estimator as long as the logarithm of the conditional mean of the outcome is linear in the covariates. In our setting, this requirement is satisfied as long as the mean of $X_{ij}$ is of the exponential form specified in \eqref{eqn:model}, without requiring the distribution to be Poisson. Since we show in \cref{thm:model} that the problem is equivalent to a Poisson regression, we can conclude that IPF yields a consistent estimate for any model with the specified log-linear means. 

\emph{Experiments.}
We compare IPF's performance on data generated from our biproportional Poisson model \eqref{eqn:model} versus data generated from an exponential distribution \eqref{eqn:exp-model} and from a negative binomial distribution \eqref{eqn:nb-model}.
For the negative binomial, we try $\gamma \in \{0.2, 0.5, 0.8\}$; since the negative binomial is equivalent to the Poisson model when $\gamma \rightarrow 1$, smaller $\gamma$ corresponds to greater deviance from our original model.
To test IPF on data generated from these models, we run 1000 trials, where in each trial, we sample $\Xagg$ from Uniform($0, 1$) (following Step 2 above); sample $\Xtime$ based on the model we are using; set $\ptime$ and $\qtime$ to the row sums and column sums of $\Xtime$, respectively; then run IPF on $\Xagg$, $\ptime$, and $\qtime$.
As in our previous experiment, we evaluate IPF using the $\ell_2$ distance between the true model parameters and IPF estimates \eqref{eqn:l2-def}.
As an additional metric, we also evaluate the cosine similarity between the true network $\Xtime$ and IPF-estimated network $\estXtime = D^0 \Xagg D^1$, where we use cosine similarity as a scale-invariant measure of the similarity:\footnote{We left out this metric when we were evaluating sparsity, since it could be arbitrarily improved with sparsity due to known 0s.}
\begin{align}
    \textrm{sim}(\Xtime, \estXtime) &= \frac{\sum_{ij} \Xtime_{ij} \estXtime_{ij}}{||\Xtime||_2 \cdot ||\estXtime||_2}.\label{eqn:cosine-sim}
\end{align}
As shown in Table \ref{tab:ipf-non-poisson}, we find that IPF performs best when the model is correctly specified (i.e., the data is generated from our biproportional Poisson model), compared to data generated from exponential or negative binomial models.
As expected, IPF's performance on data from the negative binomial model also worsens as $\gamma$ becomes smaller.
However, IPF is still reasonably effective under model misspecification: the cosine similarities between $\Xtime$ and $\estXtime$ remain above $0.8$ for the exponential model and negative binomial models, when $\gamma \geq 0.5$.
\begin{table}[]
    \centering
    \begin{tabular}{c|c|c}
        Model & $\ell_2$ error on $u$ and $v$ & Cosine sim. between $\Xtime$ and $\estXtime$\\
        \hline
        Biproportional Poisson \eqref{eqn:model} & 0.995 (0.885-1.114) & 0.911 (0.907-0.914)\\
        Negative binomial \eqref{eqn:nb-model}, $\gamma = 0.8$  & 1.117 (0.999-1.260) & 0.892 (0.887-0.897) \\
        Negative binomial \eqref{eqn:nb-model}, $\gamma = 0.5$ & 1.410 (1.266-1.572) & 0.843 (0.836-0.850) \\
        Exponential \eqref{eqn:exp-model} & 1.734 (1.519-1.976) & 0.855 (0.843-0.866) \\
        Negative binomial \eqref{eqn:nb-model}, $\gamma = 0.2$ & 2.226 (1.957-2.479) & 0.707 (0.695-0.720) \\
        \hline
    \end{tabular}
    \caption{Evaluating IPF's performance on synthetic data generated from our biproportional Poisson model vs. models with non-Poisson distributions. The models are approximately ordered in terms of best-to-worst performance.}
    \label{tab:ipf-non-poisson}
\end{table}

As we did in Section \ref{sec:synthetic-sparsity}, we also test IPF over varying levels of sparsity $r$ in $\Xagg$, but this time with $\Xtime$ generated from exponential distribution \eqref{eqn:exp-model}.
This experiment allows us to compare the shape of $\ell_2$ error over varying levels of sparsity, when the data is correctly specified (generated from Poisson) vs. incorrectly specified (generated from exponential).\footnote{We are not able to test the negative binomial model with sparsity greater than 0, since the number of successes $s$ must be strictly positive, so the model \eqref{eqn:nb-model} ill-defined for $\Xagg_{ij} = 0$.}
We visualize the results in Figure \ref{fig:sparsity-misspec}: we find that the shape of the $\ell_2$ curves are highly similar when the data is drawn from either distribution, and matches the expected shape of $1 / \sqrt{1-r}$ (see derivation in Section \ref{sec:synthetic-sparsity}), showing the usefulness of our derived error bounds even under model misspecification.

\begin{figure}
    \centering
    \includegraphics[width=0.6\linewidth]{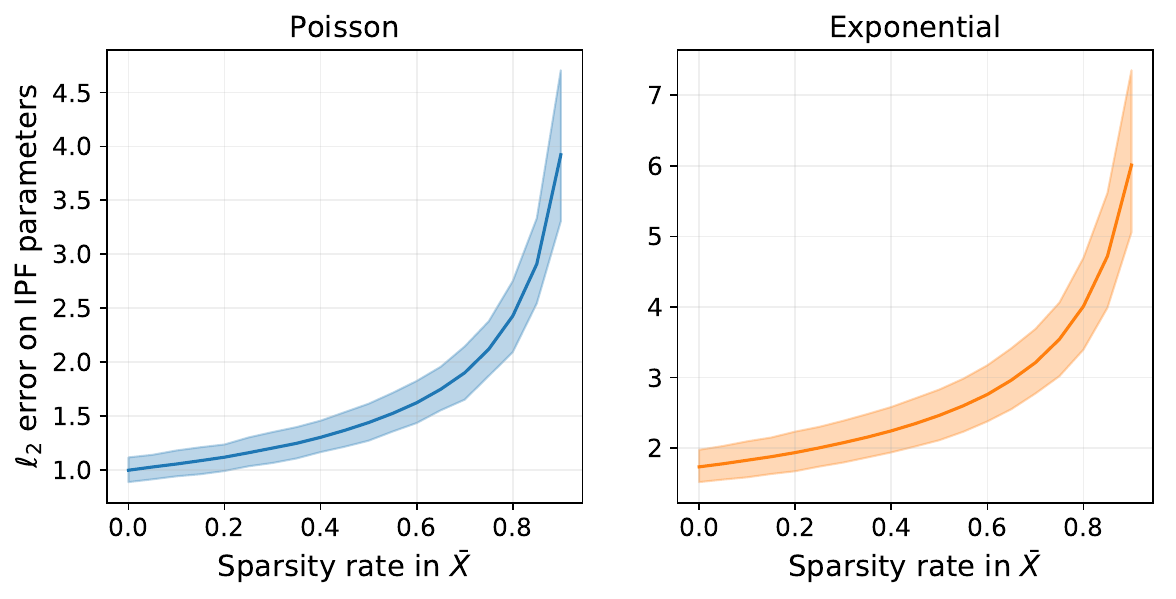}
    \caption{Comparing sparsity rate in $\Xagg$ to observed $\ell_2$ error of IPF estimates, for data drawn from Poisson (left) vs. exponential (right) distributions. The subfigure on the left, where the model is correctly specified, exactly matches Figure \ref{fig:ipf-sparsity}, right. Lines represent mean and shaded region represents 95\% CIs over 1000 trials.}
    \label{fig:sparsity-misspec}
\end{figure}

\paragraph{Models with interaction terms.}
Our original model assumes that the expected hourly values for $\Xtime_{ij}$ are a scaling of $\Xagg_{ij}$ by a row factor $e^{u_i}$ and a column factor $e^{-v_j}$.
However, there could also be interactions between rows and columns that influence the expected hourly values.
To account for interaction, we extend our model as 
\begin{align}
    \Xtime_{ij} \sim \mathrm{Poisson}(e^{u_i} \Xagg_{ij} e^{-v_j} \cdot d_{ij}^{\alpha} e^{-d_{ij} \beta})\textrm{, for }\Xagg_{ij} > 0, \label{eqn:interaction}
\end{align}
where $\alpha$ and $\beta$ are a new model parameters and $d_{ij} > 0$ represents an observed interaction feature, such as the geographical distance if the rows and columns correspond to physical places.
This new model is inspired by the doubly constrained gravity model from \citet{navick1994distance}, who estimate the number of trips between bus stops $i$ and $j$ as $d_{ij}^{\alpha} e^{-d_{ij} \beta} A_i B_j$, with learned row and column factors $A_i$ and $B_j$ which correspond to $e^{u_i}$ and $e^{-v_j}$.
We discuss their gravity model in detail in Section \ref{sec:eval-ground-truth}, where we use it as a baseline against which to compare IPF.
We can view this interaction model \eqref{eqn:interaction} as interpolating between using the time-aggregated network as the initial matrix, as a data-driven perspective, and using a gravity model to form the initial matrix, as a model-based perspective; meanwhile, continuing to learn row- and column-specific factors to match the observed time-varying marginals.
Maximizing the log-likelihood of this model yields
\begin{align*}
    \max_{u,v,\alpha,\beta} &\sum_{i,j; \Xagg_{ij}>0} \Xtime_{ij} (u_i + \log \Xagg_{ij} - v_j + \alpha \log d_{ij} - d_{ij} \beta) - e^{u_i} \Xagg_{ij} e^{-v_j} \cdot d_{ij}^{\alpha} e^{-d_{ij} \beta} \\
    &\propto \max_{u,v,\alpha,\beta} \sum_i u_i \ptime_i - \sum_j \vtime_j q_j + \alpha \sum_{ij} \Xtime_{ij} \log d_{ij} - \beta \sum_{ij} \Xtime_{ij} d_{ij} - \sum_{ij} e^{u_i} \Xagg_{ij} e^{-v_j} \cdot d_{ij}^{\alpha} e^{-d_{ij} \beta}.
\end{align*}
This new log-likelihood cannot be simplified to the original log-likelihood \eqref{eqn:poisson-marginals} unless $\alpha = \beta = 0$.
Furthermore, unlike in the biproportional Poisson model, the marginals of $\Xtime$ ($\ptime$ and $\qtime$) do not form sufficient statistics here, meaning that it cannot be used in our dynamic network inference setting where only $\ptime$, $\qtime$, and $\Xagg$ are observed (but not $\Xtime$).
However, it is still instructive to define this model, so that we can test IPF on data generated from this model.

\emph{Experiments.}
To generate data from the interaction model \eqref{eqn:interaction}, first we sample 2-dimensional positions per row $i \in [m]$ and column $j \in [n]$ from Uniform($0, 1$), then we set the distance $d_{ij}$ between row $i$ and column $j$ to be the Euclidean distance between their positions.
Then, for each value of $\alpha$ and $\beta$, we run 1000 trials, where (as before) in each trial we sample $\Xagg$ from Uniform($0, 1$); sample $\Xtime$ from the interaction model with the current values of $\alpha$, $\beta$, and $d_{ij}$; set $\ptime$ and $\qtime$ to the row sums and column sums of $\Xtime$, respectively; then run IPF on $\Xagg$, $\ptime$, and $\qtime$.
As before, we evaluate the $\ell_2$ distance between the true model parameters and IPF estimates \eqref{eqn:l2-def} as well as the cosine similarity between the true network $\Xtime$ and IPF-estimated network $\estXtime$ \eqref{eqn:cosine-sim}.
\begin{figure}[t]
     \centering
     \begin{subfigure}
         \centering
         \includegraphics[width=0.6\textwidth]{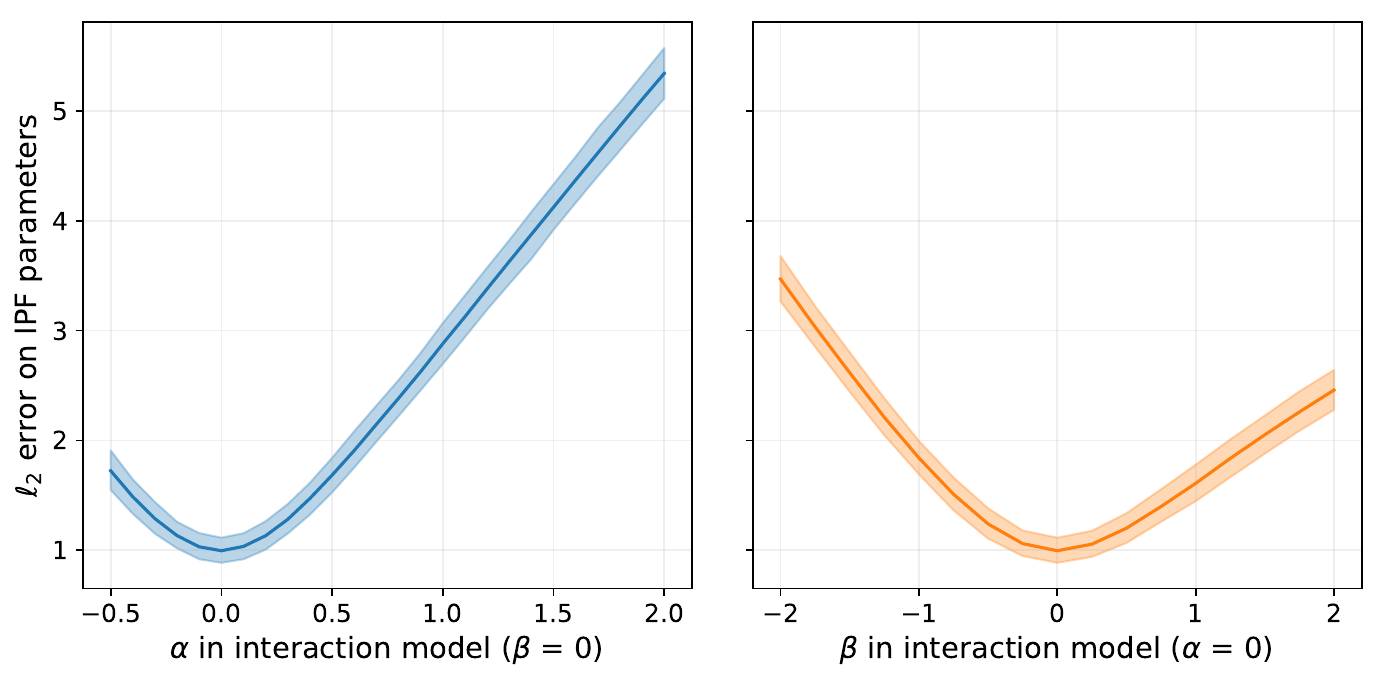}
         \caption{IPF's $\ell_2$ error on network parameters $u$ and $v$ \eqref{eqn:l2-def}.}
         \label{fig:interaction-model-l2}
     \end{subfigure}
     \begin{subfigure}
         \centering
         \includegraphics[width=0.6\textwidth]{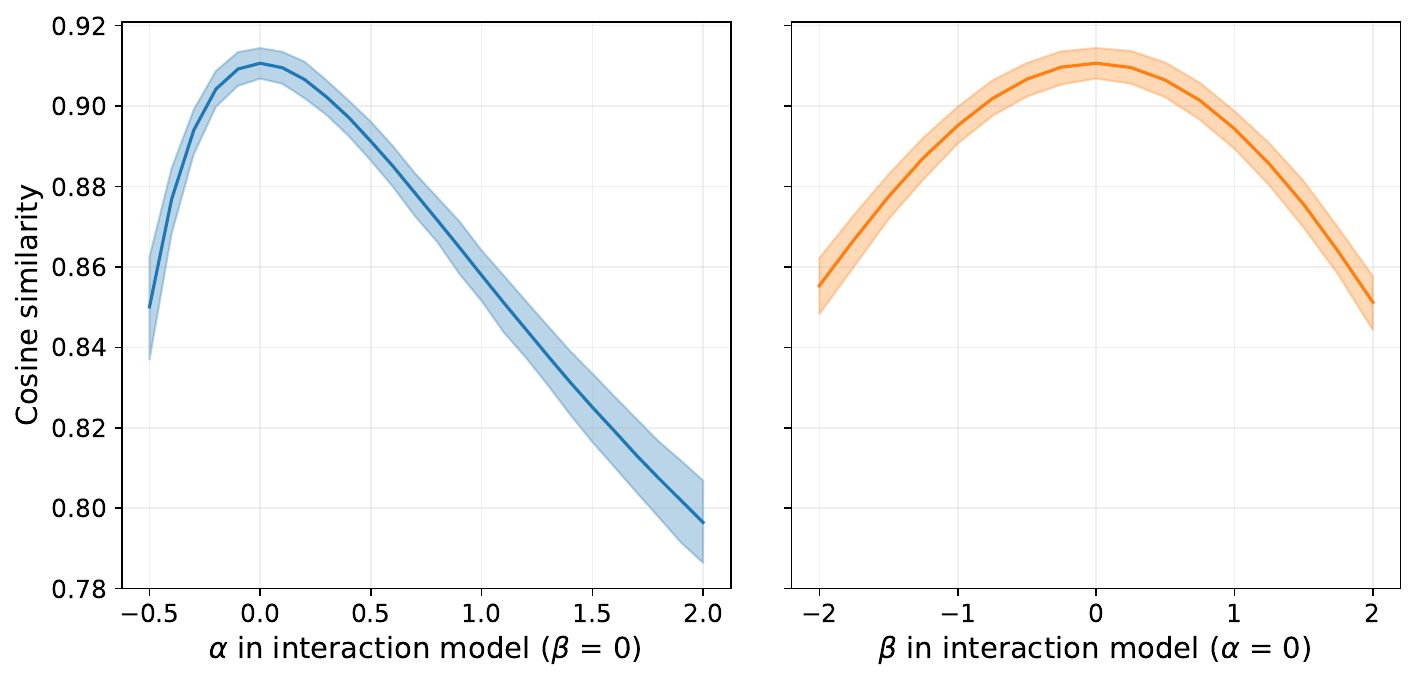}
         \caption{Cosine similarity between true network $\Xtime$ and IPF-estimated network $\estXtime$ \eqref{eqn:cosine-sim}.}
         \label{fig:interaction-model-cosine}
     \end{subfigure}
        \caption{Evaluating IPF performance on data generated from interaction model \eqref{eqn:interaction}, over different values of $\alpha$ and $\beta$.}
        \label{fig:interaction-model}
\end{figure}
To test the impact of $\alpha$ alone, we fix $\beta$ to 0 and try $\alpha$ from $-0.5$ to $2$, which is symmetric in its impact on $d_{ij}$ since we have $d_{ij}^\alpha$ in our model.
As expected, we find that performance worsens as $\alpha$ deviates further from $0$, but we find that over this range of $\alpha$ that cosine similarity between the true and estimated network remains high, at $0.8$ or higher, while the $\ell_2$ distance on parameters $u$ and $v$ is more dramatically worsened (Figure \ref{fig:interaction-model}, left).
To test the impact of $\beta$ alone, we fix $\alpha$ to 0 and try $\beta$ from $-2$ to $2$, which is also symmetric in its impact on $d_{ij}$ since we have $e^{-d_{ij} \beta}$ in our model.
We also find that performance worsens as $\beta$ deviates from $0$, but the cosine similarity remains high (above $0.85$) and the $\ell_2$ is less affected by $\beta$ than $\alpha$ (Figure \ref{fig:interaction-model}, right).
So, at least under these values of $\alpha$, $\beta$, and $d_{ij}$, we find that IPF is relatively robust to model modifications and can still effectively infer the true network. 

\subsection{Experiments with real-world mobility data} 
\label{sec:app-safegraph}
We also conduct experiments with real-world mobility data, using the same aggregated location data from SafeGraph as \citet{chang2021nature,chang2021kdd}.
Like \citet{chang2021nature,chang2021kdd}, our goal is to use the data from SafeGraph to infer hourly mobility networks that encode visits from home census block groups (CBGs), which are neighborhoods of around 600-3000 people, to individual points-of-interest (POIs), which are public locations such as restaurants, grocery stores, or gas stations.
SafeGraph data provides a realistic example of where IPF would be necessary to infer a dynamic network from its 3D marginals, since SafeGraph does not provide the hourly mobility network; instead, they provide hourly total visits to POIs and from CBGs as well as an estimated, time-aggregated mobility network.
Furthermore, SafeGraph's mobility data is noisy for several reasons, which creates the possibility of incorrect zeros in the time-aggregated matrix, and thus a need for our convergence algorithm in \cref{sec:convergence}.
We discuss these reasons below.

\paragraph{Noise and missingness in SafeGraph data.}
First, unlike in the case of bikeshare data (Section \ref{sec:app-bikeshare}), where bikes are perfectly observed checking into and out of stations, SafeGraph needs to \textit{infer} assignments to rows (i.e., POIs) and columns (i.e., CBGs). 
That is, SafeGraph's raw data takes the form of individual devices' GPS signals, and SafeGraph needs to perform both \textit{visit attribution} to determine that a device is visiting a particular POI \citep{safegraph-visits} and \textit{home attribution} to determine that a device belongs to a person who lives in a given CBG \citep{safegraph-social-distancing,huang2021athome}.
Mistakes can be made at different stages of this process, such as incorrect visit attribution due to GPS drift or incorrect home attribution due to nighttime activity in a CBG that is not the person's home (e.g., if they have night shifts for work).
Second, since SafeGraph relies on data from smartphones, they are missing individuals who either do not carry cell phones or opt out of data collection \citep{safegraph-bias}.
For example, \citet{coston2021facct} showed that SafeGraph disproportionately misses older and non-white populations in their data. 
Third, to preserve privacy, SafeGraph documents that they apply ``differential privacy techniques'' to their reported time-aggregated networks \citep{safegraph-patterns}.
Specifically, for a given POI, when reporting time-aggregated (weekly or monthly) counts of visits from a home CBG, they add Laplacian noise to the count, drop CBGs with fewer than two devices visiting that POI, and report all visit counts from 2-4 as 4.
All of these reasons lead to imperfectly observed networks and marginals, motivating the need for our convergence algorithm.

\subsubsection{Constructing IPF inputs from mobility data} \label{sec:mobility-construct}
Using this data, our goal is to infer the mobility network at hour $t$ from $n$ CBGs to $m$ POIs.
We study the Richmond, Virginia metropolitan statistical area (MSA), as one of the regions studied by \citet{chang2021kdd}.
Using their inclusion criteria, we include 9917 POIs in the MSA and 1098 CBGs that frequently visit these POIs.
To capture temporal variation, we compare two days---March 2 and April 6, 2020, two Mondays just before and after the onset of the COVID-19 pandemic in the US---and infer hourly networks over their 48 hours.
To infer these networks using IPF, we construct a time-aggregated network, $\Xagg$; hourly total visitors to POIs, $\ptime$; and hourly total visitors from CBGs, $\qtime$.
We construct these quantities from SafeGraph data in the same way as the authors did in \citet{chang2021nature,chang2021kdd}.
We summarize this procedure below, highlighting a few important facts, and refer the reader to the original text for details. 

\paragraph{Constructing time-aggregated network $\Xagg$.}
SafeGraph provides summaries of the home CBGs of each POI's visitors, per month (before March 2020) or week (after March 2020).
To account for non-uniform sampling from different CBGs, we weight the number of SafeGraph visitors from each CBG by the ratio of the CBG population (from US Census) and the number of SafeGraph devices with homes in that CBG.
Following the original text, let $\hat{W}(r)$ represent the reweighted matrix for period $r$ (we use $r$ instead of $t$ to denote time periods longer than an hour).
Since these visit matrices are sparse, we aggregate over $R$ time periods, from January to October 2020:
\begin{align}
    \bar{W} = \frac{1}{R}\sum_{r=1}^R \hat{W}(r), \quad
    X_{ij} = \frac{\bar{W}_{ij}}{\sum_k W_{kj}}.
\end{align}
So, $X_{ij}$ represents the time-aggregated \textit{proportion} of visits to POI $i$ that come from CBG $j$.
Note that SafeGraph's visit matrices include all possible home CBGs, but when we construct $X$, we only include the $n$ CBGs for the metropolitan statistical area.
So, the rows of $X$ typically do not sum to 1 and are usually around 0.9-0.97.

\paragraph{Constructing visitors to POIs, $\ptime$.}
SafeGraph provides the hourly number of \textit{visits}, not visitors, so first we apply corrections to the SafeGraph counts based on the POI's median dwell time to estimate the hourly number of visitors (see Supplementary Information from \citet{chang2021nature}).
To account for SafeGraph undersampling, we also multiply each POI's visit count by a uniform correction factor which is the ratio of the US population to the total number of SafeGraph devices; this factor is around 7.
Finally, since not all of the POI's visits are captured by the $n$ CBGs in $X$, we multiply the POI's visits by its row sum in $X$, i.e., its total proportion of visits kept. 
In Figure \ref{fig:richmond} (left), we visualize the proportion of POI marginals that are nonzero, over 24 hours in the day on March 2, 2020 and April 6, 2020.
We see that only a small proportion of POIs have nonzero marginals at nighttime, e.g., less than 10\% from 12-5am.
For both days, the proportion peaks from around 6-10pm, likely when people are visiting POIs after work.
We also see considerably more sparsity in POI marginals on April 6, compared to March 2, which reflects the onset of the COVID-19 pandemic in the US.

\begin{figure}[t]
     \centering
     \begin{subfigure}
         \centering
         \includegraphics[width=0.7\textwidth]{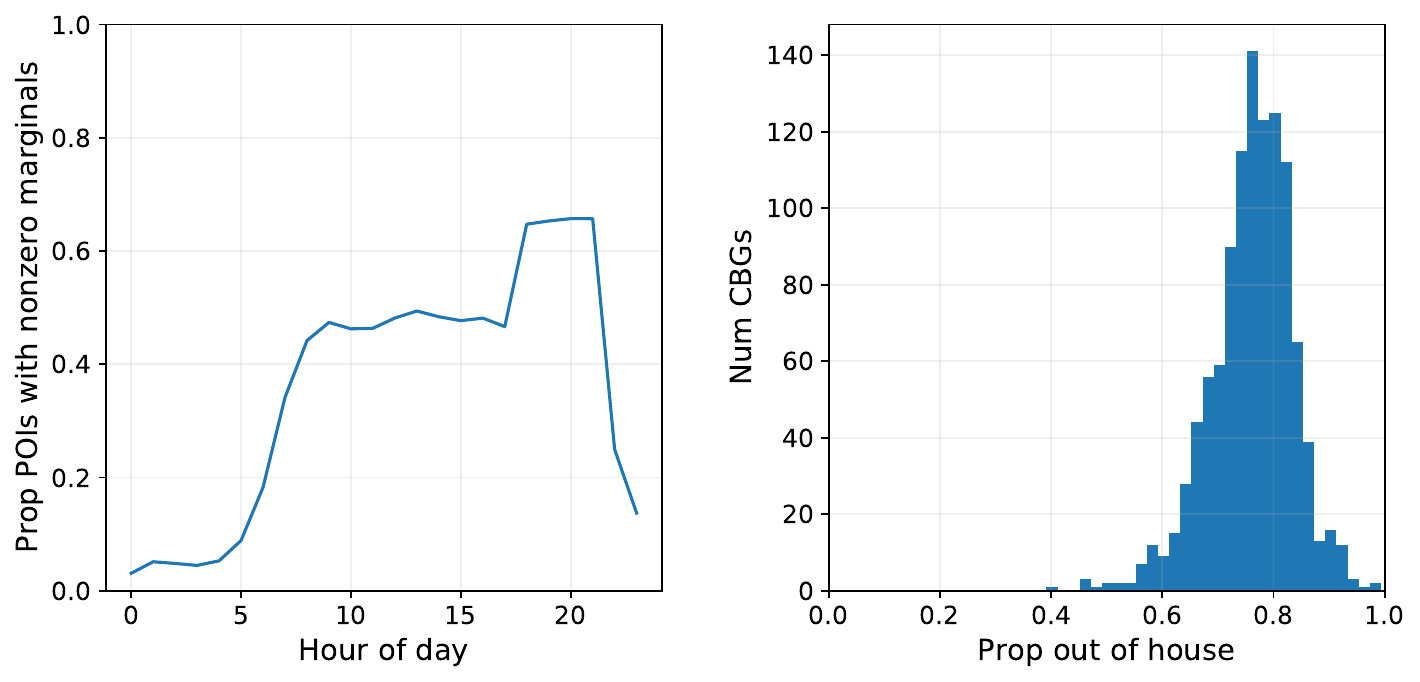}
         \caption{Marginals on Monday, March 2, 2020.}
         \label{fig:richmond-march2}
     \end{subfigure}
     \begin{subfigure}
         \centering
         \includegraphics[width=0.7\textwidth]{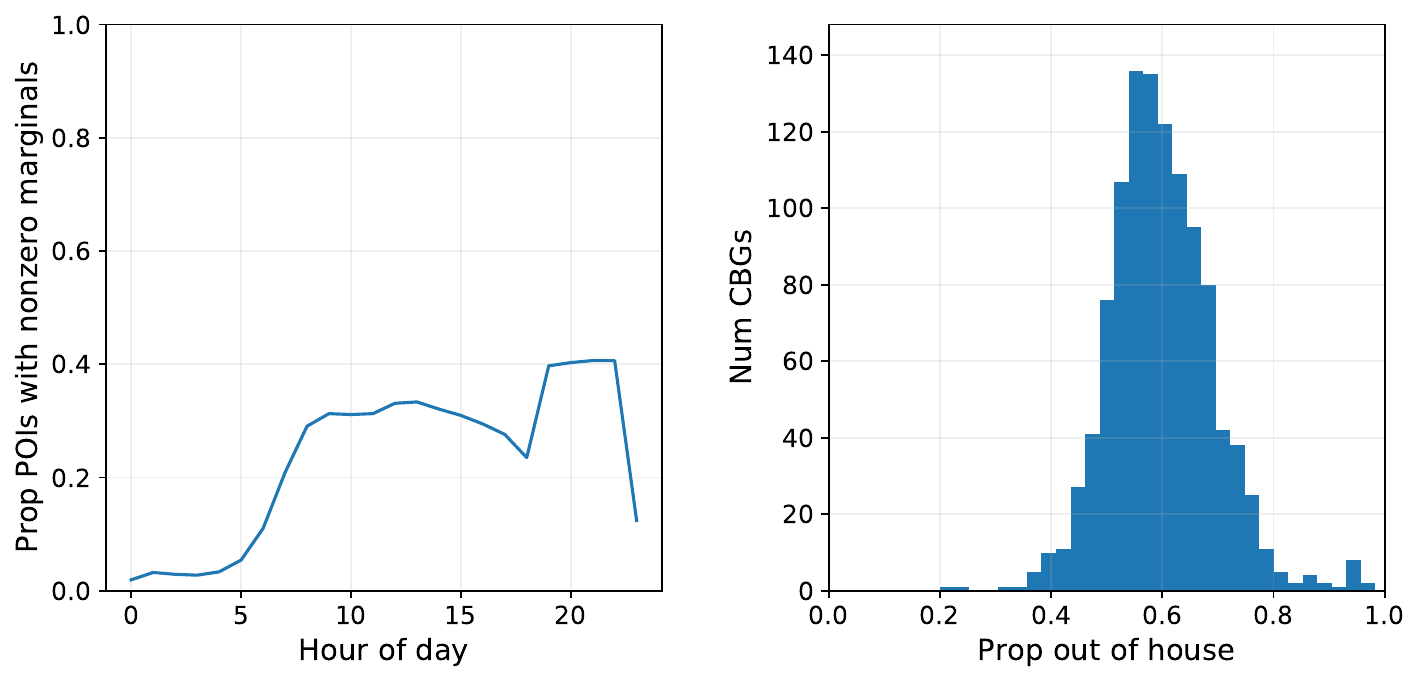}
         \caption{Marginals on Monday, April 6, 2020.}
         \label{fig:richmond-apr6}
     \end{subfigure}
        \caption{POI and CBG marginals from mobility data for Richmond, Virginia MSA.}
        \label{fig:richmond}
\end{figure}

\paragraph{Constructing visitors from CBGs, $\qtime$.}
Using SafeGraph Social Distancing Metrics \citep{safegraph-social-distancing}, we estimate $\hat{h}_j^{(t)}$, the fraction of each CBG that has left their home.
Then, we estimate the number of people who left their home by multiplying these fractions by the CBG population $N_j$ (from US Census), then we scale these estimates so that $\qtime$ sums to $\sum_i \ptime_i$.
We do this since IPF requires the sums of the row and column marginals to match and because the number of people who are not at home may not exactly match the number of people who are visiting POIs.
So, we have
\begin{align}
    \qtime_j = \hat{h}_j^{(t)} N_j \cdot \frac{\sum_i \ptime_i}{\sum_k \hat{h}_k^{(t)} N_k}.
\end{align}
In Figure \ref{fig:richmond} (right), we also visualize the proportions out of the house per CBG.
We only have these quantities at a daily granularity from SafeGraph, so we plot a histogram over CBGs instead of an hourly measure.
We can see that, in this setting, CBG marginals are always positive, even after the pandemic onset.

\subsubsection{Convergence experiments on mobility data} 
\label{sec:mobility-ipf}
\begin{figure}[t]
    \centering    
    \includegraphics[width=0.7\textwidth]{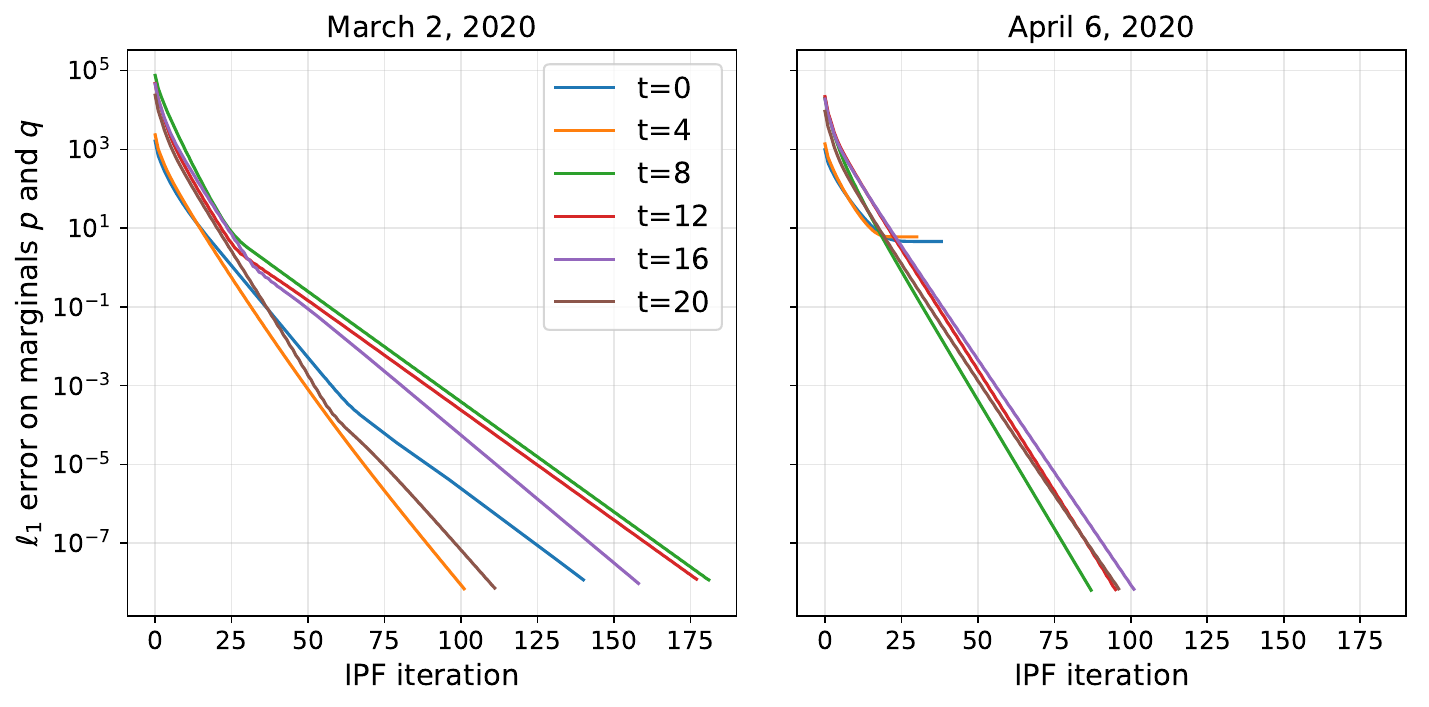}
    \caption{$\ell_1$ error on marginals $p$ and $q$ over IPF iterations on mobility data. We show convergence results from hours $t \in \{0, 4, 8, 12, 16, 20\}$ on March 2, 2020 (left) and April 6, 2020 (right).}
    \label{fig:ipf-l1-marginals}
\end{figure}

\paragraph{Non-convergence.}
We run (modified) IPF (Algorithm \ref{alg:ipf}) for all hours on March 2 and April 6, 2020, and we find that IPF converges for 45 out of the 48 hours. 
However, it gets stuck in oscillation for 3 hours during nighttime, when POI marginals are particularly sparse: 2AM on March 2 and 12AM and 4AM on April 6.
In Figure \ref{fig:ipf-l1-marginals}, we plot the $\ell_1$ error between the marginals of $\ipfestmat$ and the target marginals, which is known to decrease monotonically with each iteration \citep{pukelsheim2014}.
This $\ell_1$ error is computed as:
\begin{align}
    \ell_1(k) &= |\ipfestmat(k) \cdot \mathbf{1} - \ipfrow| + |\ipfestmat(k)^T \cdot \mathbf{1} - \ipfcol|.
\end{align}
When IPF converges, we find that the error decreases exponentially, although the convergence sometimes demonstrates a one-time ``bend'' where the exponential rate changes (e.g., for $t=8$ on March 2).
This intriguing observation is worthy of further study given that IPF only converges exponentially in certain settings \citep{pukelsheim2009ipf}.
When IPF does not converge, as in the case of $t=0$ and $t=4$ on April 6, its $\ell_1$ error gets stuck  at a fixed value, since that fixed value gets passed back and forth from error on the row marginals to error on the column marginals.

\paragraph{Evaluating our convergence algorithm.}
Now, we test our convergence algorithm, \texttt{ConvIPF}, on these three hours where IPF did not converge.
Recall that our convergence algorithm enables IPF to converge by making minimal modifications (edge additions) to the time-aggregated network $\Xagg$.
We test our two definitions of $\texttt{MODIFY-X}$ (Section \ref{sec:app-modify-x}), where we aim to either minimize the number of new edges added or minimize the change in the largest eigenvalue $\lambda_1$.
For the first objective, we showed that any row in the blocking set $S$ can be used, so, as two tie-breaking rules, we experiment with taking the row with the largest $p_i$ and the smallest $p_i$.
We compare our solutions to the typical solution for IPF non-convergence, which is to replace all zeros with a very small value $\epsilon$ \citep{lomax2015geographer,lovelace2015microsim}.

\begin{table}[]
    \centering
    \begin{tabular}{p{3cm}|p{2.5cm}|p{3cm}|p{3cm}|p{3cm}}
          & Replace all zeros & \texttt{ConvIPF}, min \# edges, largest $p_i$ & \texttt{ConvIPF}, min \# edges, smallest $p_i$ & \texttt{ConvIPF}, min change in $\lambda_1$  \\
         \hline
         \multicolumn{5}{c}{March 2, 2023, 2AM} \\
         \hline
         \# new edges & $10012193$ & $\mathbf{2}$ & $16$ & $5$ \\
         Change in $\lambda_1$ & $30.04$ & $1.85 \cdot 10^{-8}$ & $2.54 \cdot 10^{-7}$ & $\mathbf{5.60 \cdot 10^{-9}}$ \\
         \# \texttt{ConvIPF} iters. & -- & $\mathbf{2}$ & $16$ & $5$ \\
         \hline
         \multicolumn{5}{c}{April 6, 2023, 12AM} \\
         \hline
         \# new edges & $10012193$ & $\mathbf{2}$ & $3$ & $3$ \\
         Change in $\lambda_1$ & $30.04$ & $2.14 \cdot 10^{-7}$ & $1.02 \cdot 10^{-8}$ & $\mathbf{9.98 \cdot 10^{-9}}$ \\
         \# \texttt{ConvIPF} iters. & -- & $\mathbf{1}$ & $2$ & $2$ \\
         \hline
         \multicolumn{5}{c}{April 6, 2023, 4AM} \\
         \hline
         \# new edges & $10012193$ & $\mathbf{1}$ & $9$ & $5$ \\
         Change in $\lambda_1$ & $30.04$ & $1.63 \cdot 10^{-8}$ & $8.52 \cdot 10^{-8}$ & $\mathbf{5.72 \cdot 10^{-9}}$ \\
         \# \texttt{ConvIPF} iters. & -- & $\mathbf{1}$ & $9$ & $5$ \\
    \end{tabular}
    \caption{Comparing different methods for achieving IPF convergence over three hours where the mobility data did not converge. The bolded entry in each row shows the optimal (in this case, lowest) number.}
    \label{tab:convergence-results}
\end{table}
We present results in Table \ref{tab:convergence-results}.
First, we find that all three versions of \texttt{ConvIPF} massively outperform the typical solution, for both metrics of minimizing the number of new edges and minimizing change in $\lambda_1$.
Second, we find that in practice, \texttt{ConvIPF} only runs for a handful of iterations (typically, 5 or below, and at most, 16), so it terminates quickly and does not need to unblock that many blocking sets.
However, each blocking set can be complex: \texttt{BLOCKING-SET} often finds blocking sets containing hundreds of rows, demonstrating the need for our efficient algorithm to find a blocking set, compared to the simple solution of iterating through row subsets.
Third, we find that the choice of objective in \texttt{ConvIPF} makes a difference:  choosing to minimize the number of edges (with largest $p_i$) always results in the fewest number of edges added and choosing to minimize the change in $\lambda_1$ always results in the smallest change in $\lambda_1$, by an order of magnitude compared to the other \texttt{ConvIPF} variants.
This latter result also shows that, even with our approximation of $\lambda_1$ (Appendix \ref{sec:app-modify-x}), the algorithm is still effective.
Finally, we find that within the objective of minimizing number of edges added, taking the row with the largest $p_i$ is consistently better than smallest $p_i$ for minimizing both the overall number of edges added and the number of \texttt{ConvIPF} iterations.
We hypothesize that this is because, within a blocking row set $S$, the row with the largest $p_i$ is likely contributing more to the condition violation, $\sum_{i \in S} p_i > \sum_{j \in N_X(S)} q_j$, so modifying its connections improves the likelihood that there will not be a subset of $S$ that still needs to be unblocked.
This theory reveals an open question for future work, which is analyzing the optimal order of unblocking row sets in order to minimize the number of \texttt{ConvIPF} iterations.


\subsection{Experiments with bikeshare data}
\label{sec:app-bikeshare}
In this section, we consider the case where we have \textit{ground-truth} hourly networks.
We use data from New York City's CitiBike system\footnote{We downloaded \texttt{202309-citibike-tripdata.csv} from \url{https://citibikenyc.com/system-data}.}, which contains individual trips, including each trip's start time, end time, start station, and end station.
For these experiments, we use trips from September 1 to September 30, 2023, which contains around 3.5 million trips.
If we take the union of all start stations and end stations as the set of stations, there are $m = n = 2036$ stations in total.

\begin{figure}[t]
     \centering
     \begin{subfigure}
         \centering
        \includegraphics[width=0.7\linewidth]{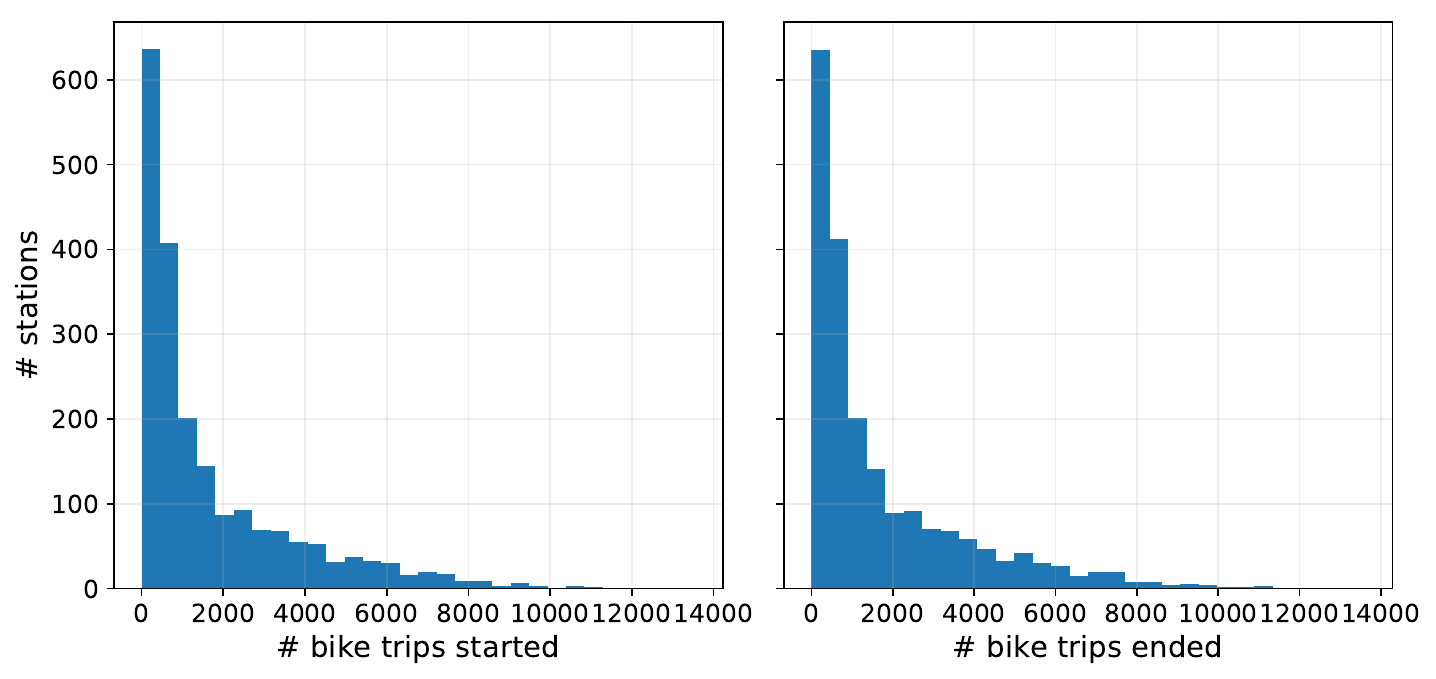}
        \caption{Distribution of bike trips over start stations (left) and end stations (right).}
        \label{fig:station-dist}
     \end{subfigure}
     \begin{subfigure}
         \centering
    \includegraphics[width=0.7\linewidth]{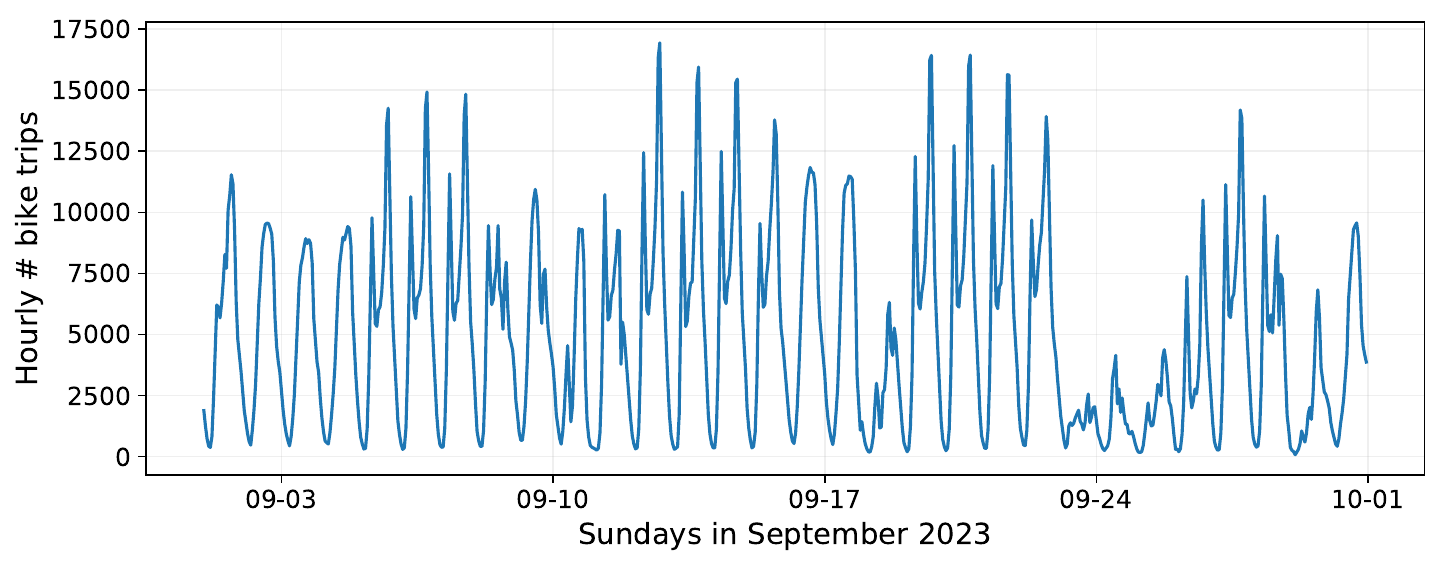}
    \caption{Distribution of bike trips over hours.}
    \label{fig:trips-over-hours}
     \end{subfigure}
        \caption{Bike trips from NYC Citibike data, from September 1 12AM to September 30 11PM, 2023.}
\end{figure}
The distribution of trips over stations is highly skewed (Figure \ref{fig:station-dist}).
To aggregate the individual trips into hourly network, we need to assign each trip to an hour.
In 79\% of trips, the trip starts and ends in the same hour, so the assignment is unambiguous.
For the remaining 21\% of trips, we assign the trip to the hour of its midway point.
So, in our hourly network $\Xtime$, each $\Xtime_{ij}$ represents the total number of trips from station $i$ to station $j$ with its midpoint falling in hour $t$.
In Figure \ref{fig:trips-over-hours}, we plot the distribution of trips over hours.
We can see clear temporal patterns, e.g., fewer trips at night, two peaks on weekdays for getting to and getting from work.
As usual, $\ptime$ and $\qtime$ are the row sums and column sums, respectively, of $\Xtime$.
For the time-aggregated network $\Xagg$, we explore different temporal aggregations (month, week, and day), so that we can experiment with IPF performance as the input network becomes increasingly time-aggregated.

\subsubsection{Evaluating IPF performance on ground-truth networks}
\label{sec:eval-ground-truth}
In these experiments, we evaluate IPF's performance at estimating the ground-truth hourly bikeshare networks, over 24 hours on September 1, 2023.
We try three versions of IPF and seven baseline methods for estimating the hourly network:
\begin{enumerate}
    \item IPF on the month-aggregated network (September 1 to September 30, 2023) and hourly marginals;
    \item IPF on the week-aggregated network (September 1 to September 7, 2023) and hourly marginals;
    \item IPF on the day-aggregated network (September 1, 2023) and hourly marginals;
    \item IPF on the distance-based matrix and hourly marginals, equivalent to the gravity model (see details below);
    \item An ablation that removes $\Xagg$ and distributes $\sum_{ij} \Xtime$ proportional to $\ptime_i \qtime_j$;
    \item An ablation that removes $\qtime$ and distributes $\ptime_i$ within each row proportional to $\Xagg_{ij} / \sum_j \Xagg_{ij}$;
    \item An ablation that removes $\ptime$ and distributes $\qtime_j$ within each column proportional to $\Xagg_{ij} / \sum_i \Xagg_{ij}$.
    \item A baseline where we use the month-aggregated network, scaled to match $\sum_{ij} \Xtime$.
    \item A baseline where we use the week-aggregated network, scaled to match $\sum_{ij} \Xtime$.
    \item A baseline where we use the day-aggregated network, scaled to match $\sum_{ij} \Xtime$.
\end{enumerate}
Each method produces an estimate, $\estXtime$, of the true hourly network, $\Xtime$
We use cosine similarity as a scale-invariant measure of the similarity between the two networks \eqref{eqn:cosine-sim}.

\paragraph{Details on gravity model.}
The gravity model \citep{zipf1946intercity} is a classic model of human mobility, inspired by Newton's law of gravity, that assumes that the amount of travel between two regions is proportional to their population sizes divided by their geographical distance.
\citet{navick1994distance} show that a ``doubly constrained'' gravity model of origin-destination matrices---where the total number of departures and arrivals are observed and must be matched---is equivalent to running IPF, where the initial matrix is replaced with a distance matrix with $D_{ij} = f(d_{ij})$, where $f(\cdot)$ is a function representing how trip propensities vary with distance.
This convenient connection between IPF and the gravity model enables us to use the latter as a baseline, by running IPF on $D_{ij}$ and hourly marginals $\ptime$ and $\qtime$.

To compute distances between bike stations, we first estimate the latitude/longitude of each station.
CitiBike reports the start latitude/longitude and end latitude/longitude of each trip, so for each station, we take the mean of the start latitude/longitude where the station is start station and end latitude/longitude where the station is the end station.
Then we compute distance between all pairs of stations using Euclidean distance, which is a reasonable proxy for short distances (all trips are within New York City).
To prevent $d_{ij} = 0$ for trips from and to the same station (which would force $f(d_{ij})$ to be 0), we replace those distances with $\epsilon = \min(\{d_{ij} | d_{ij} > 0\}) / 2$; we consider these trips still meaningful, since a rider may want to pick up and drop off their bike in the same place (e.g., if they live near that station).

Then, following \citet{navick1994distance}, we define $f(\cdot)$ as
\begin{align}
    f(d_{ij}) = d_{ij}^\alpha \exp(-d_{ij}\beta), \label{eqn:distance-func}
\end{align}
where $\alpha$ and $\beta$ are learnable parameters.
First, we estimate the empirical distribution by dividing station pairs into distance intervals (each one of size 0.001) and calculating the mean number of bike trips in the month-aggregated network $\Xagg$ for that interval.
Then, we fit \eqref{eqn:distance-func} to this empirical distribution, obtaining $\hat{\alpha} = -0.5863$ and $\hat{\beta} = 38.52$. The much larger magnitude of $\hat \beta$ is interesting, suggesting the distribution is much closer to an exponential distribution than a power distribution.
\begin{figure}
    \centering
    \includegraphics[width=0.7\linewidth]{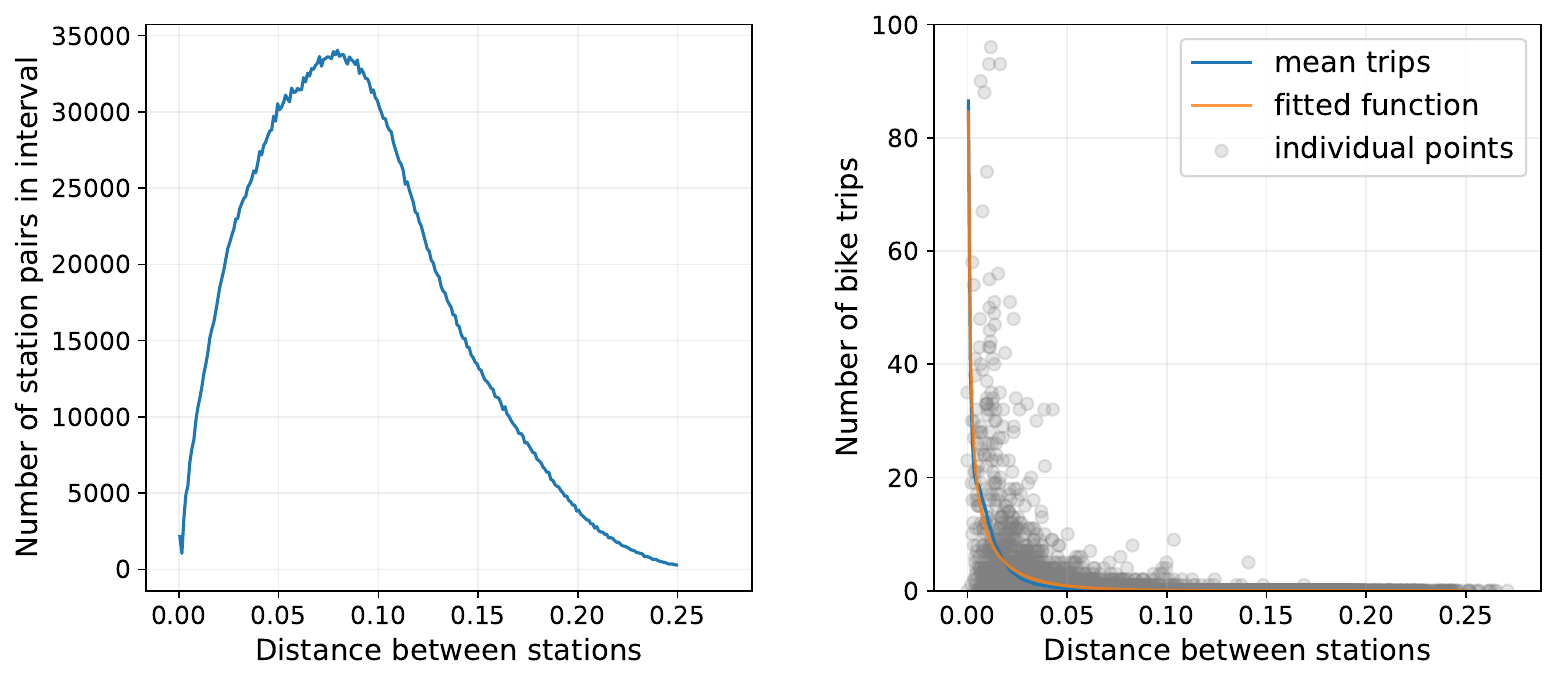}
    \caption{Number of station pairs (left) and number of bike trips (right) over }
    \label{fig:dist-vs-trips}
\end{figure}
In Figure \ref{fig:dist-vs-trips} (right), we plot the empirical distribution and our fitted function, along with a sample of individual data points (i.e., individual station pairs).
Our fitted function matches the empirical distribution well and, as expected, the number of bike trips decreases with distance.
However, from the individual data points, we can see that there is also substantial variance in the number of bike trips over station pairs of the same distance, revealing limitations of the gravity model.

\paragraph{Results (Table \ref{tab:ipf-bikeshare}).}
\begin{table}[]
    \centering
    \begin{tabular}{c|c|c|c|c}
        IPF (month) & IPF (week) & IPF (day) & Gravity & No $\Xagg$ \\
        \hline
        0.363 & 0.389 & 0.496 & 0.204 & 0.116 \\
        \hline
        \hline
        No $\ptime$ & No $\qtime$ & Scale month & Scale week & Scale day \\
        \hline
        0.278 & 0.275 & 0.201 & 0.215 & 0.278
    \end{tabular}
    \caption{Mean cosine similarity between $\Xtime$ and $\estXtime$, over 24 hours on September 1, 2024, for ten different methods.}
    \label{tab:ipf-bikeshare}
\end{table}
In our main text (Figure \ref{fig:ipf-bikeshare}), we visualized results from the first seven methods listed above.
First, we find that IPF strongly outperforms the gravity model, with a $78\%$ improvement in cosine similarity even when using the month-aggregated network, demonstrating the value of using a time-aggregated network over a distance-based model.
Second, IPF also outperforms the ablation baselines: when $\Xagg$ is the month-aggregated network, IPF outperforms the ablation without $\Xagg$ by $214\%$ and the ablations without $\ptime$ or $\qtime$ by around $31\%$, showing the value of using all three pieces of information.
Notably, removing $\Xagg$ is much more detrimental to performance than removing $\ptime$ or $\qtime$, revealing which piece of information is more informative.
Third, finer temporal granularity in the time-aggregated network improves IPF's performance, but the relative improvement is much larger from week to day ($+27\%$) than from month to week ($+7\%$).
This suggests that bike trips vary more at a daily scale within the week (e.g., weekday vs. weekend) than a weekly scale within the month.
However, the improvement in IPF performance at the daily level comes with a cost in convergence time, which we discuss more in Section \ref{sec:bikeshare-convergence}.
\begin{figure}[t]
    \centering
    \includegraphics[width=0.5\linewidth]{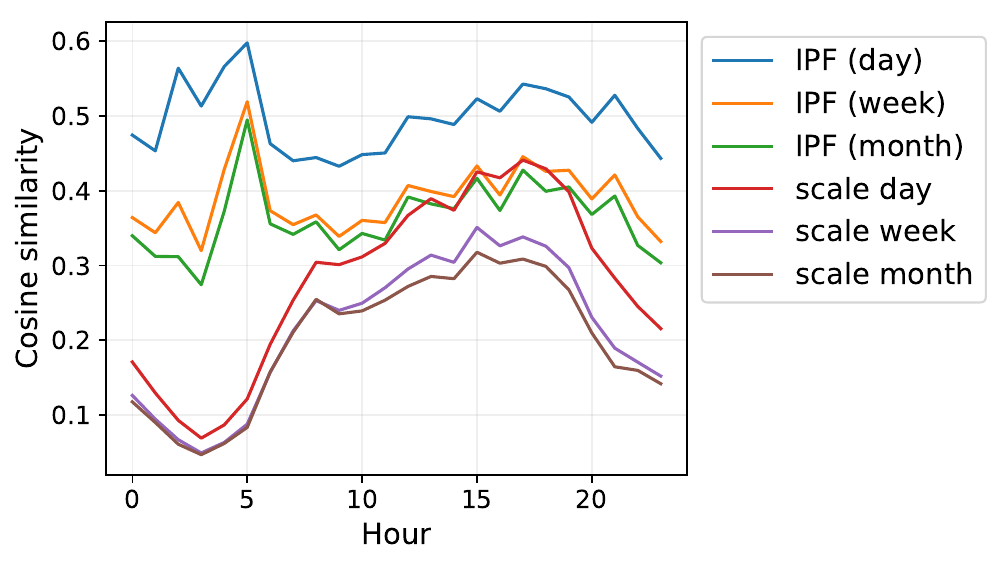}
    \caption{Cosine similarity between ground-truth hourly networks $\Xtime$ from bikeshare data and inferred networks $\estXtime$ from IPF and scaling baselines.}
    \label{fig:bikeshare-timeagg}
\end{figure}
In Figure \ref{fig:bikeshare-timeagg}, we provide extended results on the scaling baselines.
We find that, at all levels of time aggregation, running IPF on the time-aggregated network and hourly marginals greatly outperforms scaling the time-aggregated network to the hourly total.
We also find that the day-aggregated network is much more similar to the hourly network than the week-aggregated or month-aggregated networks, further supporting the theory above that bike trips vary more at a daily than a weekly scale. 
Interestingly, IPF still outperforms the day-aggregated network, even when it only has access to the \textit{month-aggregated network}, with a $31\%$ improvement in cosine similarity.
This finding reveals that, while the day-aggregated network is closer to the hourly network than the month-aggregated network, the relative information gain from the hourly marginals is still higher.
This finding may have interesting implications for data providers, when deciding how to aggregate their data.

\subsubsection{Convergence behavior}
\label{sec:bikeshare-convergence}
We observe a tradeoff between accuracy and convergence time.
Using the day-aggregated network for IPF consistently outperforms the week-aggregated or month-aggregated network at recovering the true hourly network (Table \ref{tab:ipf-bikeshare}), but running IPF on the day-aggregated network takes substantially more iterations than using the week-aggregated network or month-aggregated network.
Over the 24 hours we tested, the mean and median number of iterations when using the day-aggregated network were 7643.29 and 10000, respectively, while they were 920.79 and 419 for the week-aggregated network and 537.83 and 372 for the month-aggregated network.\footnote{These are actually underestimates, since we used a maximum of 10,000 iterations. Using the month-aggregated network never reached this maximum, week-aggregated reached it in one hour out of 24, and day-aggregated reached it in most hours.}
This is not surprising, since the day-aggregated network is far sparser (2.04\% nonzero) than the week-aggregated network (7.05\% nonzero) or month-aggregated network (13.72\% nonzero), and we know that IPF iterations increase with sparsity, which we also saw with synthetic data (Figure \ref{fig:ipf-sparsity}, left).

We can also evaluate IPF's performance over iterations, both in terms of its error on the target marginals as well as the cosine similarity between its estimated hourly network and the true hourly network.
We test IPF on the month-aggregated network and hourly marginals, for four hours on September 1, 2023 (12 AM, 6 AM, 12 PM, and 6 PM).
\begin{figure}[t]
    \centering
    \includegraphics[width=0.9\linewidth]{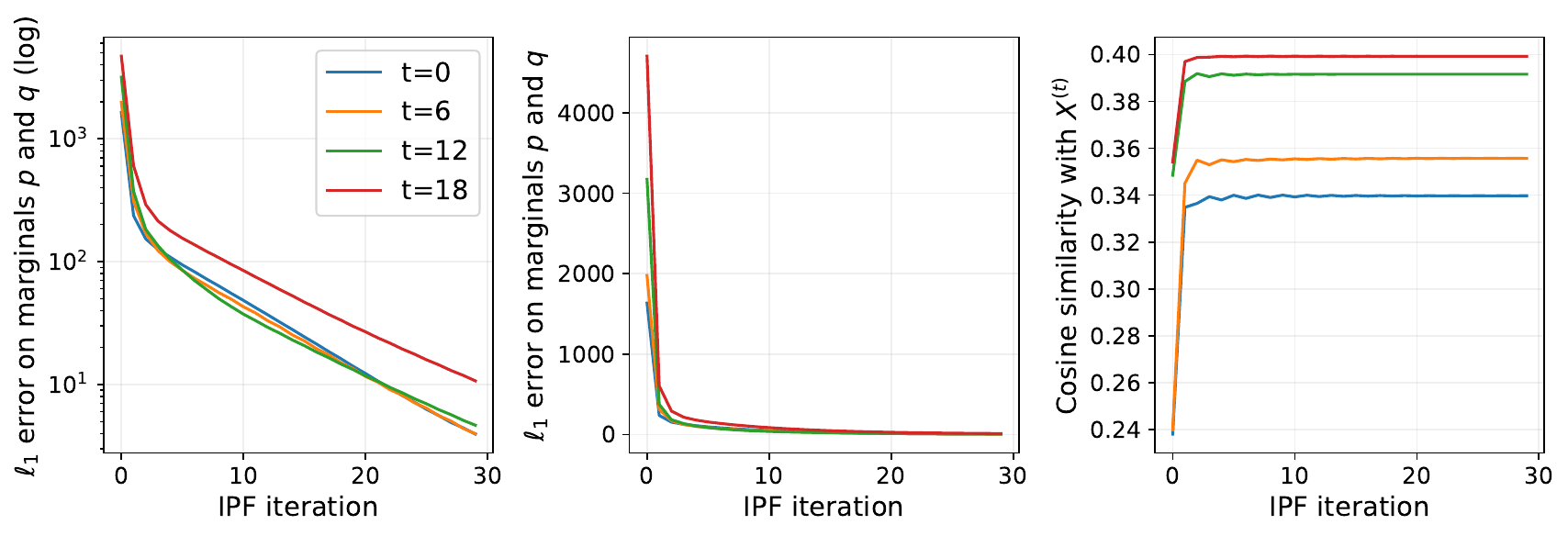}
    \caption{IPF performance over iterations: $\ell_1$ error on marginals $\ptime$ and $\qtime$ (with and without log scaling, left and middle); cosine similarity between $\Xtime$ and IPF's estimate $\estXtime$ (right).}
    \label{fig:bikeshare-over-iters}
\end{figure}
First, we observe that the $\ell_1$ error on the marginals drops exponentially, but exhibits two different rates, one larger for the first few iterations then a slower rate for the remaining iterations (Figure \ref{fig:bikeshare-over-iters}, left and middle).
This is similar to what we found on the SafeGraph mobility data, where the convergence rate exhibited a one-time bend (Figure \ref{fig:ipf-l1-marginals}).
We show the $\ell_1$ error with and without log scaling since log scaling helps to see the change in convergence rate but without log scaling emphasizes that most of the $\ell_1$ error is reduced in the first few iterations.
Second, when we consider the cosine similarity between the true network and IPF estimated network, we find that \textit{all} of the improvement occurs in the first few iterations, and after that, the similarity does not improve (Figure \ref{fig:bikeshare-over-iters}, right). 
Both of these analyses show that the vast majority of IPF's estimation capability is reached in the first few iterations, and running the next hundreds (or even thousands) of iterations will help IPF converge, but will not effectively improve the estimated network.
This is a useful finding if efficiency is desired in network estimation; for example, if IPF is used in real-time to estimate time-varying networks, e.g., for transportation planning.

\subsubsection{Examining model assumptions on bikeshare data}
\label{sec:test-assumptions}
In this final set of experiments, we leverage our opportunity with ground-truth hourly networks to test how reasonable the assumptions of the biproportional Poisson model \eqref{eqn:model} are.

\paragraph{Stationarity assumption.}
In Appendix \ref{sec:joint}, we discussed how our model, which solves the network inference problem in a decoupled fashion (separately estimating parameters per $t$), leaves out a potential piece of information, which is that $\Xagg = \sum_{t=1}^T \Xtime$, for some large $T$.
We showed that the joint problem, which incorporates the constraint that the inferred $\Xtime$'s sum to $\Xagg$, reduces the decoupled problem without the constraint when the following stationarity assumption is true:
\begin{align}
    \sum_{t=1}^T e^{u_{it}-v_{jt}} \approx c,
\end{align}
for some constant $c$, for all $i,j$ where $\Xagg > 0$.
We are not able to perfectly verify this on our data, since we do not know the true parameters $u$ or $v$ (or whether the data comes from this model at all), but we can check the estimates for $e^{u_i}$  and $e^{-v_j}$, which correspond to $d^0_i$ and $d^1_j$ in IPF, respectively.
So, we want to see whether
\begin{align}
    \sum_{t=1}^{T} d^0_{i}(t) d^1_{j}(t) \approx c.
\end{align}
To check this, we fit IPF on the month-aggregated bikeshare network for all 720 hours in the month of September 2023.
Then, for all $i,j$ where $\bar{X}_{ij} > 0$, we compute $d^0_i(t) d^1_j(t)$ over all hours in the month.
It turns out that the sums are quite close to each other, and close to 1, where the 5th to 95th percentile ranges from 0.88 to 1.19.
These results motivate our stationarity assumption, which allows us to decouple the problem.

\paragraph{Overdispersion.}
A key assumption of our model is that the values in the hourly network follow Poisson distributions.
A common issue with Poisson distributions is overdispersion, since we often find that real-world data has greater variance than the Poisson (which assumes variance is equal to the mean).
To test for overdispersion, we investigate the Pearson residuals of our fitted model, which are $(y_i - \exp(\mathbf{x}_i \beta))/ \sqrt{\exp(\mathbf{x}_i \beta)}$ for a generic Poisson regression model, which is equivalent to 
\begin{align}
    r_{ij} = \frac{\Xtime_{ij} - d^0_i \Xagg_{ij} d^1_j}{\sqrt{d^0_i \Xagg_{ij} d^1_j}} \label{eqn:pearson-residual}
\end{align}
in our IPF setting.
Recall that our set of Poisson observations $\mathcal{D}$ consist of all $(i,j)$ where $\Xagg_{ij} > 0$, $\ptime_i > 0$, and $\qtime_j > 0$.
Then, we can estimate the dispersion parameter $\hat{\phi}$, which is the sum of the squared Pearson residuals divided by the degrees of freedom (number of observations minus number of model parameters):
\begin{align}
    \hat{\phi} = \frac{\sum_{(i,j) \in \mathcal{D}} r_{ij}^2}{|\mathcal{D}|-m-n}.\label{eqn:dispersion-parameter}
\end{align}

In the absence of overdispersion, the dispersion parameter should be close to 1; otherwise, it will be greater than 1.
\begin{table}[]
    \centering
    \begin{tabular}{c|c|c|c}
        Time-aggregation & Mean $\hat{\phi}$ & 25th percentile $\hat{\phi}$ & 75th percentile $\hat{\phi}$ \\
        \hline
        Month & 1.093 & 1.027 & 1.190 \\
        Week & 1.086 & 1.046 & 1.159 \\
        Day & 1.066 & 1.032 & 1.099 \\
    \end{tabular}
    \caption{Distribution of dispersion parameter $\hat{\phi}$ when running IPF over 24 hours on September 1, 2023, and different time-aggregated networks.}
    \label{tab:dispersion-param}
\end{table}
When we test the IPF estimates on the bikeshare data (24 hours on September 1, 2023), we find that the dispersion parameter is actually quite close to 1 for all three levels of time aggregations (Table \ref{tab:dispersion-param}).
Day-aggregated is the closest to 1, with a mean of 1.066, but week-aggregated and month-aggregated are still close, with means of 1.086 and 1.093, respectively.
So, the data appears to be very slightly overdispersed, but not far from the Poisson assumptions.

\paragraph{Independence assumptions and model fit.}
Another assumption of our model is that the values in the hourly network are independently sampled from their respective Poisson distributions.
It is difficult for us to test all possible violations of this assumption, but we can test one natural dimension of correlation, which is spatial. 
For spatial relationships, our model should already capture spatial correlations between similar start stations (e.g., more traffic in certain neighborhoods at different times of day) and spatial correlations between similar end stations, since it learns a parameter for each row (i.e., start station) and column (i.e., end station).
The question is whether there are additional interaction effects between start and end stations.
Since we are interested in the interaction, we use the \textit{distance} between the start and end stations as a dimension of interest.
We described in Section \ref{sec:eval-ground-truth} how we compute distances between bike stations, in order to fit the gravity model, and visualized the relationship between distance and bike trips (Figure \ref{fig:dist-vs-trips}).

We compare the Pearson residual $r_{ij}$ to the distance between stations $i$ and $j$ for all station pairs $i,j$ in $\mathcal{D}$.
Over these pairs, we find that there \textit{is} a relationship between distance and residuals.
\begin{figure}[t]
    \centering
    \includegraphics[width=0.7\linewidth]{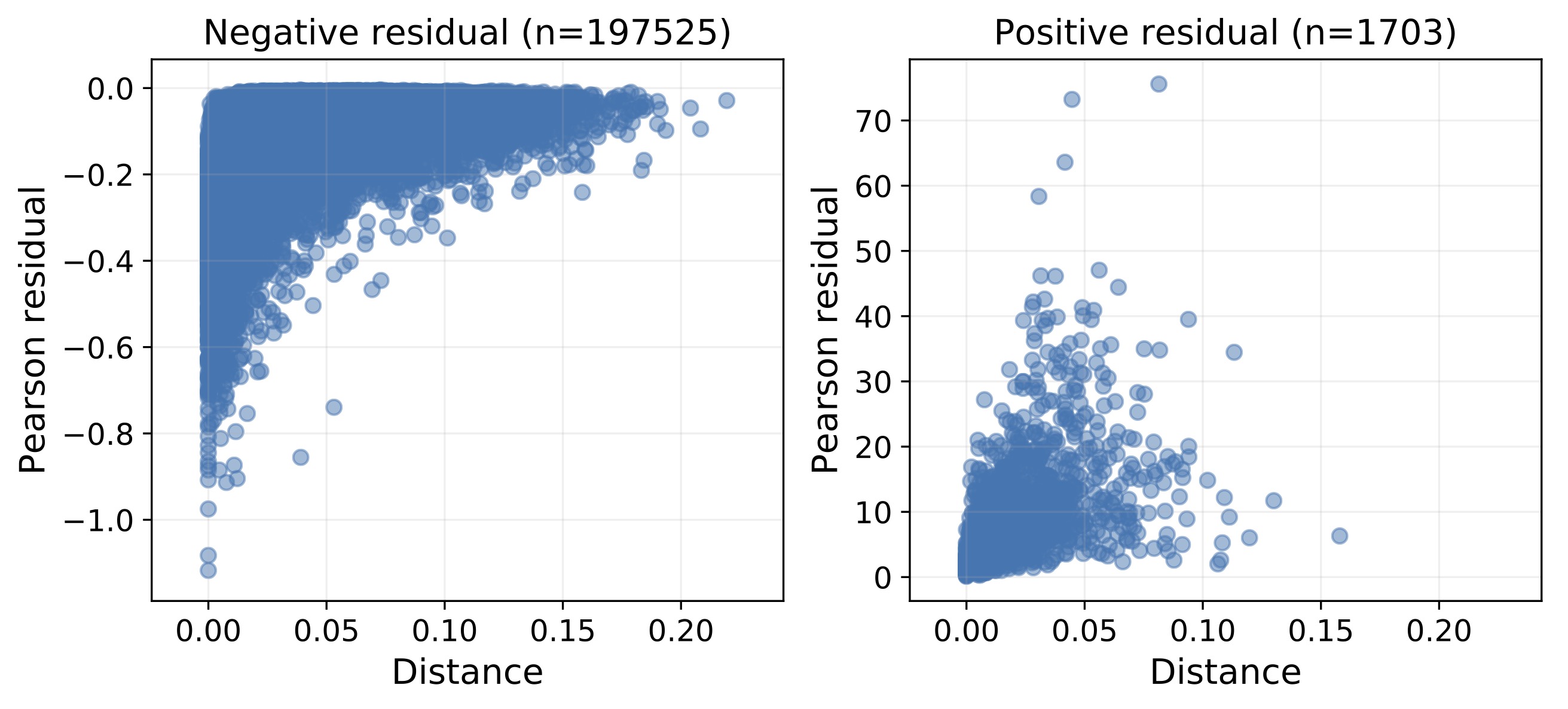}
    \caption{Distance between stations vs. Pearson residual over station pairs $i,j$. Results are from IPF estimates on month-aggregated network and hourly marginals for 12 AM on September 1, 2023.}
    \label{fig:distance-vs-residual}
\end{figure}
In Figure \ref{fig:distance-vs-residual}, we show a representative example from running IPF on the month-aggregated network and hourly marginals for 12 AM on September 1, 2023.
To make the visualization clearer, we split the data points into two groups: negative residuals and positive residuals.
We can see that the vast majority (over 99\%) of data points have negative residuals, which is expected since the true values, $\Xtime_{ij}$, are equal to zero most of the time, so any positive expected value will result in a negative residual.
Within both the negative and positive groups, we observe a positive relationship between distances and residuals.
This positive relationship is consistent over hours and time aggregations, although the correlation is weaker for smaller time aggregations.

The Pearson residuals \eqref{eqn:pearson-residual} also allow us to test the Poisson model's goodness-of-fit.
The numerator in \eqref{eqn:dispersion-parameter} is known as the Pearson statistic, which follows a chi-square distribution with $n-k-1$ degrees of freedom (where $n$ is the number of observations and $k$ is the number of model parameters).
We find that the goodness-of-fit test is \textit{rejected} for our model, meaning there is a statistically significant lack of fit.
This is not entirely surprising, since our model is simple due to the very limited information we have about the network, but it is worth keeping in mind as a caveat when using IPF to infer networks.
The lack of fit, along with the violation of independence assumptions, motivate the development of future methods that incorporate additional available information, e.g., interaction terms between features such as distance, while still only relying on the time-varying marginals and time-aggregated network to estimate the time-varying network. 

\end{document}